\documentclass[nohyperref]{article}

\usepackage{microtype}
\usepackage{graphicx}
\usepackage{subfigure}
\usepackage{booktabs} 

\usepackage{hyperref}

\usepackage[utf8]{inputenc}
\usepackage{amsfonts}
\usepackage{amsmath}
\usepackage{amsthm}
\usepackage{subfigure}
\usepackage{wrapfig}
\usepackage{graphicx}
\usepackage{color}
\usepackage{caption}
\usepackage{diagbox}
\usepackage{tabularx}

\usepackage{xcolor}
\usepackage{natbib}
\let\hat\widehat
\let\tilde\widetilde

\usepackage[accepted]{icml2022}

\usepackage{amsmath}
\usepackage{amssymb}
\usepackage{mathtools}
\usepackage{amsthm}

\usepackage[capitalize,noabbrev]{cleveref}

\theoremstyle{plain}
\newtheorem{theorem}{Theorem}[section]

\newtheorem{lemma}[theorem]{Lemma}

\theoremstyle{definition}
\newtheorem{definition}[theorem]{Definition}
\newtheorem{assumption}[theorem]{Assumption}
\theoremstyle{remark}

\usepackage{smile}

\newcommand{\set}[1]{\left\{#1\right\}}

\newcommand{\innerprod}[1]{\left\langle#1\right\rangle}

\def\##1\#{\begin{align}#1\end{align}}
\def\$#1\${\begin{align*}#1\end{align*}}
\newcommand{\braces}[1]{\left\{#1\right\}}
\newcommand{\prns}[1]{\left(#1\right)}
\newcommand{\MLE}{\mathrm{MLE}}
\newcommand{\TV}{\mathrm{D}_\mathrm{TV}}

\newcommand{\bracks}[1]{\left[#1\right]}
\def\oper{\mathop{\text{op}}}

\def\rank{\mathrm{rank}}

\usepackage[textsize=tiny]{todonotes}

\icmltitlerunning{On the Role of Discount Factor in Offline Reinforcement Learning}

\begin{document}

\twocolumn[
\icmltitle{On the Role of Discount Factor in Offline Reinforcement Learning}



\icmlsetsymbol{equal}{*}

\begin{icmlauthorlist}
\icmlauthor{Hao Hu}{equal,yyy}
\icmlauthor{Yiqing Yang}{equal,comp}
\icmlauthor{Qianchuan Zhao}{comp}
\icmlauthor{Chongjie Zhang}{yyy}
\end{icmlauthorlist}

\icmlaffiliation{yyy}{Institute of Interdisciplinary Information Sciences, Tsinghua University, Beijing, China}
\icmlaffiliation{comp}{Department of Automation, Tsinghua University, Beijing, China}

\icmlcorrespondingauthor{Hao Hu}{hu-h19@mails.tsinghua.edu.cn}
\icmlcorrespondingauthor{Yiqing Yang}{yangyiqi19@mails.tsinghua.edu.cn}

\icmlkeywords{Machine Learning, ICML}

\vskip 0.3in
]



\printAffiliationsAndNotice{\icmlEqualContribution} 

\begin{abstract}
    Offline reinforcement learning (RL) enables effective learning from previously collected data without exploration, which shows great promise in real-world applications when exploration is expensive or even infeasible. The discount factor, $\gamma$, plays a vital role in improving online RL sample efficiency and estimation accuracy, but the role of the discount factor in offline RL is not well explored. This paper examines two distinct effects of $\gamma$ in offline RL with theoretical analysis, namely the regularization effect and the pessimism effect. On the one hand, $\gamma$ is a regulator to trade-off optimality with sample efficiency upon existing offline techniques.
    On the other hand, lower guidance $\gamma$ can also be seen as a way of pessimism where we optimize the policy's performance in the worst possible models. 
    We empirically verify the above theoretical observation with tabular MDPs and standard D4RL tasks. The results show that the discount factor plays an essential role in the performance of offline RL algorithms, both under small data regimes upon existing offline methods and in large data regimes without other conservative methods.
\end{abstract}

\section{Introduction}
Online reinforcement learning has achieved great success in various domains, including video games \citep{mnih2015human}, class games~\citep{silver2016mastering} and continuous control~\citep{lillicrap2015continuous}. However, it requires extensive interaction with the environment to learn through trial-and-error. In many real-world problems, like personalized recommendation systems~\citep{swaminathan2015batch} and autonomous driving~\citep{shalev2016safe,singh2020cog}, exploration can be expensive and unsafe, which requires conducting RL in a supervised manner. In contrast, offline RL~\citep{levine2020offline,fujimoto2019off} enables learning from previously collected datasets, which shows great potential for real-world applications.

\begin{table}
    \vspace{2em}
    \centering
    \begin{tabularx}{0.5\textwidth}{|c|X|X|}
    \hline
    Dataset size/quality & w other pessimisms  & w$\backslash$o other  pessimisms  \\ 
    \hline
    Large, good coverage &      &   pessimism effect \checkmark \\ 
    \hline
    Small or bad coverage &   regurlarization effect \checkmark &      \\ 
   \hline
    \end{tabularx}
    \caption{The applicability of a lower guidance discount factor.}
    \vspace{-1.5em}
\label{summarize}
\end{table}

One of the major challenges of offline RL comes from the distributional shift between the data collection policy and the learned policy~\citep{levine2020offline}, and the direct application of online RL algorithms is known to lead to poor performance~\citep{fujimoto2019off}. One paradigm for algorithm design in offline RL incorporates proper conservatism to the learning algorithm. There are several conservative methods in existing empirical literature, including policy regularization~\citep{fujimoto2019off,peng2019advantage,kumar2020conservative}, ensemble-based uncertainty~\citep{wu2021uncertainty,an2021uncertainty} and model-based penalty~\citep{yu2021combo,yu2020mopo,kidambi2020morel}.

However, as an essential number in RL, the discount factor already provides a natural way of conservatism. The effect of $\gamma$ is extensively discussed in online RL and specifically TD-learning. For instance, \citet{jiang2015dependence} shows that the discount factor can be regarded as a complexity control parameter for the class of policies. \citet{amit2020discount} shows that it is equivalent to L2-regularization of Q-values in TD-learning to improve generalization. However, the analysis of the role of $\gamma$ in the context of offline RL is missing, which naturally leads to the following question:
\begin{center}
    {\it What are the roles of the discount factor in the context of offline RL, and how does it contribute to proper conservative offline RL methods?}
\end{center}
In this paper, we examine the two roles of the discount factor in offline RL that affect the performance in two distinct ways. 
First, we show that $\gamma$ plays as a regulator upon existing conservative methods and achieves a trade-off between optimality and sample efficiency. 
Second, we show that a lower discount factor resembles model-based pessimism. A lower discount factor is equivalent to maximizing the value function in the worst possible models among a confidence set. We summarize the applicability of the two effects in Table~\ref{summarize}.
To give a rigid characterization, we analyze the two effects above under the context of linear MDPs and derive two different performance bounds. These quantitive results also characterize how the impact of a lower guidance discount factor relies on other factors like the size of the dataset and the coverage coefficient in terms of optimal policies.
We empirically verify the two effects on both tabular MDPs and the standard D4RL benchmark~\citep{fu2020d4rl}.
Results indicate that the discount factor can significantly affect the performance of offline RL algorithms, both under small data regimes upon existing pessimism and in large data regimes without other pessimisms. We believe that our findings highlight the importance of setting a proper lower discount factor in offline RL practices.

\subsection{Related Works}
\paragraph{Role of Discount Factor.} 
The discount factor is extensively analyzed in online RL.~\citet{petrik2008biasing} first shows that the approximation error bounds may be tightened with a lower discount factor in online RL. \citet{jiang2015dependence} gives a more rigorous analysis by analyzing the size of potentially optimal policy class; ~\citet{amit2020discount} points out the equivalence between the L2 regularization and a smaller discount factor in TD-learning.
~\citet{zhang2020deeper} analyzes the discount factor in actor-critic algorithms from a bias-variance-representation trade-off perspective.
In the practical aspect, ~\citet{chen2018improving} observes a similar regularization effect in partial observable settings; ~\citet{fedus2019hyperbolic} uses a geometry discount factor to learn multiple values in an ensemble way.
Some works~\citep{xu2018meta, sherstan2020gamma, romoff2019separating} suggest learning a sequence of value functions with lower discount factors for online RL.

\paragraph{Convervatism in Offline RL.} 
There are extensive works on the conservative offline RL, which can be roughly divided into policy constraint-based and uncertainty estimation-based.
Policy constraint methods attempt to enforce the trained policy to be close to the behavior policy via introducing a generative model~\citep{fujimoto2019off,fujimoto2021td3}, KL-divergence~\citep{peng2019advantage, nair2020accelerating,siegel2020keep,wu2019behavior} or value function regularization~\citep{kumar2020conservative, agarwal2020optimistic, kostrikov2021offline}.
Some more recent policy constraint methods~\citep{yang2021believe,ma2021offline,ghasemipour2021emaq,kostrikov2021offline, brandfonbrener2021offline} also suggested that only trusting the state-action pairs given in the dataset can solve complex tasks well.
As for the uncertainty estimation of model-free methods, researchers attempt to take into account the confidence of the Q-value prediction using dropout or ensemble techniques~\citep{wu2021uncertainty, an2021uncertainty}. Differently, model-based offline methods incorporate the uncertainty in the model space~\citep{yu2020mopo, yu2021combo, kidambi2020morel}.

In addition, there are also some theoretical results for pessimism in offline RL.
\citet{jin2021pessimism} proves that using a negative bonus from the online exploration is sufficient for offline RL. 
\citet{rashidinejad2021bridging} proves that a UCB-like penalty is provably sufficient in tabular settings.
\citet{uehara2021pessimistic} studies the properties of model-based pessimisms under partial coverage.
\section{Preliminaries}
We consider infinite-horizon discounted Markov Decision Processes (MDPs), defined by the tuple $(\mathcal{S},\mathcal{A},\mathcal{P},r,\gamma),$ where $\mathcal{S}$ is a state space, $\mathcal{A}$ is a action space, $\gamma \in [0,1)$ is the discount factor and $\mathcal{P}: \mathcal{S}\times\mathcal{A}\rightarrow \Delta(\mathcal{S}), r:\mathcal{S}\times\mathcal{A}\rightarrow [0, r_{\text{max}}]$ are the transition function and reward function, respectively. We also assume a fixed distribution $\mu_0 \in \Delta(\mathcal{S})$ as the initial state distribution.\par

To make our analysis more general, we consider the \textit{linear MDP}~\citep{yang2019sample,jin2020provably} as follows, where the transition kernel and expected reward function are linear with respect to a feature map. Note that tabular MDPs are linear MDPs with the canonical one-hot representation.

\begin{definition}[Linear MDP]
    We say an infinite-horizon discounted MDP $(\cS,\cA,\cP,r, \gamma)$ is a linear MDP with known feature map $\phi:\cS\times \cA\to \RR^d, \psi:\cS\to \RR^l$ if there exist an unknown matrix parameter $M \in \RR^{d\times l}$ and an unknown vector parameter $\theta \in \RR^d$ such that
    \#
    &\cP(s'\given s,a) = \phi(s,a)^\top M \psi(s'),\notag\\
    &\EE\bigl[r(s, a)\bigr] = \phi(s,a)^\top\theta
    \#
    for all $(s,a,s')\in \cS\times \cA\times \cS$. And we assume $\|\phi(s,a)\|_\infty\leq 1,\|\psi(s')\|_\infty\leq 1$ for all $(s,a,s')\in \cS\times \cA \times\cS$ and $\max\{\| M \|_2 ,\|\theta\|_2\}\leq \sqrt{d}$.
    \label{assump:linear_mdp}
\end{definition}
A policy $\pi : \cS\rightarrow \Delta{(\cA)}$ specifies a decision-making strategy in which the agent chooses its actions based on the current state, i.e., $a_t \sim \pi(\cdot \given s_t)$. The value function $V^\pi_{\gamma,M}: \cS \rightarrow \RR$ is defined as the $\gamma$-discounted sum of future rewards starting at state $s$ for policy $\pi$ in model $M$, i.e.
\begin{equation}
V^\pi_{\gamma,M}(s) = \EE_{\pi}\Big[ \sum_{t=0}^{\infty} \gamma^t r(s_t, a_t)\Biggiven s_0=s  \Big],
\label{eq:def_value_fct}
\end{equation}
where the expectation is with respect to the trajectory $\tau$ induced by policy $\pi$. The action-value function $Q^\pi:\cS\times \cA\to \RR$ is similarly defined as
\begin{equation}
Q^\pi_{\gamma,M}(s,a) = \EE_\pi\Big[\sum_{t=0}^{\infty} \gamma^tr(s_t, a_t)\Biggiven s_0=s, a_0=a  \Big].
\label{eq:def_q_fct}
\end{equation}
We overload the notation and define $V_{\gamma,M}(\pi)$ as the expected $\gamma$-discounted value of policy $\pi$ under the initial distribution $\mu_0$, i.e. $V_{\gamma,M}(\pi) = \EE_{s_0\sim\mu_0}[V^\pi_{\gamma,M}(s_0)]$, and similarly we have  $Q_{\gamma,M}(\pi) = \EE_{\begin{subarray} ss_0\sim\mu_0\\a_0\sim\pi\end{subarray}}[Q^\pi_{\gamma,M}(s_0,a_0)]$. When it does not lead to confusion, we omit the index for $\gamma$ and $M$ for simplicity.

We define the Bellman operator as
\begin{align}
(\BB_\gamma f)(s,a) &= \EE\bigl[r(s, a) + \gamma f(s')\bigr],
\label{eq:def_bellman_op}
\end{align}
for any $f:\mathcal{S}\rightarrow \mathbb{R}$ and $\gamma\in[0,1)$.
The optimal Q-function~$Q^*$, and the optimal value function~$V^*$ are related by the Bellman optimality equation
\begin{align}
 &V^*_{\gamma,M}(s) = \max_{a\in \cA}Q^*_{\gamma,M}(s,a),\notag\\
 &Q^*_{\gamma,M}(s,a) = (\BB_\gamma V^*_{\gamma,M}) (s,a),
\label{eq:dp_optimal_values}
\end{align}
while the optimal policy is defined as
\begin{align*}
   \pi^*_{\gamma,M} (\cdot \given s)&=\argmax_{\pi}\EE_{a\sim \pi}{Q^*_{\gamma,M}(s, a)}.
\end{align*}

We define the suboptimality as the performance difference of the optimal policy $\pi^*_\gamma$ and the policy $\pi$ given the initial distribution $\mu_0$ evaluated with discount factor $\gamma$. That is  
\begin{equation}
\text{SubOpt}(\pi;\gamma) = V_{\gamma,M}(\pi^*_\gamma) - V_{\gamma,M}(\pi),
\label{eq:def_regret}
\end{equation}
We also define the suboptimality for each state, that is
\begin{equation}
    \text{SubOpt}(\pi,s;\gamma) = V^{\pi^*_\gamma}_{\gamma,M}(s) - V^\pi_{\gamma,M}(s). \notag
    \label{eq:def_regret_2}
\end{equation}

\subsection{Pessimistic Offline Algorithms}
\label{algo_meta}
In this section, we sketch two offline algorithms to characterize the effect of a lower guidance discount factor. The first one is \textit{pessimistic value iteration} \citep[PEVI;][]{jin2021pessimism}, as shown in Algorithm~\ref{alg:1}, which uses uncertainty as a negative bonus for value learning.
The second one is \textit{model-based pessimistic policy optimization} \citep[MBPPO;][]{uehara2021pessimistic}, as shown in Algorithm~\ref{alg:2}, which optimizes the worst possible performance of a policy over a set of models. 

PEVI uses negative bonus $\Gamma(\cdot,\cdot)$ over standard $Q$-value estimation $\hat{Q}(\cdot,\cdot) =  (\hat\BB_\gamma \hat{V})(\cdot)$ to reduce potential bias due to finite data, where $\hat\BB_\gamma$ is the empirical estimation of $\BB_\gamma$ from dataset $\cD$. We use the notion of  $\xi$-uncertainty quantifier as follows to formalize the idea of pessimism.
\begin{definition}[$\xi$-Uncertainty Quantifier] 
    We say $\Gamma :\cS\times\cA\to \RR$ is a $\xi$-uncertainty quantifier for $\hat\BB_\gamma$ and $\widehat{V}$ if with probability $1-\xi$,
    \begin{equation}
     \big|(\hat\BB_\gamma\hat{V})(s,a) - (\BB_\gamma\hat{V})(s,a)\big|\leq \Gamma(s,a), 
    \label{eq:def_event_eval_err_general}
    \end{equation}
    for all~$(s,a)\in \cS\times \cA$.
    \label{def:uncertainty_quantifier}
\end{definition}


\begin{algorithm}[H]
    \caption{Pessimistic Value Iteration}\label{alg:1}
    \begin{algorithmic}[1]
    \STATE {\bf Require}: Dataset $\cD=\{(s_\tau,a_\tau,r_\tau)\}_{\tau=1}^{T}$.
    \STATE Initialization: Set $\hat{V}(\cdot) \leftarrow 0$ and construct $\Gamma(\cdot, \cdot)$.
    \WHILE{not converged}
    \STATE Construct $(\hat\BB_\gamma \hat{V})(\cdot,\cdot)$
    \STATE Set $\hat{Q}(\cdot,\cdot) \leftarrow  (\hat\BB_\gamma \hat{V})(\cdot,\cdot)- \Gamma(\cdot,\cdot)$.
    \STATE Set $\hat{\pi} (\cdot \given \cdot) \leftarrow \argmax_{\pi}\EE_{\pi}{\left[\hat{Q}(\cdot, \cdot)\right]}$.
    \STATE Set $\hat{V}(\cdot) \leftarrow  \EE_{\hat{\pi}}{\left[\hat{Q}(\cdot, \cdot)\right]}$. \label{alg:general_Vhat}
    \ENDWHILE
    \STATE \textbf{Return} $\hat\pi$%
    \end{algorithmic}
\end{algorithm}

\begin{algorithm}[H]
    \caption{Model-Based Pessimistic Policy Optimization}\label{alg:2}
    \begin{algorithmic}[1]
        \STATE {\bf Require}: Dataset $\cD$, discount factor $\gamma$, policy class $\Pi$, Model set $\cM$
        \STATE Optimize policy with respect to the worst possible model:
        \vspace{-5pt}
            \begin{align}\textstyle
            \label{alg:2_1}
            &\hat\pi = \argmax_{\pi\in\Pi} \min_{M\in\cM} V^{\pi}_{\gamma,M}.
            \vspace{-15pt}
        \end{align}
        \STATE \textbf{Return} $\hat\pi$%
    \end{algorithmic}
    \end{algorithm}

    Intuitively, $\Gamma(s,a)$ represents the uncertainty when estimating the value function. The negative bonus ensures that we do not overestimate the value function due to finite samples with high probability, which allows us to give a performance lower bound with respect to the number of samples in the dataset.

    MBPPO, on the contrary, considers the uncertainty in the model space to reduce potential sampling bias. By optimizing the performance in the worst possible model, as shown in Algorithm~\ref{alg:2}, we obtain a model-based offline algorithm, which gives better suboptimality bounds compared to the model-free counterpart when we use a proper set of models. 

Both Algorithm~\ref{alg:1} and Algorithm~\ref{alg:2} are meta descriptions rather than detailed implementations. We provide more details of the algorithms in Appendix~\ref{algo_describ}.

\section{Theoretical Analysis}
This section characterizes the effects of a lower guidance discount factor in offline RL. We first show that a similar regularization effect exists in offline algorithms as in online algorithms. We quantify this effect by providing a performance bound to analyze how other factors like the data size and the coverage coefficient affect this regularization effect. We then show an equivalence between a lower discount factor and the model-based pessimism. This equivalence leads to a performance guarantee with a lower guidance discount factor without other conservative regularizations. These two effects indicate that the discount factor plays a vital role in offline reinforcement learning.

\subsection{Regularization Effect}
As~\citet{jiang2015dependence} suggests, the discount factor acts as a regularization coefficient to bound the complexity of the potential optimal policy class. However, it is unclear what affects the effectiveness of $\gamma$ regularization in the offline setting, especially when the data coverage is poor and the algorithms have additional pessimism. Empirically, we found that the quality and size of the dataset are the main factors that affect the regularization from $\gamma$. To shed light on this observation, we first derive a performance bound in the linear MDP setting for model-free pessimistic algorithms like Algorithm~\ref{alg:1}. The analysis is analogous to~\citep{jin2021pessimism}, but we focus on the discount setting.

\begin{lemma}
    \label{lemma:2}
    Suppose there exists an absolute constant 
    \begin{align}
        \label{eq:event_opt_explore}
        c^\dagger = &\sup_{x\in \RR^d} \frac{x^{\top}\Sigma_{\pi^*,s}x}{x^{\top}\Sigma_{\cD} x} < \infty,
    \end{align}
    for all $s\in\cS$ with probability $1-\xi/2$, where 
    \$&\Sigma_{\cD}~~~=\frac{1}{N}\sum_{\tau=1}^N{\left[\phi(s_\tau,a_\tau)\phi(s_\tau,a_\tau)^\top\right]},\notag\\
     &\Sigma_{\pi^*,s}=\EE_{{\pi^*}}\bigl[\phi(s_t,a_t)\phi(s_t,a_t)^\top\biggiven s_0=s\bigr].
     \$

    In Algorithm \ref{alg:1}, we follow Equation~\eqref{eq:empirical_bellman} and~\eqref{eq:linear_uncertainty_quantifier}, and set 
    \begin{equation*}
        \lambda =1,~ \beta= c \cdot d r_{\text{\rm max}}\sqrt{\zeta}/(1-\gamma), ~ \zeta = \log{(4dN/(1-\gamma)\xi)},
    \end{equation*}
    where $c>0$ is an absolute constant  and $\xi \in (0,1)$ is the confidence parameter. Then with probability $1-\xi$, the policy $\widehat{\pi}$ generated by Algorithm~\ref{alg:1} satisfies
    \begin{align*}\label{eq:event_opt_explore_d}
     \text{\rm SubOpt}\big(\widehat{\pi},s;\gamma \big) 
    \leq  \frac{2c r_{\text{\rm max}}}{(1-\gamma)^2} \sqrt{c^\dagger d^3\zeta /N}, \ \forall s\in \cS.
    \end{align*}
\end{lemma}
\begin{proof}
    See Appendix~\ref{proof_lemma_1} for a detailed proof.
\end{proof}

Equation~\eqref{eq:event_opt_explore} defines a finite coverage coefficient, namely $c^\dagger$, which represents the maximum ratio between the density of empirical state-action distribution and the density induced from the optimal policy. Intuitively, it represents the quality of the dataset. For example, the \verb|expert| dataset has a low coverage ratio while the \verb|random| dataset may have a high ratio. The probability $1-\xi/2$ for a finite coverage coefficient is measured concerning the data collection process. That is, we are only making assumptions about the data collection process rather than the specific dataset.
The dependence on $\gamma$ in Lemma~\ref{lemma:2} suggests that the performance bound $\text{SubOpt}\big(\widehat{\pi},s;\gamma \big)$ decreases as the discount factor gets lower. However, a lower discount factor also biases the optimal policy, as characterized by the following lemma.


\begin{lemma}[\citet{jiang2015dependence}]
    \label{lemma:1}
     For any MDP $M$ with rewards in $[0,r_{\text{\rm max}}]$, $\forall \pi: \mathcal{S}\times\mathcal{A}\rightarrow \mathbb{R}$ and $\gamma\leq \gamma_{e}$,
    \begin{equation}
        V_{M,\gamma}(\pi) \leq V_{M,\gamma_{e}}(\pi) \leq  V_{M,\gamma}(\pi) + \frac{\gamma_{e}-\gamma}{(1-\gamma)(1-\gamma_{e})}r_{\text{\rm max}},
    \end{equation}
    where $\gamma_{e}$ is the evaluation discount factor.
\end{lemma}
The bound above is tight in a trivial case, while it is usually very loose in practice.
Together, we have the following bound with a guidance discount factor $\gamma$ different from the true evaluation discount factor $\gamma_e$.
\begin{theorem}
    \label{theorem:1}

    Based on the same assumption and definition as in Lemma~\ref{lemma:2}, we set $\zeta=\log{(4dN/(1-\gamma)\xi)}$, $\beta=c \cdot d r_{\text{\rm max}}\sqrt{\zeta}/(1-\gamma)$ and $\lambda =1$.
    Then with probability $1-\xi$, the suboptimality bound of the policy $\widehat{\pi}$ generated by Algorithm~\ref{alg:1} satisfies
    \begin{align*}
     \text{\rm SubOpt}\big(\widehat{\pi};\gamma_e \big) 
    \leq &\frac{2c}{(1-\gamma)^2} \sqrt{c^\dagger d^3\zeta/N}\cdot r_{\text{\rm max}} \nonumber\\
    &+ \frac{\gamma_{e}-\gamma}{(1-\gamma)(1-\gamma_{e})}\cdot r_{\text{\rm max}}.
    \end{align*}
\end{theorem}
\begin{proof}
    The result follows immediately from Lemma~\ref{lemma:2} and Lemma~\ref{lemma:1}.
\end{proof}

Theorem~\ref{theorem:1} gives an upper bound for pessimistic offline algorithms with two terms. Both terms monotonically depend on the guidance discount factor $\gamma$ but in an opposite way, which suggests {\em there exists an optimal trade-off $\gamma^*\in [0, \gamma_{e}]$}.
It also suggests that this optimal trade-off $\gamma^*$ is dependent on other factors like the coverage ratio of the dataset and the size of the dataset. A small or poorly covered dataset (i.e., a large coverage coefficient) makes the first term's coefficient larger, requiring a lower discount factor to achieve optimal performance. We empirically observe this effect on both toy examples as well as large D4RL tasks, as shown in Section~\ref{experiments}.

\subsection{Pessimistic Effect}
This section analyses the pessimism effect of a lower discount factor. We show that, perhaps surprisingly, learning with a lower discount factor is equivalent to one type of the model-based pessimism mechanism, as depicted in Algorithm~\ref{alg:2}. This is characterized by the following lemma.
\begin{lemma}
    \label{lemma:3}
    The optimal value function with a lower discount factor is equivalent to the pessimistic value function over a set of models.
    Formally, let
    \begin{equation}
    \label{model_based_policy_opt}
        \pi^*_{\mathcal{M}_\varepsilon} \in \argmax_{\pi} \min_{M\in\mathcal{M}_\varepsilon} V_{M,\gamma}(\pi),
    \end{equation}
    where
    \begin{equation}
        \label{small_gamma_policy_opt}
        \mathcal{M}_\varepsilon = \set{M| \cP_M(\cdot|s,a)= (1-\varepsilon)\cP_{M_0}(\cdot|s,a)+\varepsilon \cP(\cdot) }, \notag
    \end{equation}
    and $\cP(\cdot)$ is an arbitrary distribution over $\cS$, then we have
    \begin{equation}
        V^*_{M_0,(1-\varepsilon)\gamma}= V_{M_0,\gamma}(\pi^*_{\cM_\epsilon})+\Delta,
    \end{equation}
    where $\Delta$ is an absolute constant.
\end{lemma}

\begin{proof}
    See Appendix~\ref{proof_lemma_3} for a detailed proof.
\end{proof}

The equality in Lemma~\ref{lemma:3} shows that learning with a lower discount factor itself acts as a kind of model-based pessimism, which allows us to derive a bound without any other additional regularizations. We consider the model-based pessimism, where the model parameters are learned through maximum likelihood estimation (MLE). With the techniques in~\citep{geer2000empirical} that allow us to estimate the concentration rate of the MLE estimator, we have the following theorem. The proof of the following theorem is analogous to the analysis in~\citep{uehara2021pessimistic}.
\begin{theorem}
    \label{theorem:2}
    (\textit{informal}) Suppose there exists an absolute constant 
    \begin{align}
        \label{eq:event_opt_explore_2}
        c^\ddagger=\sup_{x\in\mathbb{R}^d} \frac{x^{\top}  \Sigma_{\pi^*} x}{x^{\top} \Sigma_\rho x } < \infty,
    \end{align} 
    $\Sigma_{\rho}=\EE_{\rho}[\phi(s,a)\phi(s,a)^{\top}],~\Sigma_{\pi^{*}}=\EE_{d^{\pi^{*}} }[\phi(s,a)\phi(s,a)^{\top}].$

    And suppose the underlying MDP follows the regularity condition in Assumption~\ref{assumption_regularity}. Set 
    $$\gamma = (1-\varepsilon)\gamma_e,~\varepsilon=c_1 \sqrt{d\zeta/N},\zeta = \log^2{(c_2Nd/\xi)}.$$
    Then, with probability $1-\xi$, Learning with a guidance discount factor $\gamma$ yields a policy $\hat{\pi}$ such that
    \begin{align}
        \text{\rm SubOpt}\big(\widehat{\pi};\gamma_e \big) \leq 
        \frac{c_3}{(1-\gamma_e)^{2}}\sqrt{c^\ddagger d^2 \zeta / N} \cdot r_{\text{max}},
    \end{align}
    where $c_1,c_2,c_3$ are constants.
\end{theorem}
\begin{proof}
    See Appendix~\ref{proof_theorem_2} for a detailed description and proof.
\end{proof}

Similar to $c^{\dagger}$ in Equation~\eqref{eq:event_opt_explore}, $c^{\ddagger}$ can be interpreted as another coverage coefficient, and the difference between $c^{\dagger}$ and $c^{\ddagger}$ are  technical. Theorem~\ref{theorem:2} shows that, with a properly chosen $\gamma$, we have a performance guarantee without any other offline techniques. We name this effect of the discount factor the pessimistic effect. 
Note that to make this bound meaningful, we need the dataset size $N$ to be large enough so that $1-\varepsilon >0$. This means that this theorem applies when the data is sufficient, contrary to the condition for the regularization effect. Compared to the bound in Theorem~\ref{theorem:1}, we note that the bound in Theorem~\ref{theorem:2} is smaller with a factor of $\sqrt{d}$ because a lower $\gamma$ resembles the model-based pessimism rather than a model-free one. 

We empirically verify this effect on the tabular MDPs and the D4RL dataset, where simple discount factor regularization is enough to derive a reasonable performance. 
This pessimistic effect suggests that $\gamma$ affects the performance in offline RL differently compared to online settings.

\begin{figure}[t]
    \centering
    \subfigure{
    \includegraphics[scale=0.35]{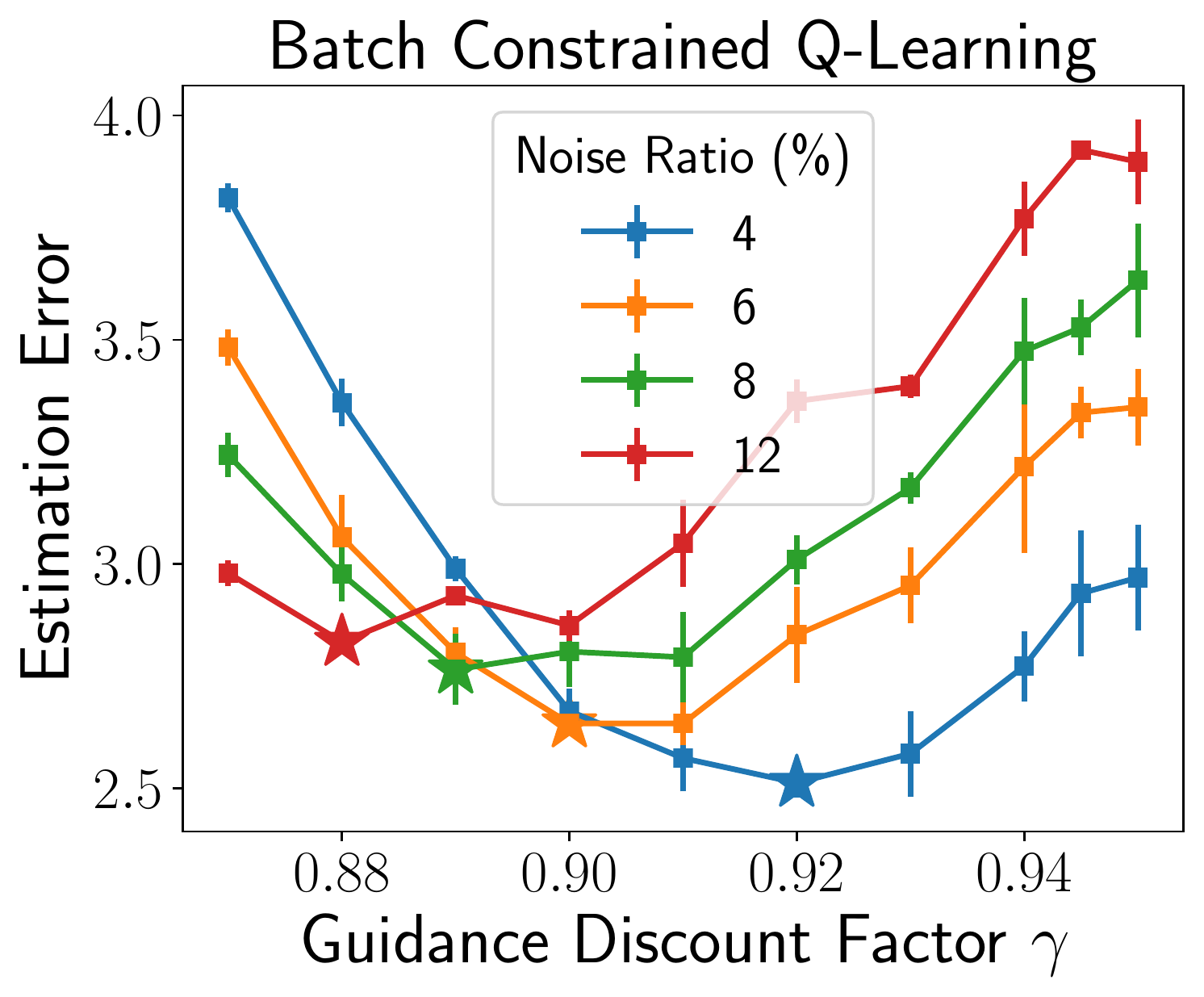}}
     \subfigure{
    \includegraphics[scale=0.35]{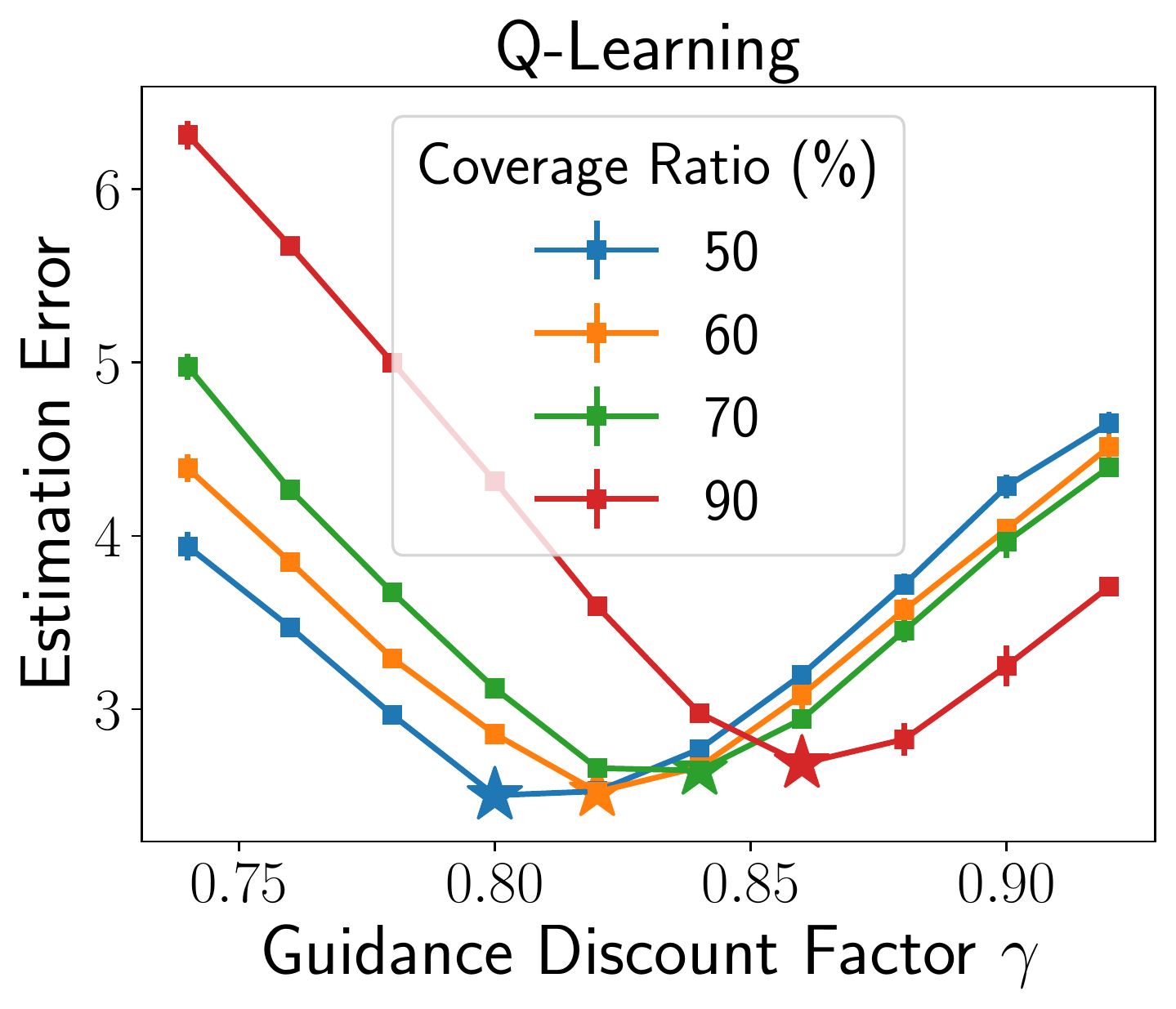}}
    \caption{The estimation error of BCQ and Q-Learning in the random MDP task.
    The star shapes mark the minimum of the curve.}
    \label{fig6: bcq}
\end{figure}

\begin{table*}[t]
    \centering
    \begin{tabular}{l|c|c|c|c|c|c}
        \hline
        Tasks & {\color{blue}BCQ} & {\color{red}BCQ~($\gamma$)} & {\color{blue}TD3+BC} & {\color{red}TD3+BC~($\gamma$)} & {\color{blue}COMBO} & {\color{red}COMBO~($\gamma$)}\\
         \hline
         \hline
        walker2d (0 noised traj) & 59.6$\pm$2.7 & 51.5$\pm$3.6 & \textbf{62.0$\pm$3.2} & 52.2$\pm$1.1  & 26.1$\pm$3.2 & \textbf{65.5$\pm$1.7} \\
        walker2d (10 noised traj) & 53.7$\pm$2.5 & 51.8$\pm$1.3 & \textbf{60.9.$\pm$1.2} & 45.7$\pm$4.2  & 27.9$\pm$2.3 & \textbf{63.1$\pm$1.6} \\
        walker2d (50 noised traj) & 20.3$\pm$3.3 & \textbf{52.4$\pm$3.9} & 4.3$\pm$1.2 & \textbf{46.8$\pm$1.9} & 27.2$\pm$1.6 & \textbf{69.6$\pm$1.9} \\
        walker2d (100 noised traj) & 18.6$\pm$1.9 & \textbf{52.1$\pm$2.2} & 2.1$\pm$0.2 & \textbf{46.6$\pm$1.3} & 13.3$\pm$1.1 & \textbf{70.7$\pm$2.3} \\
        \hline
        hopper (0 noised traj) & 52.8$\pm$2.1 & 40.3$\pm$2.5 & \textbf{52.5$\pm$1.8} & 51.0$\pm$0.9 & 1.5$\pm$0.1 & \textbf{53.5$\pm$3.2} \\
        hopper (10 noised traj) & 47.9$\pm$2.1 & 41.0$\pm$2.7 & 15.4$\pm$0.5 & \textbf{47.9$\pm$0.3}  & 1.2$\pm$0.1 & \textbf{56.5$\pm$2.5} \\
        hopper (50 noised traj) & 12.7$\pm$3.5 & \textbf{44.1$\pm$1.9} & 3.0$\pm$0.2 & \textbf{47.0$\pm$0.5} & 1.0$\pm$0.1 & \textbf{48.6$\pm$4.2}\\
        hopper (100 noised traj) & 1.0$\pm$0.1 & \textbf{41.6$\pm$0.6} & 1.5$\pm$0.4 & \textbf{46.3$\pm$0.7} & 1.3$\pm$0.1 & \textbf{52.3$\pm$1.7}\\
        \hline
        halfcheetah (0 noised traj) & 40.2$\pm$1.3 & \textbf{42.1$\pm$1.1} & 45.3$\pm$1.5 & \textbf{46.9$\pm$1.6} & 32.6$\pm$1.6 & 27.6$\pm$1.5 \\
        halfcheetah (10 noised traj) & 39.5$\pm$0.3 & \textbf{40.2$\pm$3.3} & 45.7$\pm$0.4 & \textbf{47.3$\pm$1.6} & 32.3$\pm$2.8 & 29.7$\pm$2.7 \\
        halfcheetah (50 noised traj) & 36.5$\pm$0.9 & \textbf{37.8$\pm$0.8} & 45.9$\pm$0.3 & \textbf{47.3$\pm$1.3} & 31.1$\pm$4.7 & 28.0$\pm$1.6 \\
        halfcheetah (100 noised traj) & 35.4$\pm$1.1 & \textbf{36.4$\pm$1.7} & 47.3$\pm$1.0 & \textbf{46.1$\pm$1.8} & 30.0$\pm$1.9 & 29.3$\pm$0.6 \\
        \hline
    \end{tabular}
    \caption{Experimental results on noised D4RL tasks, such as walker2d ($x$), hopper ($x$) and halfcheetah ($x$), containing 50 medium and $x$ noised trajectories.
    We conduct a comparison between \textcolor{blue}{original} and \textcolor{red}{lower} guidance discount factor~(e.g., \textcolor{blue}{BCQ} and \textcolor{red}{BCQ ($\gamma$)}).}
    \label{tab1: performance}
\end{table*}

\section{Empirical Results}
\label{experiments}
In this section we examine the role of discount factor through various experiments and aim to investigate the following questions:
1. How effective is the regularization effect of $\gamma$ and how other factors affect its performance?
2. How effective is the pessimistic effect of $\gamma$ and how does it contribute to the performance? 
3. Is a lower guidance discount factor effective and essential in practical offline settings?
We answer the questions above by experiments on both tabular MDPs and D4RL tasks. Each experiment result is averaged over three random seeds with the standard deviation.

\subsection{Regularization Effect}

\subsubsection{Tabular experiments}\label{regular toy example}
We first adopt the BCQ-style offline method in a tabular MDP environment to investigate the effectiveness of the lower discount factor.
We consider a random MDP, where the state space consists of a 30 $\times$ 30 grid, and each state has 10 actions.
The reward and the transition probabilities are generated randomly.
Given the tabular MDP, we compute the true optimal value $Q^*_{\gamma_e}$, where $\gamma_e=0.95$.
Then, we calculate the behavior policy according to $\mu(\cdot\mid s) = \text{softmax}( Q^*_{\gamma_e}(s, \cdot))\cdot \text{mask}(s,\cdot)$, where $\text{mask}(s,\cdot)$ is randomly selected in state-action space to approximate the unseen pairs in offline tasks.
Further, we calculate $\hat{\mu}(\cdot\mid s) = \text{softmax}( Q^*_{\gamma_e}(s, \cdot))\cdot \overline{\text{mask}}(s,\cdot)$ to approximate the generative model in BCQ, where $\|\overline{\text{mask}}\|_1 - \|\text{mask}\|_1 > 0$.
The $\text{Noise Ratio} = \frac{\|\overline{\text{mask}}\|_1 - \|\text{mask}\|_1}{\|\text{mask}\|_1} * 100\%$ represents the inaccuracy of the generative model.
We calculate BCQ-style value function $\hat{Q}_{\gamma}$ by constraining the maximization operator into the finite state-action space:
\begin{equation}
    \hat{Q}(s,a) = r(s,a) + \gamma \max_{a'\ \text{s.t.}\ \hat{\mu}(a'\mid s')>0}\hat{Q}(s', a')
\end{equation}

In the experiments, the proportion of masked state-action pairs is 0.5, and the noise ratio coefficients are \{4\%, 6\%, 8\%, 12\%\} respectively.
We compute the estimation error $\|(\hat{Q}_{\gamma} - Q^*_{\gamma_e})\|_{\infty}$ at seen state-action pairs for different noise ratio coefficients and different discount factors. The results are summarized in Figure~\ref{fig6: bcq}.
Each plot shows the average estimation error across 100 MDP instances.
The experimental results demonstrate that the lower guidance discount factor significantly reduces the estimation error in the tabular offline task.
Moreover, as the noise ratio increases, the optimal discount factor $\gamma^*$ marked by the star shapes becomes smaller. This indicates that discount regularization is more significant when function approximation error is large, which may result from insufficient data or poor data coverage.

\begin{table*}[t]
    \centering
    \begin{tabular}{|c|c|c|c|c|c|}
        \hline
        Halfcheetah & random-v2 & medium-v2 & medium-expert-v2 & medium-replay-v2 & expert-v2 \\
         \hline
        SAC-N ($\gamma$=0.95) & \textbf{30.0$\pm$1.6} & \textbf{65.1$\pm$0.9} & \textbf{51.4$\pm$2.2} & \textbf{28.1$\pm$1.2} & \textbf{82.7$\pm$0.8} \\
        \hline
        SAC-N ($\gamma$=0.99) & 26.6$\pm$1.5 & 48.7$\pm$1.3 & 26.7$\pm$1.1 & 0.6$\pm$0.1 & 80.2$\pm$0.6 \\
        \hline
        \hline
        Hopper & random-v2 & medium-v2 & medium-expert-v2 & medium-replay-v2 & expert-v2 \\
         \hline
        SAC-N ($\gamma$=0.95) & 8.4$\pm$1.7 & \textbf{22.4$\pm$2.1} & \textbf{23.1$\pm$1.9} & 15.5$\pm$3.2 & \textbf{14.5$\pm$2.6} \\
        \hline
        SAC-N ($\gamma$=0.99) & 14.5$\pm$3.5 & 7.1$\pm$2.0 & 15.4$\pm$1.4 & 100.9$\pm$0.3 & 2.3$\pm$0.3 \\
        \hline
    \end{tabular}
    \caption{Experimental results on Halfcheetah and Hopper tasks in D4RL, where the Q-ensemble size N is 2 in Halfcheetah tasks and N is 50 in Hopper tasks.}
    \label{tab2: performance}
\end{table*}

\begin{table*}[h]
    \centering
    \begin{tabular}{|c|c|c|c|c|}
        \hline
        Adroit & pen-expert-v0 & door-expert-v0 & hammer-expert-v0 & relocate-expert-v0 \\
         \hline
        SAC-N (lower $\gamma$) & \textbf{97.1$\pm$3.2} & \textbf{106.4$\pm$1.9} & \textbf{100.6$\pm$2.3} & 0.5$\pm$0.1 \\
        \hline
        SAC-N ($\gamma$=0.99) & 3.6$\pm$1.1 & 2.2$\pm$0.2 & 65.5$\pm$4.2 & 0.4$\pm$0.1 \\
        \hline
    \end{tabular}
    \caption{Experimental results on Adroit tasks in D4RL, where the Q-ensemble size N is 50. We select $\gamma=0.95$ in pen-expert-v0, hammer-expert-v0 and relocate-expert-v0 tasks. We select lower $\gamma=0.9$ in door-expert-v0 tasks.}
    \label{tab4: performance}
\end{table*}
\subsubsection{Experimental results on D4RL tasks}
This section investigates the regularization effect in complex tasks.
Specifically, we evaluate various offline RL algorithms~(TD3+BC~\citep{fujimoto2021td3}, BCQ~\citep{fujimoto2019off} and COMBO~\citep{yu2021combo}) with a lower guidance discount factor on the limited and noised D4RL benchmark of MuJoCo tasks. To test performance in low data regimes and poor coverage scenarios, the training dataset in our experiments contains 50 medium trajectories and additional noisy trajectories ranging from 0 to 100. The noised trajectories are fragments of the random datasets in D4RL.
We use the author-provided implementation or the recognized code to ensure a fair and identical experimental evaluation across algorithms.
The experimental results are shown from two aspects: coverage ratio and data size.

\textbf{Coverage Ratio.}
We report the final performance of TD3+BC with different amounts of noisy data in Table~\ref{tab1: performance} and the detailed training curves in Appendix~\ref{complete results of TD3+BC}.
The training datasets contain 50 medium and $x$ noised trajectories, with $x$ ranging from 0 to 100. The noised trajectories are fragments from the random dataset.
Results show that the performance of current offline RL methods with original discount facotr~($\gamma=0.99$) degrades with more noisy data due to the poor coverage ratio in the dataset.
In contrast, offline RL methods with a lower guidance discount factor~($\gamma=0.95$) achieve stable and robust performance in most scenarios.

Further, the generated data of model-based offline RL is usually noisy since the challenges of the limited data~(please see the performance gap between COMBO and COMBO~($\gamma$) in walker2d and hopper task).

Most scenarios in Table~\ref{tab1: performance} adopt $\gamma=0.95$ as a lower discount factor other than BCQ~($\gamma=0.9$) and COMBO~($\gamma=0.9$) in hopper tasks.

\textbf{Data Size.}
We evaluate TD3+BC on datasets containing $x$ medium and 100 noised trajectories, with $x$ ranging from 0 to 1000. We select the optimal discount factor $\gamma^* \in [0.89, 0.99]$.
Experimental results in Figure~\ref{fig4: data size} show that the optimal discount factor $\gamma^*$ increases with the number of trajectories.
That is, the effect of the discount factor is more significant when the size of the dataset is small, which is consistent with the analysis in Theorem~\ref{theorem:1}.
The experimental results also suggest we prefer a higher discount factor in large datasets~(e.g., the standard datasets in D4RL tasks).

\textbf{Sparse Reward.}
We evaluate EVL~\cite{ma2021offline} with various $\gamma$ on standard Antmaze tasks in D4RL, where the stitching (approximate dynamical programming) is necessary.
Further, Antmaze is a sparse reward task that tightens the bound in Lemma~\ref{lemma:2} up to a $1/(1-\gamma)$ factor, making the optimal $\gamma$ larger.
However, a lower $\gamma$ still works better than the usual $0.99$ even in large tasks~(Please refer to experimental results in Table~\ref{antmaze_IQL}).
These results show that the regularization effects and finetuning $\gamma$ work generally, even in sparse reward tasks.
(In the medium and large tasks, we set the terminal reward $r_T$ as 100 and 50, and the expectile ratio is 0.96.)

\begin{table}[H]
    \centering
    \begin{tabularx}{0.5\textwidth}{|c|c|c|X|}
    \hline
    play-v0 & umaze & medium~($r_T$=50) & large~($r_T$=100) \\
    \hline
     $\gamma$=0.98 & 86.7$\pm$2.5 & 73.6$\pm$2.3 & \textbf{45.7}$\pm$2.2 \\
    \hline
    $\gamma$=0.99 & 87.4$\pm$2.6 & 71.0$\pm$1.8 & 37.8$\pm$3.0 \\
    \hline
    \end{tabularx}
    \caption{Results on Antmaze-play-v0 tasks with EVL.}
    \label{antmaze_IQL}
\end{table}

\subsection{Pessimistic Effect}

We conduct experiments on both tabular MDPs and D4RL tasks with a lower discount factor and no other offline regularization to investigate the pessimistic effect.

\textbf{Tabular MDPs.} We adopt the same setting and evaluation metric as the toy example in Section~\ref{regular toy example}.
In the experiments, we define the coverage ratio as the proportion of masked state-action pairs ranging from 50\% to 90\%.
The experimental results in Figure~\ref{fig6: bcq} show that a lower discount factor promotes offline RL algorithms to have a better estimation, and the effect is more significant when the data coverage is poor. 

\begin{figure}[t]
    \centering
    \subfigure{
    \includegraphics[scale=0.35]{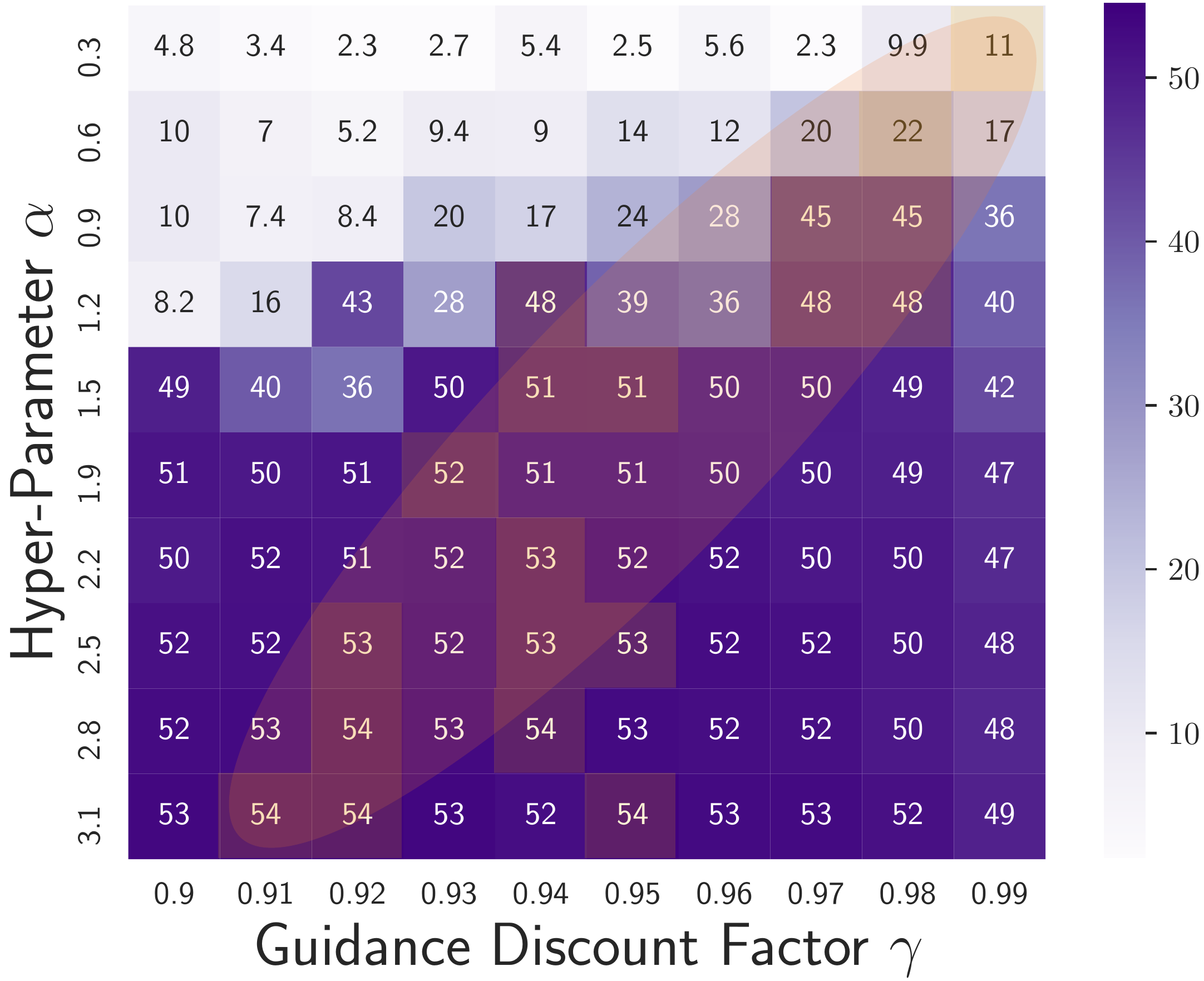}}
    \caption{Relationship between $\gamma$ and $\alpha$ of TD3+BC on halfcheetah task.
    The value is the normalized return metric.
    The optimal $\gamma^*$ is marked with orange color.}
    \label{fig5: heatmap}
    \vspace{-0.5cm}
\end{figure}

\begin{figure*}[h]
    \centering
    \subfigure{
    \includegraphics[scale=0.26]{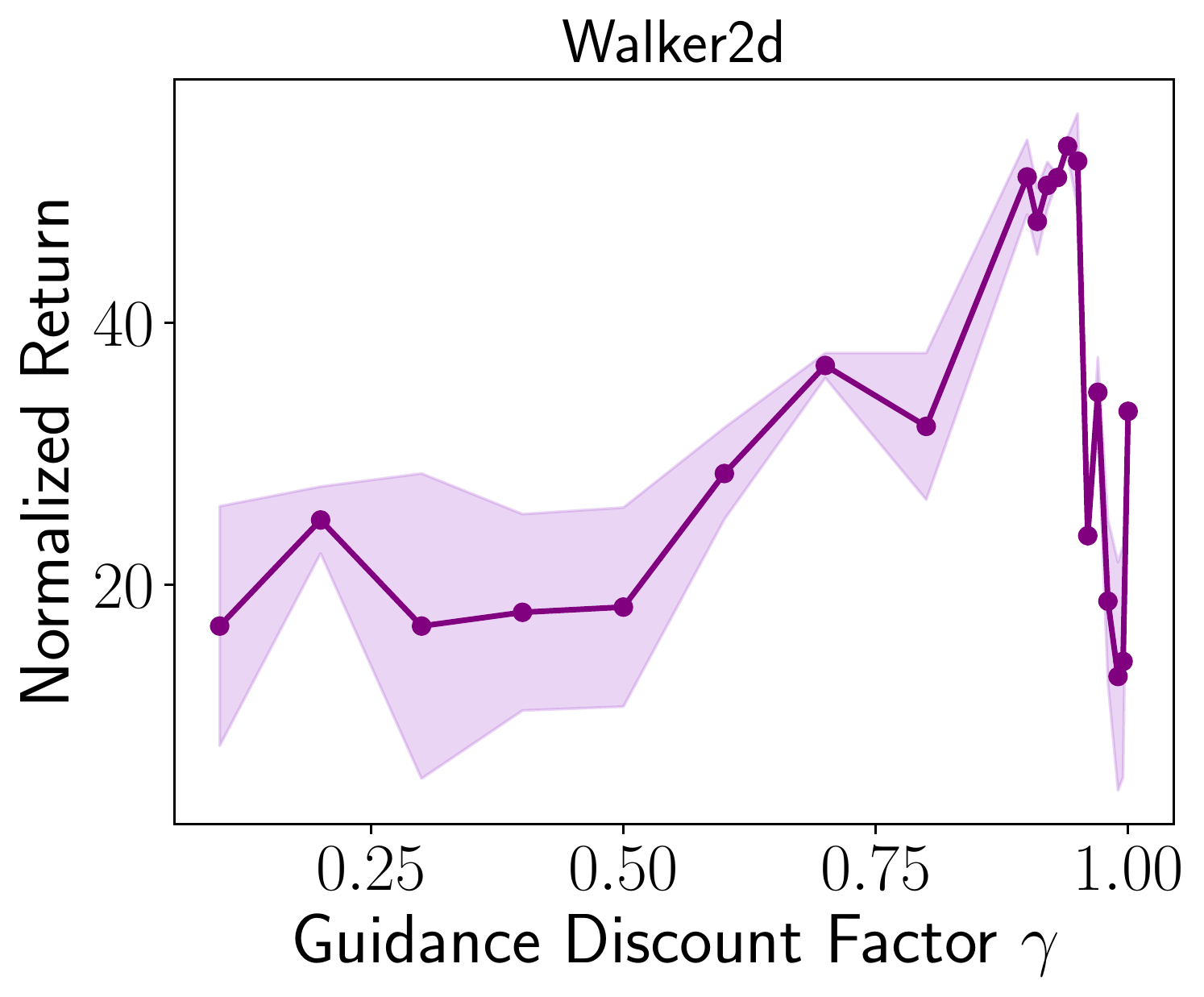}}
    \hspace{2.5mm}
    \subfigure{
    \includegraphics[scale=0.26]{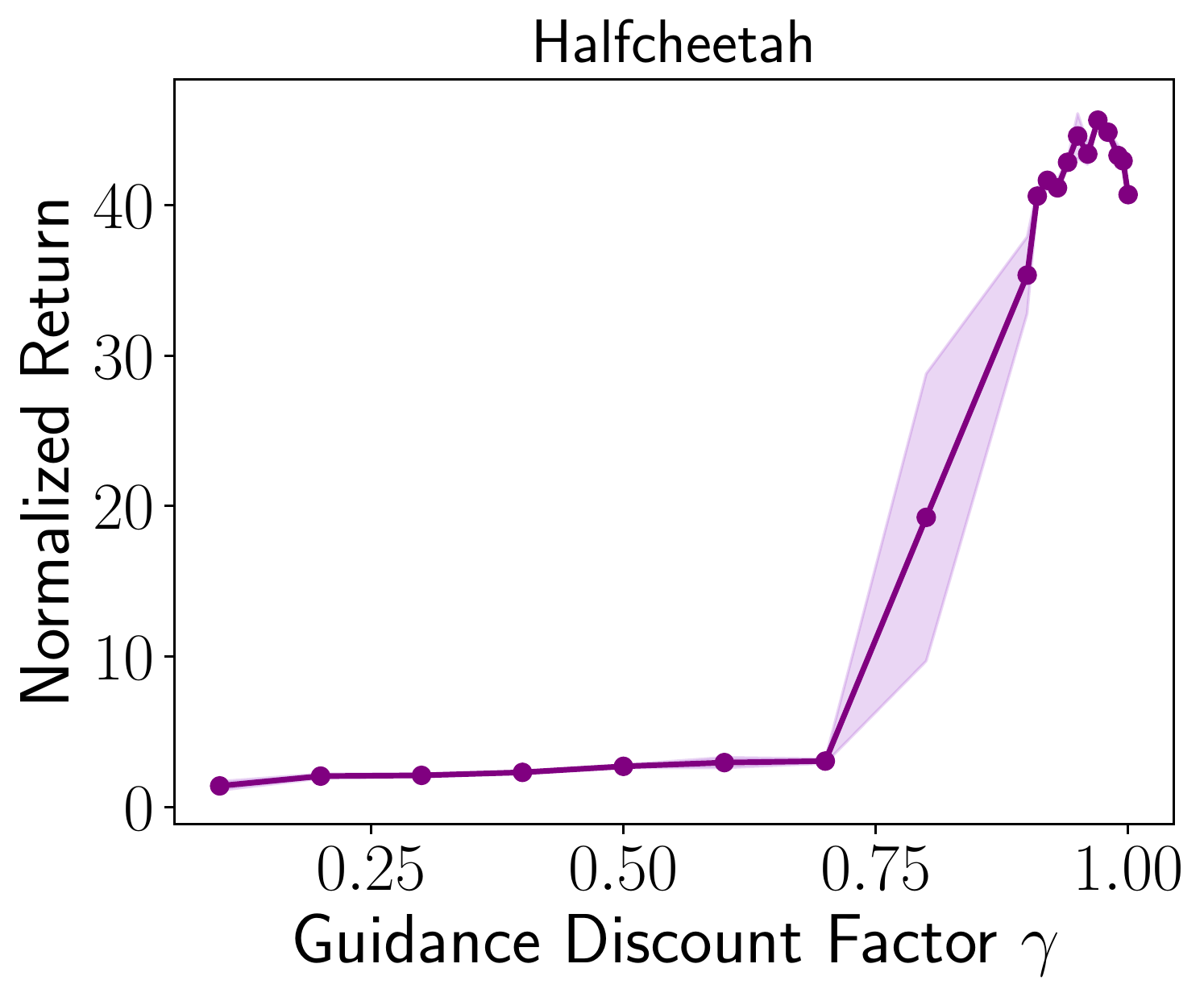}}
    \hspace{2.5mm}
    \subfigure{
    \includegraphics[scale=0.26]{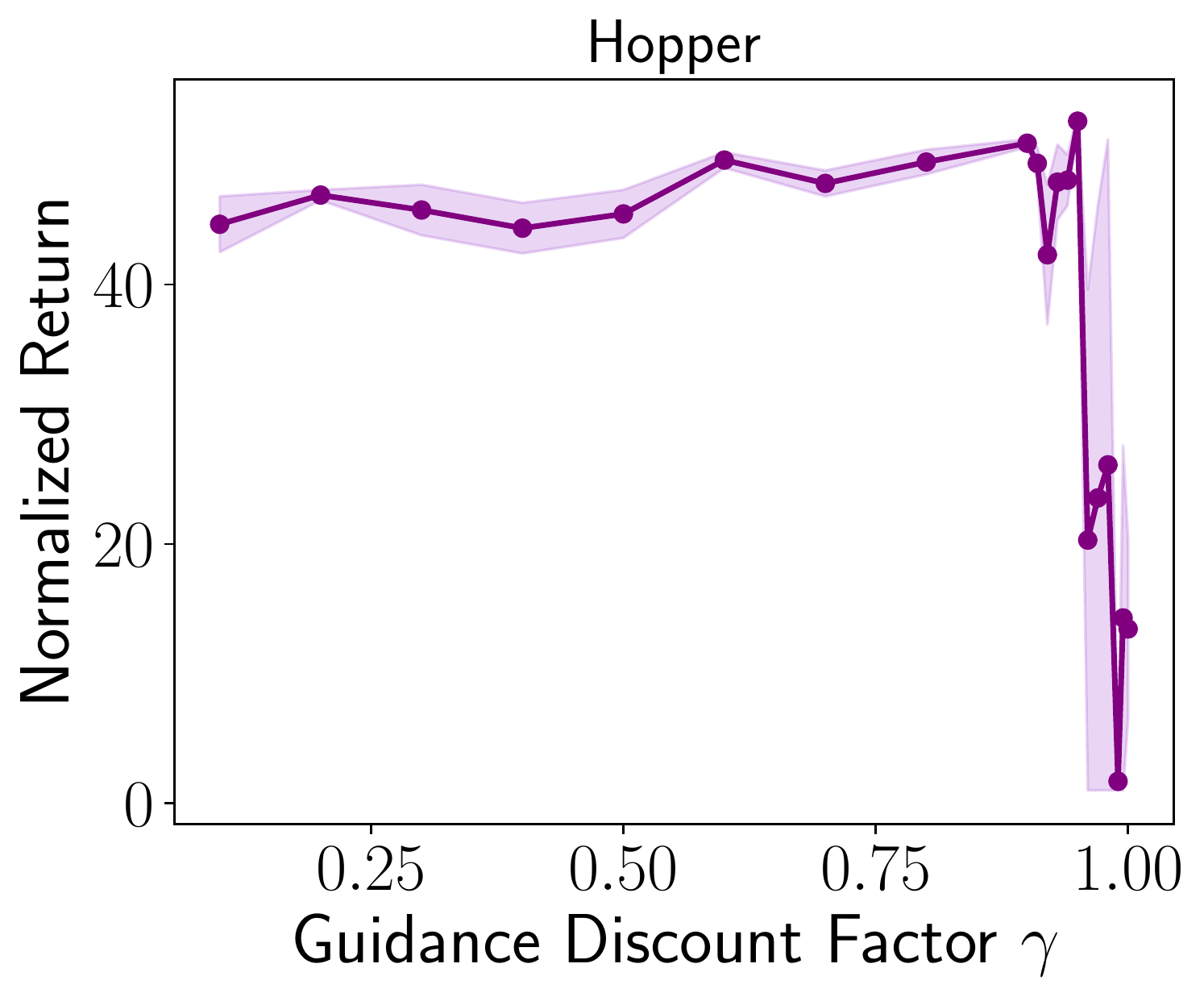}}
    \caption{Performance of TD3+BC on noised D4RL tasks containing 50 medium and 25 noised random trajectories.
    We adopt the normalized return metric proposed by D4RL benchmark~\citep{fu2020d4rl}.
    Scores roughly range from 0 to 100.
    }
    \label{fig3: gamma}
\end{figure*}

\begin{figure*}[h]
    \centering
    \subfigure{
    \includegraphics[scale=0.26]{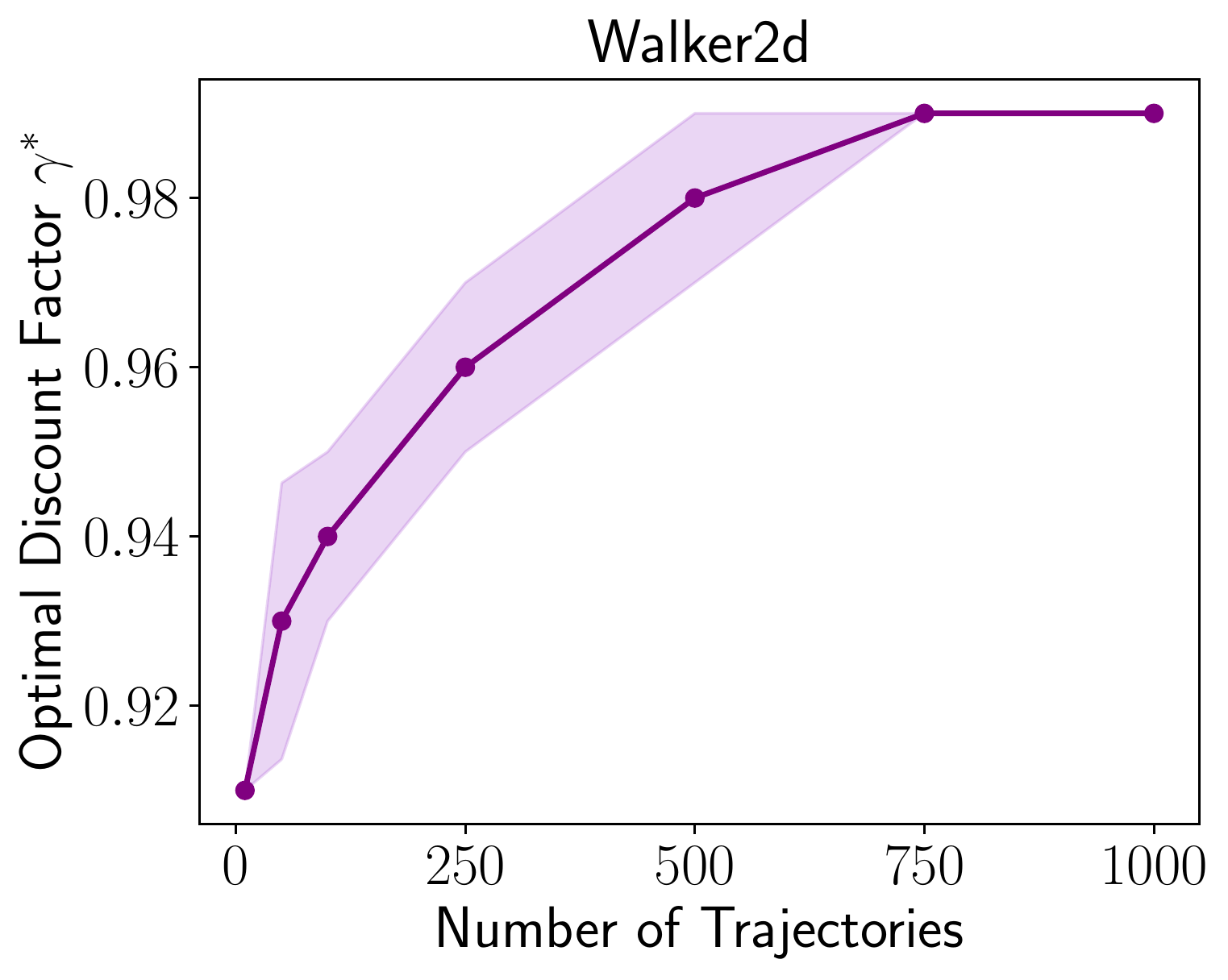}}
    \hspace{2.5mm}
    \subfigure{
    \includegraphics[scale=0.26]{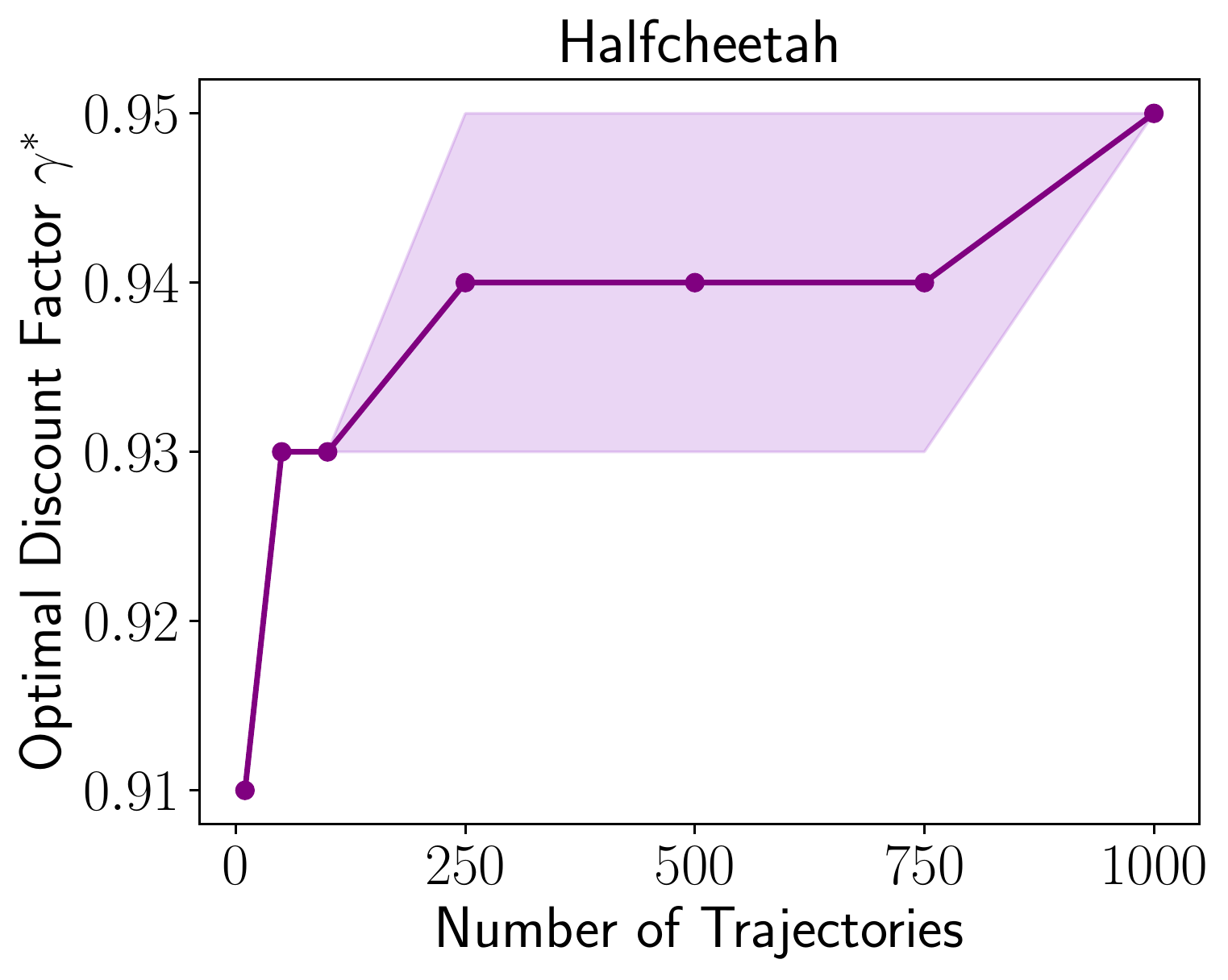}}
    \hspace{2.5mm}
    \subfigure{
    \includegraphics[scale=0.26]{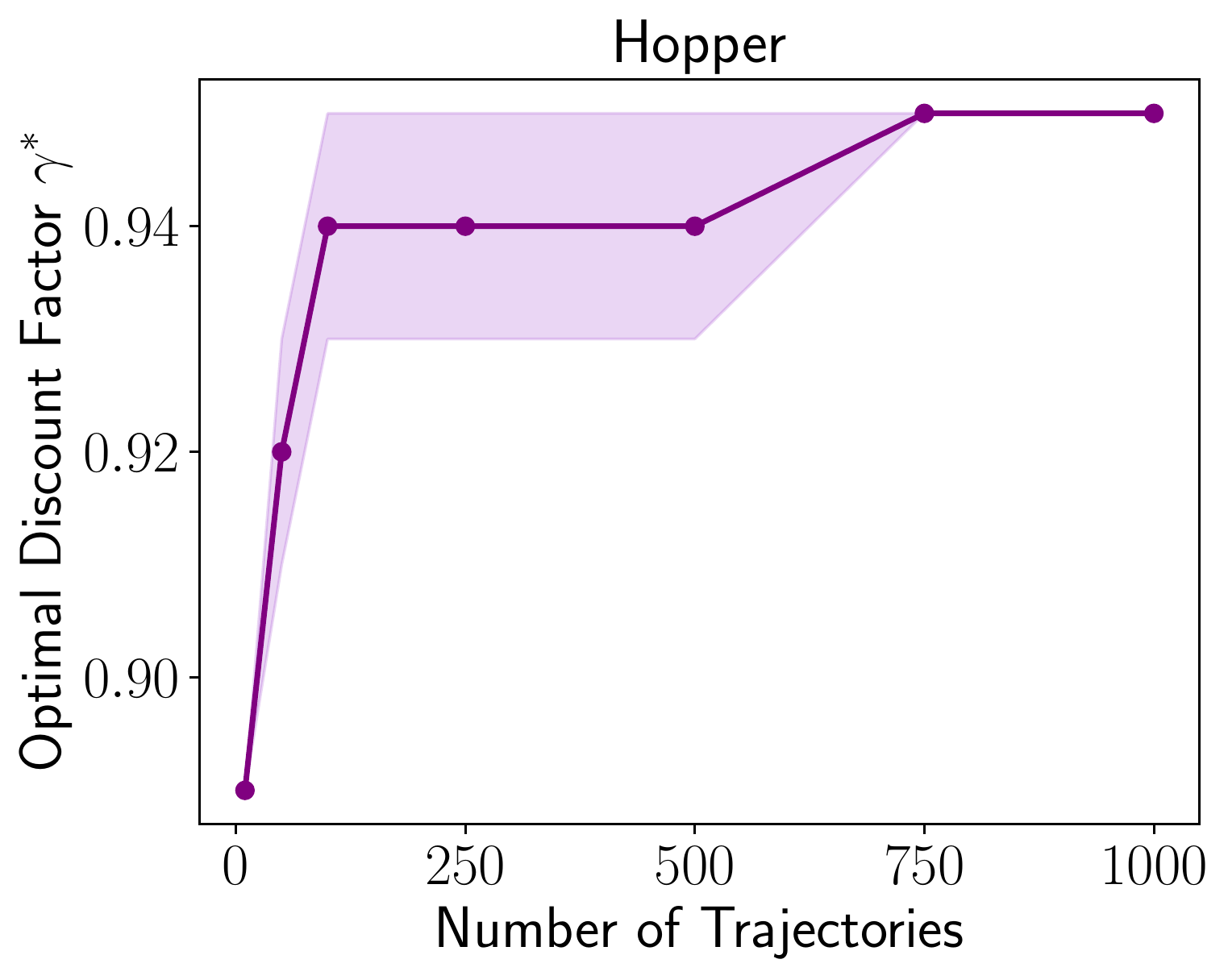}}
    \caption{The relationship between optimal discount factor $\gamma^*$ and datasize on noised D4RL tasks.
    }
    \label{fig4: data size}
\end{figure*}

\textbf{D4RL Task.} 
We evaluate the standard off-policy algorithm SAC with ensemble networks on standard D4RL datasets, which is shown in Table~\ref{tab2: performance} and Table~\ref{tab4: performance}.
Adroit tasks require dynamic programming to infer the complete action sequence for finishing the task, while a lower $\gamma$ works surprisingly well with $Q$-ensemebles.
Moreover, the experimental results on mujoco tasks show that SAC achieves a noticeable performance with a lower guidance discount factor, demonstrating the role of the discount factor as a pessimistic mechanism.
While the pessimism of lower $\gamma$ is not as good as the state-of-the-art offline algorithms, it serves as a good baseline. 
It significantly surpasses online algorithms with no discount factor pessimism. 
The pessimistic mechanism of the discount factor is coarse, in the sense that it is a single parameter that affects all state-action pairs, which explains why the pessimism of lower $\gamma$ is not as good as the state-of-the-art offline algorithms. 
We leave finding methods for more fine-grained discount factor control as interesting future work.

\subsection{Effectiveness of Discount Regularization}
\subsubsection{Sensitivity of the discount factor}
This section tests the performance sensitivity regarding $\gamma$ and whether we need to fine-tune $\gamma$ for each task.
To this end, we evaluate TD3+BC on datasets containing 50 medium and 25 noised trajectories with various $\gamma$.
In Figure~\ref{fig3: gamma}, TD3+BC achieves a high performance by decreasing $\gamma$ from 0.99 to around 0.95.
Different offline RL algorithms have minor differences in selecting $\gamma$ on various tasks~(e.g., the particular case in Table~\ref{tab1: performance}).
The suitable area of $\gamma$ is almost identical [0.95, 0.99].

\subsubsection{Discussion between $\gamma$ and $\alpha$ in TD3+BC}
Many offline RL algorithms achieve a trade-off between conservation and generalization by some other parameter (e.g., the regularized hyper-parameter $\alpha$ in TD3+BC). This naturally leads to the following question: Is finetuning $\gamma$ more effective than simple regularizations like behavior cloning?
To answer this question, we evaluate TD3+BC on the halfcheetah task containing 100 medium and random trajectories fragments.
Experimental results in Figure~\ref{fig5: heatmap} show that by properly combining $\alpha$ and $\gamma$, we can solve tasks more effectively. This indicates that $\gamma$ offers more flexibility than the original regularization. Moreover, note that the behavior cloning performance in this task is poor~(the normalized score is 0.0). Therefore, the role of the lower discount factor is not equivalent to behavior cloning.

Further, the experimental results also suggest that other conservative regularization affects the optimal guidance discount factor~(shown in orange in Figure~\ref{fig5: heatmap}). In general, the stronger the other regularization, the higher the optimal discount factor, matching our intuition that there is a trade-off between two regularizations.

\section{Conclusion}
This paper examines the two distinct effects of the discount factor in offline RL, i.e., the regularization effect and the pessimistic effect. On the one hand, the discount factor acts as a regulator to trade-off optimality with sample efficiency upon existing offline techniques like a negative uncertainty-based bonus. On the other hand, we show that a lower guidance discount factor is equivalent to the model-based pessimism, where we optimize the policy's performance in the worst possible models. We quantify the above effects by analyzing the performance bounds with a lower guidance discount factor in linear MDPs. Moreover, we verify the above theoretical observations in tabular MDPs and D4RL tasks. Empirical results show that a lower discount factor can significantly improve performance in two scenarios. The first is when the data size is limited or there is poor data coverage, and we can apply a lower discount factor with additional pessimism. The second is when the data is sufficient, and we can directly use a lower discount factor as a proper pessimistic mechanism. 

Our work suggests that the discount factor plays a vital role in offline RL, promoting current offline RL methods in complex and diverse scenarios. This leaves several interesting future works: (i) Can a lower guidance discount factor be better integrated with current offline algorithms, like fine-grained guidance $\gamma$ at each transition? (ii) Is there a better theoretical explanation for the success of a lower discount factor in offline RL? (iii) How to develop more efficient offline algorithms in limited data size and insufficient coverage scenarios?


\section{Acknowledgements}
This work is supported in part by Science and Technology Innovation 2030 - ``New Generation Artificial Intelligence'' Major Project (No. 2018AAA0100904) and National Natural Science Foundation of China (62176135).

\nocite{*}

\bibliography{reference}
\bibliographystyle{icml2022}

\newpage
\appendix
\onecolumn

\section{Algorithm Details}
In this section, we describe some details of the algorithms mentioned in Section~\ref{algo_meta}.
\label{algo_describ}
\subsection{Pessimistic Value Iteration}
\label{algo_describ_pvi}
In linear MDPs, we can construct $\hat\BB_\gamma\hat{V}$ and $\Gamma$ based on $\cD$ as follows, where $\hat\BB_\gamma\hat{V}$ is the empirical estimation for $\BB_\gamma\hat{V}$. For a given dataset $\cD=\{(s_\tau,a_\tau,r_\tau)\}_{\tau=1}^{N}$, we define the empirical mean squared Bellman error (MSBE) as
\begin{equation*}
M(w) = \sum_{\tau=1}^N \bigl(r_\tau + \gamma \widehat{V}(s_{\tau+1}) - \phi (s_\tau,a_\tau)^\top w\bigr)^2 + \lambda \norm{w}_2^2
\end{equation*}
Here $\lambda>0$ is the regularization parameter. Note that $\hat{w}$ has the closed form
\#\label{eq:w18}
&\hat{w} =  \Lambda ^{-1} \Big( \sum_{\tau=1}^{N} \phi(s_\tau,a_\tau) \cdot \bigl(r_\tau + \gamma\hat{V}(s_{\tau+1})\bigr) \Bigr ) , \notag\\
&\text{where~~} \Lambda = \lambda I+\sum_{\tau=1}^N \phi(s_\tau,a_\tau)  \phi(s_\tau,a_\tau) ^\top. 
\#
Then we simply let 
\#
\label{eq:empirical_bellman}
\hat\BB_\gamma\hat{V}=\langle\phi,\hat w  \rangle .
\#
Meanwhile, we construct $\Gamma$ based on $\cD$ as 
\#\label{eq:linear_uncertainty_quantifier}
\Gamma(s, a) = \beta\cdot \big( \phi(s, a)^\top  \Lambda ^{-1} \phi(s, a)  \big)^{1/2}.
\#
Here $\beta>0$ is the scaling parameter.
\subsection{Model-based Pessimistic Policy Optimization}
\label{algo_describ_mpo}
To give a proper performance bound, we consider the following model set
\begin{equation}
    \label{MLE_model_ellipsoid}
    \cM_{\cD}=\left\{P(\cdot|s,a) \in \cM \biggiven \EE_{\cD}\left[\TV (\widehat{P}(\cdot | s,a), P(\cdot | s,a))^2\right] \leq \xi\right\},
\end{equation}
where $\widehat{P}=\argmax_{P}\EE_{\cD}[\ln P(s'\mid s,a)]$ and $\cM$ is the set of linear models. In practice, we can parameterize the model $P_\theta(\cdot|s,a)$ and train the model via the following loss
\begin{equation}
    \label{MLE_loss_ellipsoid}
    \cL(\theta,\cD)=\frac{1}{N}\sum_{\tau=1}^N \ln P_\theta (s_{\tau+1}|s_\tau,a_\tau).
\end{equation}
When assuming the transitions are Gaussian, the MLE objective can be reduced to prediction loss as follows
\begin{equation}
    \label{prediction_loss_ellipsoid}
    \cL(\theta,\cD)=\frac{1}{N}\sum_{\tau=1}^N \| f_\theta(s_\tau,a_\tau) - s_{\tau+1}\|_2.
\end{equation}
As to the minimax optimization in \eqref{alg:2_1}, we can use techniques like bi-level optimization~\citep{hong2020two} to get the approximate solution.
\section{Addtional Lemmas and Missing Proofs}
\subsection{Proof of Lemma~\ref{lemma:2}}
\label{proof_lemma_1}
\begin{proof}
    For a sufficiently large $\lambda$, it is easy to see that $\mathcal{T}\hat V \coloneqq \max_a (\hat \BB_\gamma \hat{V}-\Gamma)$ is a contraction. Without loss of generality, we assume $\lambda =1$. Then Algorithm~\ref{alg:1} converges and we have 
    \begin{align*}
        &\hat{V}(\cdot) ~~= \max_a {\hat{Q}(\cdot,a)}, \\
        &\hat{Q}(\cdot,\cdot) = \hat \BB_\gamma \hat{V}-\Gamma(\cdot,\cdot).
    \end{align*}
    From the definition of $\delta(\cdot,\cdot)$, we have 
    \begin{align}
        \delta(s,a)=\BB_\gamma \hat{V}(s) - \hat{Q}(s,a) = \BB_\gamma \hat{V}(s) - \hat \BB_\gamma \hat{V}+\Gamma(s,a).
    \end{align}
    Under the condition of Lemma~\ref{lemma:xi_quantifier}, it holds that 
    \begin{align}
        0\leq\delta(s,a)\leq 2\Gamma(s,a), \text{for all}~s,a. \label{gamma_inequality}
    \end{align}
    From Lemma~\ref{lemma:subopt_decompose}, we have 
    \begin{align}
        \text{SubOpt}\big(\widehat{\pi},s;\gamma \big) =& - \EE_{\hat{\pi}}\left[\sum_{t=0}^\infty{\gamma^t \delta(s_t,a_t)}\Biggiven s_0=s\right] + \EE_{\pi^*}\left[\sum_{t=0}^\infty{\gamma^t \delta(s_t,a_t)}\Biggiven s_0=s\right]\nonumber\\
        &+ \EE_{\pi^*}\left[\sum_{t=0}^\infty{\gamma^t \innerprod{
            \hat{Q}(s_t,\cdot),\pi^*(\cdot|s_t)-\hat{\pi}(\cdot|s_t)
            }}\Biggiven s_0=s\right]\nonumber\\
            \leq &  - \EE_{\hat{\pi}}\left[\sum_{t=0}^\infty{\gamma^t \delta(s_t,a_t)}\Biggiven s_0=s\right] + \EE_{\pi^*}\left[\sum_{t=0}^\infty{\gamma^t \delta(s_t,a_t)}\Biggiven s_0=s\right]\nonumber\\
            \leq&  2 \EE_{\pi^*}\Bigl[\sum_{t=0}^\infty \gamma^t \Gamma(s_t,a_t) \Biggiven s_0=s\Bigr]\nonumber\\
            = &  2 \beta \EE_{\pi^*}\Bigl[\sum_{t=0}^\infty \gamma^t \bigl(\phi(s_t,a_t)^\top \Lambda^{-1}\phi(s_t,a_t)\bigr)^{1/2} \Biggiven s_0=s\Bigr].
    \end{align}
    Here the first inequality follows from the fact that $\hat{\pi}(\cdot|s) = \argmax_\pi \innerprod{\hat{Q}(\cdot,\cdot),\pi(\cdot|s)}$ and the second inequality follows from Equation~\eqref{gamma_inequality}. 
    Then the following event
    \begin{align}
        \label{eq:def_ce}
        \cE = \bigg\{\text{SubOpt}\big(\widehat{\pi},s;\gamma \big) \leq 2 \beta \EE_{\pi^*}\Bigl[\sum_{t=0}^\infty \gamma^t \bigl(\phi(s_t,a_t)^\top \Lambda^{-1}\phi(s_t,a_t)\bigr)^{1/2} \Biggiven s_0=s\Bigr]\text{ for all }s\in \cS\bigg\}
    \end{align}
    holds with probability $1-\xi/2$. From the assumption in Equation~\eqref{eq:event_opt_explore}, the following event 
    \begin{align*}
        \cE^\dagger =  \bigg\{c^\dagger \cdot  \frac{1}{N}\sum_{\tau=1}^N{\phi(s_\tau,a_\tau)\phi(s_\tau,a_\tau)^\top}\succeq  \EE_{\pi^*}\bigl[\phi(s_t,a_t)\phi(s_t,a_t)^\top\biggiven s_0=s\bigr] ~ \text{for all }s\in \cS\bigg\}
    \end{align*}
    also holds with probability $1-\xi/2$. Then from the union bound, the event $\cE\cap\cE^\dagger$ holds with probability $1-\xi$. We condition on this event here after.
    By the Cauchy-Schwarz inequality, we have
    \begin{align}
        \label{eq:bound_eigen}
        &\EE_{\pi^*}\Bigl[ \sum_{t=0}^\infty \gamma^t \bigl(\phi(s_t,a_t)^\top \Lambda^{-1}\phi(s_t,a_t)\bigr)^{1/2} \Biggiven s_0=s\Bigr]\notag \\
        &\qquad = \frac{1}{1-\gamma}\EE_{d^{\pi^*}}\Bigl[ \sqrt{\Tr\big(\phi(s,a)^\top \Lambda^{-1}\phi(s,a)\big)} \Biggiven s_0=s\Bigr]\notag \\
        &\qquad = \frac{1}{1-\gamma}\EE_{d^{\pi^*}}\Bigl[ \sqrt{\Tr\big(\phi(s,a)\phi(s,a)^\top \Lambda^{-1}\big)} \Biggiven s_0=s\Bigr]\notag \\
        &\qquad \leq  \frac{1}{1-\gamma}\sqrt{\Tr\Big(\EE_{d^{\pi^*}}\big[\phi(s,a)\phi(s,a)^\top \biggiven s_0=s\big]\Lambda^{-1}\Big)} \notag \\
        &\qquad = \frac{1}{1-\gamma}\sqrt{\Tr\Big(\Sigma_{\pi^*,s}^\top \Lambda^{-1}\Big)},
    \end{align}
    for all $s\in \cS$. 
    On the event $\cE \cap \cE^\dagger$, we have
    \begin{align*}
        \text{SubOpt}\big(\widehat{\pi},s;\gamma \big)&\leq 2 \beta \EE_{\pi^*}\Bigl[ \sum_{t=0}^\infty \gamma^t\bigl(\phi(s_t,a_t)^\top \Lambda^{-1}\phi(s_t,a_t)\bigr)^{1/2} \Biggiven s_0=s\Bigr]\\
    &\leq\frac{ 2 \beta }{1-\gamma} \sqrt{\Tr\Big(\Sigma_{\pi^*,s}\cdot  \big(I + \frac{1}{c^\dagger} \cdot N \cdot \Sigma_{\pi^*,s} \big)^{-1}\Big)}\\
    & =\frac{ 2 \beta }{1-\gamma} \sqrt{\sum_{j=1}^d \frac{\lambda_{j}(s)}{1+\frac{1}{c^\dagger}\cdot N \cdot \lambda_{j}(s)}}.
    \end{align*}
    
    Here $\{\lambda_{j}(s)\}_{j=1}^d$ are the eigenvalues of $\Sigma_{\pi^*,s}$ for all $s\in \cS$, the first inequality follows from the definition of $\cE$ in Equation~\eqref{eq:def_ce}, and the second inequality follows from Equation~\eqref{eq:bound_eigen} and the definition of $\cE^\dagger$ in Equation~\eqref{eq:event_opt_explore}.
    Meanwhile, by Definition \ref{assump:linear_mdp}, we have $\|\phi(s,a)\|\leq 1$ for all $(s,a)\in \cS \times \cA$. By Jensen's inequality, we have
    \begin{equation}
        \|\Sigma_{\pi^*,s}\|_{\oper} \leq \EE_{\pi^*}\big[ \|\phi(s,a)\phi(s,a)^\top \|_{\oper}\biggiven s_0=s  \big] \leq 1
    \end{equation}
    
    for all $s\in \cS$. As $\Sigma_{\pi^*,s}$ is positive semidefinite, we have $\lambda_{j}(s) \in [0,1]$ for all $s\in \cS$ and all $j\in [d]$. Hence, on $\cE \cap \cE^\dagger$, we have
    \begin{align*}
        \text{SubOpt}\big(\widehat{\pi}, s;\gamma \big)&\leq \frac{ 2 \beta }{1-\gamma} \sqrt{\sum_{j=1}^d \frac{\lambda_{j}(s)}{1+\frac{1}{c^\dagger}\cdot N \cdot \lambda_{j}(s)}} \\
    &\leq \frac{ 2 \beta }{1-\gamma}  \sqrt{\sum_{j=1}^d \frac{1}{1+\frac{1}{c^\dagger}\cdot N}} \leq   \frac{2c r_{\text{max}}}{(1-\gamma)^2}  \sqrt{c^\dagger d^3\zeta/N}
    \end{align*}
    for all $x\in \cS$, where the second inequality follows from the fact that $\lambda_{j}(s) \in [0,1]$ for all $s\in \cS$ and all $j\in [d]$, while the third inequality follows from the choice of the scaling parameter $\beta > 0$. Then we have the conclusion in Lemma~\ref{lemma:2}.
\end{proof}

\begin{lemma}[Suboptimality Decomposition]
    \label{lemma:subopt_decompose}
    We have
    \begin{align}
        \text{SubOpt}(\hat{\pi},s;\gamma) =& - \EE_{\hat{\pi}}\left[\sum_{t=0}^\infty{\gamma^t \delta(s_t,a_t)}\Biggiven s_0=s\right] + \EE_{\pi^*}\left[\sum_{t=0}^\infty{\gamma^t \delta(s_t,a_t)}\Biggiven s_0=s\right]\nonumber\\
        & + \EE_{\pi^*}\left[\sum_{t=0}^\infty{\gamma^t \innerprod{
        \hat{Q}(s_t,\cdot),\pi^*(\cdot|s_t)-\hat{\pi}(\cdot|s_t)
        }}\Biggiven s_0=s\right],
    \end{align}
    where $\langle f,g\rangle=\int_{a\in\cA}{f(a)g(a) \ud a} .$
\end{lemma}
\begin{proof}
    We have
    \begin{align*}
        \text{SubOpt}(\hat{\pi},s;\gamma) &= V^{{\pi}^*}(s) - V^{\hat{\pi}}(s) = V^{{\pi}^*}(s) - \hat{V}(s)+ \hat{V}(s) - V^{\hat{\pi}}(s).
    \end{align*}
    The first term satisfies
    \begin{align*} 
        \hat{V}(s) - V^{\hat{\pi}}(s) & = \EE_{a\sim\hat{\pi}}\left[\hat{Q}(s,a)\right]- \EE_{a\sim\hat{\pi},s'\sim\cP(\cdot|s,a)}\left[r(s,a)+\gamma V^{\hat{\pi}}(s')\right] \\
        & =  \EE_{a\sim\hat{\pi},s'\sim\cP(\cdot|s,a)}\left[\hat{Q}(s,a)-r(s,a)-\gamma \hat{V}(s')\right] + \gamma\EE_{a\sim\hat{\pi},s'\sim\cP(\cdot|s,a)}\left[ \hat{V}(s') - V^{\hat{\pi}}(s')\right] \\
        & =  \EE_{\hat{\pi}}\left[\delta(s,a)\right]+ \gamma\EE_{a\sim\hat{\pi},s'\sim\cP(\cdot|s,a)}\left[\hat{V}(s')-V^{\hat{\pi}}(s')\right] \\
        & =  \EE_{\hat{\pi}}\left[\delta(s,a)\right]+ \cdots \\
        & =  \EE_{\hat{\pi}}\left[\sum_{t=0}^{\infty}\gamma^t\delta(s_t,a_t)\given s_0=s\right], \\
    \end{align*}
    while the second term statisfies
    \begin{align*} 
        V^{{\pi}^*}(s) - \hat{V}(s) & = \EE_{a\sim{\pi}^*,s'\sim\cP(\cdot|s,a)}\left[r(s,a)+\gamma V^{\hat{\pi}}(s')\right]- \EE_{a\sim\hat{\pi}}\left[\hat{Q}(s,a)\right] \\
        & = \EE_{a\sim{\pi}^*,s'\sim\cP(\cdot|s,a)}\left[r(s,a)+\gamma V^{\hat{\pi}}(s')-\hat{Q}(s,a)\right]+ \EE_{a\sim{\pi}^*}\left[\hat{Q}(s,a)\right] - \EE_{a\sim\hat{\pi}}\left[\hat{Q}(s,a)\right] \\
        & = \EE_{a\sim{\pi}^*,s'\sim\cP(\cdot|s,a)}\left[r(s,a)+\gamma \hat{V}(s')-\hat{Q}(s,a)\right] + \gamma \EE_{a\sim{\pi}^*,s'\sim\cP(\cdot|s,a)}\left[V^{\hat{\pi}}(s') - \hat{V}(s')\right] \\
        & + \innerprod{\hat{Q}(s,\cdot), \pi^*(\cdot\given s)-\hat{\pi}(\cdot\given s)}_\cA \\
        & = -\EE_{a\sim{\pi}^*,s'\sim\cP(\cdot|s,a)}\left[\delta(s,a)\right] + \innerprod{\hat{Q}(s,\cdot), \pi^*(\cdot\given s)-\hat{\pi}(\cdot\given s)}_\cA + \cdots \\
        & = -\EE_{\pi^*}\left[\sum_{t=0}^{\infty}\gamma^t\delta(s_t,a_t)\given s_0=s\right] + \EE_{\pi^*}\left[\sum_{t=0}^{\infty}\gamma^t \innerprod{\hat{Q}(s_t,\cdot), \pi^*(\cdot\given s_t)-\hat{\pi}(\cdot\given s_t)}_\cA \given s_0=s\right]. \\
    \end{align*}
    Combining the two equations above, we have the desired result. 
\end{proof}

\begin{lemma}[$\xi$-Quantifiers]
    \label{lemma:xi_quantifier}
    Let 
    \begin{equation}
        \lambda =1, \quad \beta= c\cdot d V_{\text{max}}\sqrt{\zeta}, \quad \zeta = \log{(2dN/(1-\gamma)\xi)}.
    \end{equation}
    Then $\Gamma=\beta \cdot \big( \phi(s, a)^\top  \Lambda ^{-1} \phi(s, a)  \big)^{1/2}$ specified in Equation~\eqref{eq:linear_uncertainty_quantifier} are $\xi$-quantifiers. That is, with probability at least $1-\xi$,
    \begin{equation}
        |(\BB \widehat{V})(s,a) -(\widehat{\BB} \widehat{V})(s,a)| \leq \Gamma = \beta \sqrt{\phi(s,a)^\top\Lambda^{-1}\phi(s,a)}, \forall (s,a) \in \cS\times\cA.
    \end{equation}
\end{lemma}
\begin{proof}
    we have 
    \begin{align}
        \BB \widehat{V} -\widehat{\BB} \widehat{V} & = \phi(s,a)^\top (w-\widehat{w})\notag\\
        & = \phi(s,a)^\top w - \phi(s,a)\Lambda^{-1}\left(\sum_{\tau=1}^N{\phi_\tau(r_\tau+\gamma \widehat{V}(s_{\tau+1})}\right)\notag\\
        & = \underbrace{\phi(s,a)^\top w - \phi(s,a)\Lambda^{-1}\left(\sum_{\tau=1}^N{\phi_\tau\phi_\tau^\top w}\right)}_{\displaystyle \text{(i)}} +\underbrace{\phi(s,a)\Lambda^{-1}(\sum_{\tau=1}^N{\phi_\tau\phi_\tau^\top w}-\sum_{\tau=1}^N\phi_\tau(r_\tau+\gamma \widehat{V}(s_{\tau+1}))}_{\displaystyle \text{(ii)}}, \label{eq:term1_diff}
    \end{align}
    Then we bound $\text{(i)}$ and $\text{(ii)}$, respectively.

    For $\text{(i)}$, we have
    \begin{align}
        \label{eq:zzz888}
        \text{(i)} &= \phi(s,a)^\top w - \phi(s,a)\Lambda^{-1}(\Lambda-\lambda I)w \notag\\
        &= \lambda \phi(s,a)\Lambda^{-1}w\notag\\
        &\leq \lambda \norm{\phi(s,a)}_{\lambda^{-1}} \norm{w}_{\lambda^{-1}}\notag\\
        &\leq V_{\text{max}} \sqrt{d\lambda} \sqrt{\phi(s,a)^\top\Lambda^{-1}\phi(s,a)},
    \end{align}
    where the first inequality follows from Cauchy-Schwartz inequality. The second inequality follows from the fact that $\norm{\Lambda^{-1}}_{\text{op}}\leq \lambda^{-1}$ and Lemma~\ref{lemma:bounded_weight_value}.

    For notation simplicity, let $\epsilon_\tau = r_\tau+\gamma \widehat{V}(s_{\tau+1}) - \phi_\tau^\top w$, then we have 
    \begin{align}
        \label{eq:define_term3}
        |\text{(ii)}| & = \phi(s,a)\Lambda^{-1}\sum_{\tau=1}^N{\phi_\tau\epsilon_\tau}\notag\\
        & \leq  \norm{\sum_{\tau=1}^N{\phi_\tau\epsilon_\tau}}_{\Lambda^{-1}}\cdot\norm{\phi(s,a)}_{\Lambda^{-1}}\notag\\
        & =  \underbrace{\norm{\sum_{\tau=1}^N{\phi_\tau\epsilon_\tau}}_{\Lambda^{-1}}}_{\text{(iii)}} \cdot \sqrt{\phi(s,a)^\top \Lambda^{-1}\phi(s,a)}.
    \end{align}
    The term $\text{(iii)}$ is depend on the randomness of the data collection process of $\cD$. To bound this term, we resort to uniform concentration inequalities to upper bound
    \begin{equation}
        \sup_{V \in \cV(R,B,\lambda)} \Big\|  \sum_{\tau=1}^{N} \phi(x_\tau,a_\tau) \cdot \epsilon_\tau(V) \Big\|,\notag
    \end{equation}
    where
    \begin{equation}
        \cV(R,B,\lambda) = \{V(s;w,\beta,\Sigma):\mathcal{S}\rightarrow [0,V_{\text{max}}]~\text{with} \norm{w}\leq R, \beta\in[0,B],\Sigma\succeq \lambda\cdot I\},
    \end{equation}
    where $V(s; w,\beta,\Sigma) = \max_a\{\phi(s,a)^\top w-\beta \cdot\sqrt{\phi(s,a)^\top\Sigma^{-1}\phi(s,a)}\}$. For all $\epsilon>0$, let $\cN(\epsilon;R,B,\lambda)$ be the minimal cover if $\cV(R,B,\lambda)$. That is, for any function $V\in\cV(R,B,\lambda)$, there exists a function $V^\dagger\in \cN(\epsilon;R,B,\lambda)$, such that
    \begin{equation}
        \sup_{s\in\cS}{|V(s)-V^\dagger(s)|\leq \epsilon}.
    \end{equation}
    Let $R_0=V_{\text{max}}\sqrt{Nd/\lambda}, B_0=2\beta$, it is easy to show that at each iteration, $\widehat{V}^{u}\in \cV(R_0,B_0,\lambda)$. From the definition of $\BB$, we have 
    \begin{equation}
        |\BB\widehat{V}-\BB V^\dagger| = \gamma\left|\int{ (\widehat{V}(s')-V^\dagger(s'))\innerprod{\phi(s,a),\mu(s')}\ud s'} \right| \leq \gamma \epsilon.
    \end{equation}
    Then we have 
    \begin{equation}
        |(r+\gamma V -\BB V)- (r+\gamma V^\dagger -\BB V^\dagger)| \leq 2\gamma \epsilon.
    \end{equation}
    Let $\epsilon_\tau^\dagger = r(s_\tau,a_\tau)+\gamma V^\dagger(s_{\tau+1})-\BB V^\dagger(s,a)$, we have 
    \begin{align*}
        \text{(iii)}^2 = \norm{\sum_{\tau=1}^N\phi_\tau\epsilon_\tau}^2_{\Lambda^{-1}} &\leq 2 \norm{\sum_{\tau=1}^N\phi_\tau\epsilon^\dagger_\tau}^2_{\Lambda^{-1}} +2 \norm{\sum_{\tau=1}^N\phi_\tau(\epsilon^\dagger_\tau-\epsilon_\tau)}^2_{\Lambda^{-1}} \\
        & \leq 2 \norm{\sum_{\tau=1}^N\phi_\tau\epsilon^\dagger_\tau}^2_{\Lambda^{-1}} + 8\gamma^2 \epsilon^2 \sum_{\tau=1}^{N}|\phi_\tau \Lambda^{-1} \phi_\tau|\\
        & \leq 2 \norm{\sum_{\tau=1}^N\phi_\tau\epsilon^\dagger_\tau}^2_{\Lambda^{-1}} + 8\gamma^2 \epsilon^2 N^2 / \lambda
    \end{align*}

    It remains to bound $\norm{\sum_{\tau=1}^N\phi_\tau\epsilon^\dagger_\tau}^2_{\Lambda^{-1}}$. From the assumption for data collection process, it is easy to show that $\EE_\cD{[\epsilon_\tau \given \cF_{\tau-1}]}=0$, where $F_{\tau-1} = \sigma(\{(s_i,a_i)_{i=1}^{\tau}\cup (r_i,s_{i+1})_{i=1}^{\tau} \})$ is the $\sigma$-algebra generated by the variables from the first $\tau$ step. Moreover, since $\epsilon_\tau \leq 2V_{\text{max}}$, we have $\epsilon_\tau$ are $2V_{\text{max}}$-sub-Gaussian conditioning on $F_{\tau-1}$. Then we invoke Lemma \ref{lem:concen_self_normalized} with $M_0=\lambda \cdot I$ and $M_k  = \lambda \cdot  I + \sum_{\tau =1}^k \phi(x_\tau,a_\tau)\ \phi(x_\tau,a_\tau)^\top$. For the fixed function $V\colon \cS\to [0,V_{\text{max}}]$, we have
\begin{equation}
    \PP_{\cD}  \bigg( \Big\|   \sum_{\tau=1}^{N} \phi(x_\tau,a_\tau) \cdot \epsilon_\tau(V)  \Big\|_{\Lambda^{-1}}^2  >   8 V^2_{\text{max}} \cdot  \log \Big(  \frac{\det(\Lambda)^{1/2}}{\delta  \cdot \det( \lambda \cdot I) ^{1/2} }  \Big)   \bigg ) \leq   \delta
\end{equation}

  for all $\delta \in (0,1)$. Note that  $\|\phi(s,a)\|\leq 1$ for all $(s,a )\in \cS\times \cA$ by Definition \ref{assump:linear_mdp}.
  We have 
  \begin{equation*}
  \|\Lambda\|_{\oper}  =   \Big\|\lambda \cdot  I + \sum_{\tau=1}^N \phi(s_\tau,a_\tau)\phi(s_\tau,a_\tau)^\top \Big\| _{\oper} \leq  \lambda  +  \sum_{\tau = 1} ^N  \| \phi(s_\tau,a_\tau)\phi(s_\tau,a_\tau)^\top  \|_{\oper} \leq \lambda  + N,
  \end{equation*}
  where $\|\cdot\|_{\oper}$ denotes the matrix operator norm.
  Hence, it holds that
  $\det(\Lambda)\leq (\lambda+N)^d$ and $\det(\lambda \cdot I) = \lambda ^d $, which implies
\begin{align*}
    & \PP_{\cD}  \bigg( \Big\|   \sum_{\tau=1}^{N} \phi(s_\tau,a_\tau) \cdot \epsilon_\tau(V)  \Big\|_{\Lambda_{-1}}^2  >  4V^2_{\text{max}}\cdot \bigl (  2  \cdot \log(1/ \delta ) + d\cdot \log(1+N/\lambda)\big) 
\biggr)  \notag \\
&\qquad \leq \PP_{\cD}  \bigg( \Big\|   \sum_{\tau=1}^{N} \phi(s_\tau,a_\tau) \cdot \epsilon_\tau(V)  \Big\|_{\Lambda_{-1}}^2  >   8 V^2_{\text{max}} \cdot  \log \Big(  \frac{\det(\Lambda)^{1/2}}{\delta  \cdot \det( \lambda \cdot I) ^{1/2} }  \Big)   \bigg ) \leq   \delta. \notag
\end{align*} 

Therefore, we conclude the proof of Lemma \ref{lemma:xi_quantifier}. 

Applying Lemma~\ref{lemma:xi_quantifier} and the union bound, we have 
\begin{equation}
    \PP_{\cD} \bigg( \sup_{V \in \cN (\varepsilon)}     \Big\| \sum_{\tau=1}^{N} \phi(x^\tau,a^\tau) \cdot \epsilon_\tau(V)  \Big\|_{\Lambda^{-1}}^2 >  4V^2_{\text{max}}  \cdot \bigl (  2 \cdot \log(1/ \delta ) + d \cdot \log(1+N/\lambda)\big)   \biggr )  \leq   \delta \cdot | \cN(\varepsilon ) | .
\end{equation}

Recall that 
\begin{equation}
    \hat V \in \cV (R_0, B_0, \lambda),\qquad \text{where}~~  R_0 = V_{\text{max}}\sqrt{ Nd/\lambda},~  B_0 = 2\beta,~ \lambda = 1 ,~ \beta = c \cdot d V_{\text{max}} \sqrt{\zeta}.
\end{equation}

Here $c>0$ is an absolute constant, $\xi\in (0,1)$ is the confidence parameter, and $\zeta = \log (2 d V_\text{max} / \xi) $ is specified in Algorithm \ref{alg:1}. Applying Lemma \ref{lem:covering_num} with $\varepsilon = d V_{\text{max}} / N$,
we have 
\begin{align}\label{eq:apply_cov_num}
\log  | \cN(\varepsilon) | &  \leq  d \cdot \log  ( 1 + 4 d^{-1/2}N^{3/2}  ) + d^2 \cdot \log  ( 1 + 32 c^2\cdot d^{1/2}N^2\zeta  )\notag \\
&  \leq  d \cdot \log  ( 1 + 4 d^{1/2}N^2  ) + d^2 \cdot \log  ( 1 + 32 c^2 \cdot d^{1/2}N^2\zeta  ).
\end{align}
By setting $\delta=\xi/| \cN(\varepsilon ) |$, we have that with probability at least $1-\xi$,
\begin{align}
    \label{eq:bound_term3_8}
    & \Big\|  \sum_{\tau=1}^{N} \phi(s_\tau,a_\tau) \cdot \epsilon_\tau(\hat{V}) \Big\|_{\Lambda^{-1}} ^2\notag
\\&\qquad \leq 8 V_{\text{max}}^2 \cdot  
\bigl ( 2 \cdot \log (V_{\text{max}}/ \xi  ) + 4d^2 \cdot \log  ( 64c^2\cdot d^{1/2 }  N^2  \zeta    ) + d\cdot \log(1+N) + 4 d^2  \bigr ) \notag
\\&\qquad \leq 8V_{\text{max}}^2 d^2 \zeta (4+\log{(64c^2)}).
\end{align}
Here the last inequality follows from simple algebraic inequalities. We set $c\geq 1$ to be sufficiently large, which ensures that $36+8\cdot \log(64c^2)\leq c^2/4$ on the right-hand side of Equation~\eqref{eq:bound_term3_8}. By Equations~\eqref{eq:define_term3} and~\eqref{eq:bound_term3_8}, it holds that  
\begin{equation}
    \label{eq:rrr888}
    | \text{(ii)}| \leq c/2 \cdot d V_{\text{max}} \sqrt{ \zeta } \cdot \sqrt{ \phi(x,a) ^\top  \Lambda  ^{-1} \phi(s,a) }  =  \beta /2 \cdot \sqrt{ \phi(x,a) ^\top  \Lambda  ^{-1} \phi(s,a) } 
\end{equation}
By Equations \eqref{eq:linear_uncertainty_quantifier}, \eqref{eq:term1_diff}, \eqref{eq:zzz888}, and \eqref{eq:rrr888}, for all $(s,a) \in \cS\times \cA$, it holds that 
\begin{equation}
    \bigl |  (\BB \hat V ) (x,a) - (\hat\BB \hat V ) (s,a) \bigr |  \leq (  V_{\text{max}} \sqrt{d} +  \beta /2 ) \cdot \sqrt{ \phi(s,a) ^\top \Lambda  ^{-1} \phi(s,a) }  \leq \Gamma (s,a)
\end{equation}
 with probability at least $1 - \xi$. Therefore, we conclude the proof of Lemma~\ref{lemma:xi_quantifier}.
\end{proof}

\begin{lemma}[Bounded weight of value function]
    \label{lemma:bounded_weight_value}
    Let $V_{\text{max}} = r_{\text{max}}/(1-\gamma)$. For any function $V: \cS \rightarrow [0, V_{\text{max}}]$, we have 
    \begin{align*}
        \norm{w} \leq V_{\text{max}}\sqrt{d}, \norm{\widehat{w}} \leq V_{\text{max}} \sqrt{\frac{Nd}{\lambda}}.
    \end{align*}
\end{lemma}
\begin{proof}
    since 
    \begin{align*}
        w^\top \phi(s,a) = \innerprod{M,\phi(s,a)} + \gamma \int{V(s')\psi(s')^\top M \phi(s,a)\ud s'},
    \end{align*}
    We have 
    \begin{align*}
        w &= M + \gamma \int{V(s')\psi(s')^\top M\ud s' } \\
        &= r_{\text{max}} \sqrt{d}+\gamma V_{\text{max}} \sqrt{d}\\
        &= V_{\text{max}} \sqrt{d}.
    \end{align*}
    For $\widehat{w}$, we have 
    \begin{align*}
        \norm{\widehat{w}} &= \norm{\Lambda^{-1}\sum_{\tau=1}^N{\phi_\tau (r_\tau+\gamma V(s_{\tau+1}))}} \\
        & \leq \sum_{\tau=1}^N{\norm{\Lambda^{-1}{\phi_\tau (r_\tau+\gamma V(s_{\tau+1}))}}}  \\
        & \leq V_{\text{max}}\sum_{\tau=1}^N{\norm{\Lambda^{-1}{\phi_\tau}}} \\
        & \leq V_{\text{max}}\sum_{\tau=1}^N{\sqrt{\phi_\tau^\top \Lambda^{-1/2}\Lambda^{-1}\Lambda^{-1/2}\phi_\tau}} \\
        & \leq \frac{V_{\text{max}}}{\sqrt{\lambda}}\sum_{\tau=1}^N{\sqrt{\phi_\tau^\top \Lambda^{-1}\phi_\tau}} \\
        & \leq V_{\text{max}}\sqrt{\frac{N}{\lambda}}\sqrt{\mathrm{Tr}(\Lambda^{-1}\sum_{\tau=1}^T{\phi_\tau\phi_\tau^\top })} \\
        & \leq V_{\text{max}}\sqrt{\frac{Nd}{\lambda}}.
    \end{align*}
\end{proof}

\subsection{Proof of Lemma~\ref{lemma:3}}
\label{proof_lemma_3}
\begin{proof}
    We consider the following iteration:
    \begin{align}
        \label{model_based_iteration}
        &V_{\text{min}}~~~\leftarrow \min_{s'}V(s'), \nonumber\\
        &Q(s,a)\leftarrow r(s,a) + \gamma(1-\varepsilon)\EE_{s'\sim P_0}V(s') + \gamma \varepsilon V_{\text{min}}, \nonumber\\
        &V(s)~~~\leftarrow \max_a Q(s,a).
    \end{align}
    It is easy to see that if the iteration in~\eqref{model_based_iteration} converges, it is the value function for the policies specified in Equation~\eqref{model_based_policy_opt}. In fact, the iteration above has a unique stationary solution follows from the fact that it is a $\gamma$-contraction for $V(s)$. Then it suffices to show that the solution to the value iteration with discount factor $(1-\varepsilon)\gamma$ is the same as the above stationary solution up to a constant. Let $Q(s,a)$ and $V(s)$ be  the solution to the value iteration with discount factor $(1-\varepsilon)\gamma$. Then we have 
    \begin{align}
        Q(s,a) &= r(s,a) + (1-\varepsilon)\gamma \EE_{s'}{V(s')}, \nonumber
    \end{align}
    Let $\Delta={\gamma\varepsilon\min_s[\max_aQ(s,a)]}/{(1-\gamma)}$ and $\tilde{Q}(\cdot,\cdot)=Q(\cdot,\cdot)+\Delta, \tilde{V}(\cdot) = V(\cdot)+\Delta$,
    then we have $$\min_s[\max_a\tilde{Q}(s,a)] = \frac{(1-\gamma+\gamma\varepsilon)\Delta}{\gamma\varepsilon}.$$
    This leads to
    \begin{align*}
        &\tilde{Q}(s,a) \\
        =& r(s,a) + \gamma (1-\varepsilon) \EE_{s'} V(s')+\Delta\\
        =& r(s,a) + \gamma (1-\varepsilon) \EE_{s'} {\tilde{V}(s')} + (1-\gamma+\gamma\varepsilon)\Delta\\
        =& r(s,a) + \gamma (1-\varepsilon) \EE_{s'} \tilde{V}(s') + \gamma \varepsilon \min_s[\max_a \tilde{Q}(s,a)].\\
    \end{align*}
    This means that $\tilde{Q}$ is the unique stationary solution to the iteration in~\eqref{model_based_iteration}. Then we have the value function for policies in Equation~\eqref{model_based_policy_opt} has the same value function as the policies in Equation~\eqref{small_gamma_policy_opt} up to a constant. Then we finish the proof of Lemma~\ref{lemma:3}.
\end{proof}

\subsection{Proof of Theorem~\ref{theorem:2}}
\label{proof_theorem_2}
\begin{proof}
    We first specialize the algorithm for offline value iteration with a lower discount factor and an estimatied model, as depicted in Algorithm~\ref{alg:3}.
    \begin{algorithm}[H]
        \caption{Generalized Value Iteration}\label{alg:3}
        \begin{algorithmic}[1]
            \STATE {\bf Require}: Dataset $\cD$, discount factor $\gamma$, dicount factor coefficient $\varepsilon$
            \STATE Estimated the model by MLE: $\widehat{P}=\argmax_{P\in \cM}\EE_{\cD}[\ln P(s'\mid s,a)]$. 
            \STATE Obtain the estimatied Bellman operator $\hat \BB_\gamma$ from learned model $\widehat{P}$.
            \WHILE{not converge}
            \STATE Set $\hat{Q}(\cdot,\cdot) \leftarrow  \min_{e\in \cE(\varepsilon)} \left[(\hat\BB_{(1-e)\gamma} \hat{V})(\cdot,\cdot)\right]$, where $\cE(\varepsilon) = \{e(s,a) | \EE_{\cD}{\left[e(s,a)\right]} \leq \varepsilon\}$.
            \STATE Set $\hat{\pi} (\cdot \given \cdot) \leftarrow \argmax_{\pi}\EE_{\pi}{\left[\hat{Q}(\cdot, \cdot)\right]}$.
            \STATE Set $\hat{V}(\cdot) \leftarrow \EE_{\hat{\pi}}{\left[\hat{Q}(\cdot, \cdot)\right]}$.
            \ENDWHILE
            \STATE \textbf{Return} $\hat\pi$%
        \end{algorithmic}
    \end{algorithm}
    Algorithm~\ref{alg:3} is similar to the idea of a lower discount factor $(1-\varepsilon)\gamma$, with only technical differences.
    It is easy to show that Algorithm~\ref{alg:3} with discount factor $\gamma$ and a lower yields the same policy as the following optimization problem $$\argmax_{\pi\in\Pi}\argmin_{M\in\cM_\varepsilon} V_{M,\gamma}(\pi),$$
    where $$\cM_\varepsilon=\set{M\in\cM \Biggiven \exists~\cP(\cdot|s,a), e\in \cE(\varepsilon),\cP_M(\cdot|s,a)= (1-e)\cP(\cdot|s,a;\widehat{M})+e \cP(\cdot|s,a), \forall (s,a) \in \cD}.$$
     Here $\widehat{M}$ is the model obtained from MLE estimator and $\cM$ is the set of all linear models. The proof is similar to Lemma~\ref{lemma:3} and we omit it for simplicity.

    Then we prove the theorem with the following steps.
    \paragraph{1. Bounding $\cM_\varepsilon$.}~\\
    Let $\cM_{\text{TV}}(M_0,\varepsilon)=\set{M\given \EE_{\cD}\left[\TV{(\cP(\cdot|s,a;{M_0}),\cP(\cdot|s,a;M))}^2\right]\leq \varepsilon}$. To etablish the equivalence between $\cM_\varepsilon$ and $\cM_{\text{TV}}(M_0,\varepsilon)$, we need the follow assumption.
    
    \begin{assumption}[Regularity]
        \label{assumption_regularity}
        We assume that the underlying linear MDP satisfies
        \$\tilde{p}=\min\{p_{\text{min}},1-p_{\text{max}}\}>0,\label{model_based_regularity}\$
        where $p_{\text{min}}=\inf_{\cP(s'|s,a)>0}{\cP(s'|s,a)},p_{\text{max}}=\sup_{\cP(s'|s,a)<1}{\cP(s'|s,a)}$.
    \end{assumption}
    Note that this assumption is always true for tabular MDPs. It only rules out the case when there exists a sequence $\{s'_n\}_{n=1}^{\infty}$ such that $\lim_{n\rightarrow \infty} \cP(s'_n|s,a) = 0$ or $\lim_{n\rightarrow \infty} \cP(s'_n|s,a) = 1$.
    Then we have
    \begin{equation}
        \label{eq:m_epsilon_bound}
        \cM_{\text{TV}}(\widehat{M},\tilde{p}^2 \varepsilon^2 /4 ) \subseteq \cM_\varepsilon \subseteq  \cM_{\text{TV}}(\widehat{M},\varepsilon^2),
    \end{equation}
    
    Recall that $\cM_\varepsilon=\set{M| \cP(\cdot|s,a;M)= (1-\varepsilon)\cP(\cdot|s,a;{\widehat{M}})+\varepsilon \cP(\cdot) }$. On the one hand, it is easy to see that the largest deviation in $\cM_\varepsilon$ from $\widehat{M}$ happens when $\cP(\cdot|s,a;{\widehat{M}}))$ is close to 0 or close to 1. On the other hand, we have $\TV{((1-\varepsilon)\cP(\cdot|s,a;M)+\varepsilon \cP(\cdot),\cP(\cdot|s,a;{M}))}\leq \varepsilon$. Then we have the result in Equation~\eqref{eq:m_epsilon_bound}. The following steps largely follows~\citep{uehara2021pessimistic}.

    \paragraph{2. Upper-bounding $   \EE_{(s,a)\sim\rho}[\TV(\cP(\cdot \mid s,a;M^{\star}),\cP(\cdot \mid s,a;\widehat{M}))^2 ]$. }~ \\ 
    Let  
    \begin{align*}
       \cM =\braces{\cP(\cdot\given s,a;M) \mid M \in  \RR^{d\times r}, \norm{M}_2\leq \sqrt{d},\int \phi(s,a)^\top M\psi(s,a)\ud(s')=1,~ \forall(s,a)},
    \end{align*}
    and $\cH = \braces{ \sqrt{\frac{\cP+\cP^{\star}}{2}} \mid \cP\in  \cM } . $

    By invoking Theorem~\ref{thm:mle}, we first show
    \begin{align*}
        \EE_{(s,a)\sim\rho}[\TV(\cP(\cdot \mid s,a;M^{\star}),\cP(\cdot \mid s,a;\widehat{M}))^2 ]\leq c\{(d/N)\ln^2(Nd)+\ln(c/\delta)/N \}. 
    \end{align*}
     
    To do that, we calculate the entropy integral with bracketing.
    First, we have 
    \begin{align}\label{eq:bracketing}
           \cN_{[]}(\epsilon,\cH,d)    \leq  \cN_{[]}(\epsilon,\cM,d'). 
    \end{align}
    where 
    \begin{align}
        d'(a,b) &=\EE_{(s,a)\sim\rho}\bracks{\int (a(s,a,s')-b(s,a,s'))^2 \ud(s')}^{1/2},\\
        d(a,b) &=\EE_{(s,a)\sim\rho}\bracks{\int (\sqrt{a(s,a,s')}-\sqrt{b(s,a,s')})^2 \ud(s')}^{1/2}. 
    \end{align}
    Here, we use two observations. The first observation is 
    \begin{align*}
       d^2\prns{\sqrt{\frac{\cP(M')+\cP^{\star}}{2}}, \sqrt{\frac{\cP(M'')+\cP^{\star}}{2}}}
       \leq   c_1d'^2( \cP(M'),\cP(M'')) 
    \end{align*}
    due to the mean-value theorem
    \begin{align*}
        \sqrt{a}-\sqrt{b}\leq \max(1/\sqrt{a},1/\sqrt{b})(a-b) 
    \end{align*}
    and Assumption~\ref{assumption_regularity} that $\cP^{\star}(s'\mid s,a)\geq c_0>0$. The second observation is when we have $P'<g<P''$, we also have $\sqrt{(\cP'+\cP^{\star})/2}<\sqrt{(g+P^{\star})/2}<\sqrt{(\cP''+\cP^{\star})/2}$. Then, \ref{eq:bracketing} is concluded. 
    
    Next, by letting $M^{(1)},\cdots,M^{(K)}$ be an $\epsilon$-cover of the $d$-dimensional ball with a radius $\sqrt{d}$, i.e, $B_d(\sqrt{d})$, we have the brackets $\{[\cP(M^{(i)})-\epsilon,\cP(M^{(i)})+\epsilon]\}_{i=1}^{K}$ which cover $\cM$. This is because for any $\cP(M)\in \cM$, we can take $M^{(i)}$ s.t. $\|M-M^{(i)}\|_2\leq \epsilon/\sqrt{d} $, then, 
    \begin{align*}
        \cP(M^{(i)})-\epsilon<\cP(M)<     \cP(M^{(i)})+\epsilon,\quad \forall(s,a,s')
    \end{align*}
    noting 
    \begin{align}\label{eq:braket}
    |\cP(M)(s,a,s') -\cP(M^{(i)})(s,a,s')|\leq \sqrt{d}\|M-M^{(i)}\|_2\leq \epsilon, \quad \forall(s,a,s')
    \end{align}
    The last equality follows from the fact that $\norm{M}_2 \leq \sqrt{d}$ and $\norm{\psi}_2 \leq 1$.Therefore, we have
    \begin{align*}
        \cN_{[]}(\epsilon,\cM,\|\cdot\|_2)\leq \cN(\epsilon/\sqrt{d}, B_d(c\sqrt{d}),\|\cdot\|_2), 
    \end{align*}
    where $\cN(\epsilon/\sqrt{d}, B_d(c\sqrt{d}),\|\cdot\|_2)$ is a covering number of $ B_d(c\sqrt{d})$ w.r.t $\|\cdot\|_2$. This is upper-bounded by $(c\sqrt{d}/\epsilon)^d$ \citep[Lemma 5.7]{wainwright2019high}. Thus, we can calculate the upper bound of the entropy integral $J_{B}(\delta,\cM,\|\cdot\|_2)$:
    \begin{align*}
        \int^{\delta}_0 d^{1/2}\ln^{1/2}(cd/u)\mathrm{d}u & \leq    \int^{\delta}_0 d^{1/2}\ln(1/u)\mathrm{d}u + \delta d^{1/2}\ln(c_1\sqrt{d}) \\
         &=c d^{1/2}(\delta+\delta\ln(1/\delta))+  \delta d^{1/2}\ln(c d)\\
         &\leq   c d^{1/2}\delta\ln(cd/\delta). 
    \end{align*}
    By taking $G(x)=d^{1/2}x\ln(c \sqrt{d}/x) $ in \ref{thm:mle}, $\delta_n=(d/n)^{1/2}\ln(nd)$ satisfies the critical inequality $$\sqrt{n}\delta^2_n\geq  d^{1/2}\delta_n \ln(c d/\delta_n).$$ Finally, with probability $1-\delta$
    \begin{align}\label{eq:final}
        \EE_{(s,a)\sim\rho}[\TV(P(\cdot \mid s,a;M^{\star}),P(\cdot \mid s,a;\widehat{M}))^2 ]\leq \xi',\quad \xi'\coloneqq \{(d/N)\ln^2(Nd)+\ln(c/\delta)/N \}. 
    \end{align}
    Hereafter, we condition on this event. 

    \paragraph{3. Upper bounding $\EE_{\cD}[\|P(\cdot\mid s,a;M^\star)-P(\cdot\mid s,a;\hat M) \|^2_1 ]$. }~\\
    We take an $\epsilon$-cover of the ball $B_d(R)$ in terms of $\|\cdot\|_2$, i.e., $\bar M=\{M^{(1)},\cdots,M^{(K)}\}$, where $K=(cR/\epsilon)^d$.  Then, for any $M \in B_d(R)$, there exists $M^{(i)}$ s.t.  $\forall (s,a)\in \cS\times \cA$, 
    \begin{align}
       & \left|\|{P}(\cdot\mid s,a;M)-{ P(\cdot\mid s,a;M^\star)}\|^2_1- \|{ P(\cdot\mid s,a;M^{(i)})}-{P(\theta^{\star})(\cdot\mid s,a)}\|^2_1\right|  \nonumber  \\
       &\leq 4|\|{P(\cdot\mid s,a;M)}-{ P(\cdot\mid s,a;M^\star)}\|_1- \|{ P}(\cdot\mid s,a;M^{(i)})-{P}(\cdot\mid s,a;M^{\star})\|_1|    \tag{$a^2-b^2=(a-b)(a+b)$} \\
       &\leq 4\|P(\cdot\mid s,a;M)-P(\cdot\mid s,a;M^{(i)})\|_1  \tag{$|\|a\|-\|b\||\leq \|a-b\|$}\\
       &\leq 4\|M-M^{(i)} \|_2   \tag{From \eqref{eq:braket} } \\
       &\leq 4 \epsilon.  \label{eq:mle_general}
    \end{align}
    
    Here, note for $\hat M$, we have $M^{(i)}$ s.t. $\|\hat M - M^{(i)} \|_2\leq \epsilon$. Then, we have
    \begin{align*}
        &\mathbb{E}_{\cD}\|P(\cdot \mid s,a; M^\star)-P(\cdot \mid s,a;\hat M)   \|^2_1\\
       &\lesssim \mathbb{E}_{\cD}\|P(\cdot \mid s,a; M^\star)-P(\cdot \mid s,a;M^{(i)})   \|^2_1+\epsilon \tag{From \ref{eq:mle_general}}\\ 
         &\lesssim (\mathbb{E}_{\cD}- \mathbb{E}_{(s,a)\sim \rho})\|P(\cdot \mid s,a; M^\star)-P(\cdot \mid s,a;M^{(i)})   \|^2_1+\epsilon+ \mathbb{E}_{(s,a)\sim \rho}\|P(\cdot \mid s,a; M^\star)-P(\cdot \mid s,a;M^{(i)})   \|^2_1\\
        &\lesssim \sqrt{ \frac{ \mathrm{var}_{(s,a)\sim \rho}[\|P(\cdot \mid s,a; M^\star)-P(\cdot \mid s,a;M^{(i)})   \|^2_1    ]\ln(K/\delta) }{n}  }+\frac{\ln(K/\delta)}{n}\\
        &+\epsilon+\mathbb{E}_{(s,a)\sim \rho}\|P(\cdot \mid s,a; M^\star)-P(\cdot \mid s,a;M^{(i)})   \|^2_1  \tag{Berstein inequality}\\
      &\lesssim \sqrt{ \frac{ \EE_{(s,a)\sim \rho}[\|P(\cdot \mid s,a; M^\star)-P(\cdot \mid s,a;M^{(i)})   \|^2_1    ]\ln(K/\delta) }{n}  }+\frac{\ln(K/\delta)}{n}\\
      &+\epsilon+\mathbb{E}_{(s,a)\sim \rho}\|P(\cdot \mid s,a; M^\star)-P(\cdot \mid s,a;M^{(i)})   \|^2_1 \tag{$\|P(\cdot \mid s,a; M^\star)-P(\cdot \mid s,a;M^{(i)})   \|^2_{1}\leq 4$}.
    \end{align*}
    Then, 
    \begin{align*}
      & \mathbb{E}_{\cD}\|P(\cdot \mid s,a; M^\star)-P(\cdot \mid s,a;\hat M)   \|^2_1\\ 
      &\lesssim \sqrt{ \frac{ \{\EE_{(s,a)\sim \rho}[\|P(\cdot \mid s,a; M^\star)-P(\cdot \mid s,a;\hat M)   \|^2_1    ]+\epsilon\}\ln(K/\delta) }{n}  }+\frac{\ln(K/\delta)}{n}\\
      & +\epsilon+\mathbb{E}_{(s,a)\sim \rho}\|P(\cdot \mid s,a; M^\star)-P(\cdot \mid s,a;\hat M)   \|^2_1  \\
       &\lesssim \sqrt{ \frac{ \{\xi'+\epsilon\}\ln(K/\delta) }{n}  }+\frac{\ln(K/\delta)}{n}+\epsilon+\xi'  \tag{From  \eqref{eq:final}}. 
    \end{align*}
    In the end, by taking $\epsilon=1/n$, we have with probability $1-\delta$, 
    \begin{align*}
        \mathbb{E}_{\cD}\|P(\cdot \mid s,a; M^\star)-P(\cdot \mid s,a;\hat M)   \|^2_1 \leq \xi,\quad \xi=  c\{(d/n)\ln^2(nR)+\ln(c/\delta)/n \}. 
    \end{align*}
    This implies with probability $1-\delta$,  $P^{\star}\in \cM_{\text{TV}}{(\hat M,\xi)}$. Note that $\cM_{\text{TV}}(\widehat{M},\tilde{p}^2\varepsilon^2/4) \subseteq \cM_\varepsilon   $, this also implies that with probability $1-\delta$, 
    $P(\cdot \mid s,a;M^{\star})\in M_{\varepsilon}$, where $\varepsilon=2\sqrt{\xi}/\tilde{p}$.
    
    \paragraph{4. Show  $  \EE_{(s,a)\sim \rho}\left[\TV(\cP(\cdot \mid s,a;M^{\star}),\cP(\cdot \mid s,a;M))^2\right]\lesssim \xi,~\forall~\cP(M)\in \cM_\varepsilon $. }
    ~ \\ 
    We show for any $\cP\in \cM_\varepsilon$, the distance between $\cP^{\star}$ is controlled in terms of TV distance. Formally, we need
    \begin{align*}
        \EE_{(s,a)\sim \rho}\left[\TV(\cP(\cdot \mid s,a;M^{\star}),\cP(\cdot \mid s,a;M))^2\right]\lesssim \xi,\quad  \forall P(M)\in \cM_\varepsilon. 
    \end{align*}
    Since $\cM_\varepsilon \subseteq  \cM_{\text{TV}}(\widehat{M},\varepsilon^2)$, it suffices to show that 
    \begin{align*}
        \EE_{(s,a)\sim \rho}\left[\TV(\cP(\cdot \mid s,a;M^{\star}),\cP(\cdot \mid s,a;M))^2\right]\lesssim \xi,\quad  \forall P(M)\in \cM_{\text{TV}}(\widehat{M},2\varepsilon). 
    \end{align*}
    For any $P \in \cM_{\text{TV}}(\widehat{M},\varepsilon^2)$, we have 
    \begin{align*}  
    & \EE_{\cD} [\TV(P(\cdot \mid s,a),P^\star(\cdot \mid s,a))^2]  \nonumber \\ 
    &\leq 2\EE_{\cD} [\TV(\widehat{P}(\cdot \mid s,a),P(\cdot \mid s,a))^2] + 2\EE_{\cD} [\TV(\widehat{P}(\cdot \mid s,a),P^\star(\cdot \mid s,a))^2] \leq 16 \xi / \tilde{p}^2.
    \end{align*} 
    Thus, we have
    \begin{align}
    &\EE_{s,a\sim \rho} [\TV(P(\cdot \mid s,a),P^\star(\cdot \mid s,a))^2] \nonumber \\
    &= \EE_{s,a\sim \rho} [\TV(P(\cdot \mid s,a),P^\star(\cdot \mid s,a))^2] -\EE_{\cD} [\TV(P(\cdot \mid s,a),P^\star(\cdot \mid s,a))^2] +\EE_{\cD}[\TV(P(\cdot \mid s,a),P^\star(\cdot \mid s,a))^2]  \nonumber  \\ 
    &\leq  A(M)+ c\xi,   \label{eq:key_version}
\end{align}
    where $A(M)\coloneqq  |(\EE_{\cD}-\EE_{(s,a)\sim \rho})\left[\TV(\cP(\cdot \mid s,a;M^{\star}),\cP(\cdot \mid s,a;M))^2\right]|.$

    We again consider an $\epsilon/\sqrt{d}$-cover of the ball $B_d(\sqrt{d})$ in terms of $\|\cdot\|_2$, i.e., $M'=\{M^{(1)},\cdots,M^{(K)}\}$, where $K=(c_1d/\epsilon)^d$ ($\epsilon=1/N$). Then $M'$ is also an $\epsilon/\sqrt{d}$-cover for $\cM_{\text{TV}}(\widehat{M},\varepsilon^2)$. That is for any $M$ s.t.  $\forall {\cP(M)}\in  \cM_{\text{TV}}(\widehat{M},\varepsilon^2)$, we can take $M'\in \cM'$ s.t. $\|M-M'\|_2\leq \epsilon/\sqrt{d}$.

    Then, we have
    \begin{align}\label{eq:m'}
            \EE_{(s,a)\sim \rho}\left[\TV(\cP(\cdot \mid s,a;M^{\star}),\cP(\cdot \mid s,a;M))^2\right]\leq A(M)+c\xi,\quad \forall M \in \cM'.   
    \end{align}
    This is because for any $M^{(i)}\in \cM'$, we can take $P(M)\in \cM_{\text{TV}}(\widehat{M},\varepsilon^2)$ such that 
    \begin{align*}
        &\EE_{(s,a)\sim \rho}\left[\TV(\cP(\cdot \mid s,a;M^{\star}),\cP(M^{(i)})(\cdot \mid s,a))^2\right]\\
        &\leq     \EE_{(s,a)\sim \rho}[\TV(\cP(\cdot \mid s,a;M^{\star}),\cP(M^{(i)})(\cdot \mid s,a))^2-\TV(\cP(\cdot \mid s,a;M^{\star}),P(\cdot \mid s,a;M))^2]   \\
        &+     \EE_{(s,a)\sim \rho}[ \TV(\cP(\cdot \mid s,a;M^{\star}),\cP(\cdot \mid s,a;M))^2  ]\\
        &\leq 4\epsilon +   \EE_{(s,a)\sim \rho}[ \TV(\cP(\cdot \mid s,a;M^{\star}),\cP(\cdot \mid s,a;M))^2  ]  \\ 
        &\lesssim A(M)+\xi.
    \end{align*}
    
    From Bernstein's inequality, we have that with probability $1-\delta$,
    \begin{align}\label{eq:bernstein}
        A(M)&=| (\mathbb{E}_{\cD}- \mathbb{E}_{(s,a)\sim \rho})[\TV(\cP(\cdot \mid s,a;M),\cP(\cdot \mid s,a;M^{\star}))^2 ] |\nonumber \\
        & \lesssim \sqrt{ \frac{ \mathrm{var}_{(s,a)\sim \rho}[\TV(\cP(\cdot \mid s,a;M),\cP(\cdot \mid s,a;M^{\star}))^2  ]\ln(K/\delta) }{N}  }+\frac{\ln(K/\delta)}{N},\quad  \forall M \in  \cM'.  
    \end{align} 
    
    Based on the construction of $\cM'$ and Equation~\eqref{eq:m'}, we have
    \begin{align}
        \label{eq:bernstein2}
        \mathrm{var}_{(s,a)\sim \rho}[\TV(\cP(\cdot \mid s,a;M^{\star}),\cP(\cdot \mid s,a;M))^2]\lesssim A(M)+\xi,\quad \forall M \in \cM'.  
    \end{align}
    Taking Equation~\eqref{eq:bernstein2} into Equation~\eqref{eq:bernstein}, we have $A(M)$ satisfing
    \begin{align*}
      A^2(M)-A(M)B_1-B_2\leq 0,\quad B_1=\frac{\ln(K/\delta)}{N}, B_2=\xi\frac{\ln(K/\delta)}{N}+\prns{\frac{\ln(K/\delta)}{N}}^2. 
    \end{align*}
    Then,  we have 
    \begin{align}\label{eq:bound_a}
        A(M) \leq \frac{\ln(K/\delta)}{N}+\xi^{1/2}\sqrt{\frac{\ln(K/\delta)}{N}}\lesssim \xi,\quad \forall M \in \cM'.  
    \end{align}
    
    We combine all steps. Recall for any  $\forall {P(M)}\in  \cM_\varepsilon$, we can take $M'\in \cM'$ s.t. $\|M-M'\|_2\leq 1/n$. Then, for any $P(M)\in \cM_\varepsilon$, we have 
    \begin{align*}
          A(M) &=|(\EE_{\cD}-\EE_{(s,a)\sim \rho})\left[\TV(\cP(\cdot \mid s,a;M),\cP(\cdot \mid s,a;M^{\star}))^2\right]\\ 
           &\leq |(\EE_{\cD}-\EE_{(s,a)\sim \rho})[\TV(\cP(\cdot \mid s,a;M),\cP(\cdot \mid s,a;M^{\star}))^2-\TV(P(\cdot \mid s,a;M'),\cP(\cdot \mid s,a;M^{\star}))^2]\\
           &+(\EE_{\cD}-\EE_{(s,a)\sim \rho})[\TV(\cP(\cdot \mid s,a;M'),\cP(\cdot \mid s,a;M^{\star}))^2]\\
             &\lesssim 8\varepsilon+|(\EE_{\cD}-\EE_{(s,a)\sim \rho})[\TV(\cP(\cdot \mid s,a;M'),\cP(\cdot \mid s,a;M^{\star}))^2]  \\ 
           &\lesssim \xi \tag{From \eqref{eq:bound_a} and $M'\in \cM'$}.
    \end{align*}
    Then, we have with probability $1-\delta$,
    \begin{align}\label{eq:concent_a}
              A(M)\lesssim \xi,\quad \forall P(M) \in \cM_\varepsilon. 
    \end{align}
    Finally, for any $P(M)\in \cM_\varepsilon$, with probability $1-\delta$, we have 
    \begin{align*}
       \EE_{(s,a)\sim \rho}[\TV(P(\cdot \mid s,a;M^{\star}),P(\cdot \mid s,a;M))^2] &\leq A(M)+c\xi \tag{From \eqref{eq:m'}} \\
     &  \lesssim  \xi.   \tag{From \eqref{eq:concent_a}}
    \end{align*}

    \paragraph{5.Bounding the performance of $\pi^{*}$ .}~\\
    We first prove 
    \begin{align}
        \label{eq:inter_goal}
          V^{\pi^{*}}_{P^*}- V^{\pi^{*}}_{P} &\lesssim (1-\gamma)^{-2}\sqrt{c^\ddagger d \xi} \cdot r_{\text{max}},
    \end{align}
    for all $P\in \cM_\varepsilon$. Recall from the third step, for $P(M)\in \cM_\varepsilon$,  we have 
    \begin{align*}
       \EE_{(s,a)\sim \rho}\left[\TV(P(\cdot \mid s,a;M^{\star}),P(\cdot \mid s,a;M))^2\right]\lesssim  \xi. 
    \end{align*}
    From the second statement of Lemma~\ref{lem:mixture}, 
    \begin{align*}
     \forall V:\cS \to [0,V_{\text{max}}],\quad (M-M^{*})^{\top}  \Sigma_{\rho,V }(M-M^{*})\lesssim V_{\text{max}}^2\xi,\quad \Sigma_{\rho,V }=\E_{(s,a)\sim \rho}[\psi_{V}(s,a)\psi^{\top}_{V}(s,a)]. 
    \end{align*}
    Here, we have 
    \begin{align*}
      V^{\pi^{*}}_{P^*}-       V^{\pi^{*}}_{P} & \leq (1-\gamma)^{-1}\left|\EE_{(s,a)\sim d^{\pi^{*}}}\bracks{\int \{P(s' \mid s,a) - P^\star(s' \mid s,a)\}V^{\pi^{*}}_{P}(s')\ud(s')  }\right| \tag{Simulation lemma}\\
        &\leq  (1-\gamma)^{-1}\left|\EE_{(s,a)\sim d^{\pi^{*}}}\bracks{(M-M^{*})\psi_{V^{\pi^{*}}_{P}}(s,a) }\right| \\ 
        &\leq  (1-\gamma)^{-1}\underbrace{\|M-M^{*}\|_{\lambda I+\Sigma_{\rho,V^{\pi^{*}}_{P} }}}_{(a)}\underbrace{\EE_{(s,a)\sim d^{\pi^{*}}}  \bracks{\|\psi_{V^{\pi^{*}}_{P}}(s,a)\|_{(\Sigma_{\rho,V^{\pi^{*}}_{P} }+\lambda I)^{-1}} }}_{(b)}.  \tag{C-S inequality}
    \end{align*}
    The first term (a) is upper-bounded by $\sqrt{V_{\text{max}}^2 \xi+\lambda d }$ noting $0\leq V^{\pi^*}_P\leq V_{\text{max}}$. The term (b) is  upper-bounded by 
    \begin{align*}
        \EE_{(s,a)\sim d^{\pi^{*}}}  \bracks{\|\psi_{V^{\pi^{*}}_{P}}(s,a)\|_{{(\Sigma_{\rho,V^{\pi^{*}}_{P} }+\lambda I)^{-1}} } } &\leq     \EE_{(s,a)\sim d^{\pi^{*}}}  \tag{Jensen's inequality} \bracks{\|\psi_{V^{\pi^{*}}_{P}}(s,a)\|^2_{{(\Sigma_{\rho,V^{\pi^{*}}_{P} }+\lambda I)^{-1}} } }^{1/2} \\
        &=  \sqrt{\Tr( \Sigma_{d^{\pi^{*}},V^{\pi^{*}}_{P} }(\lambda I+\Sigma_{\rho,V^{\pi^{*}}_{P} })^{-1} )  }\\
        &\leq \sqrt{c_V\Tr( \Sigma_{\rho,V^{\pi^{*}}_{P} }(\lambda I+\Sigma_{\rho,V^{\pi^{*}}_{P} })^{-1} )  }\\
        &\leq  \sqrt{c_V \rank(\Sigma_{\rho,V^{\pi^{*}}_{P} })}\\
        &\leq \sqrt{c_Vd } = \sqrt{ c^\ddagger d},
    \end{align*}
    where $c_V=\sup_{V\in \{\cS \to [0,V_{\text{max}}]\}}\sup_{x}\frac{ x^{\top}\Sigma_V(d^{\pi^*}) x}{x^{\top}\Sigma_V(\rho) x}$, $c^\ddagger=\sup_{x}\frac{x^{\top } \Sigma(d^{\pi^{*}}) x}{x^{\top} \Sigma(\rho) x }$, and $$\Sigma_V(\mu)=\E_{(s,a)\sim \mu}[\psi_{V}(s,a)\psi^{\top}_{V}(s,a)],~\Sigma(\mu)=\E_{(s,a)\sim \mu}[\phi(s,a)\phi^{\top}(s,a)].$$ The second inequality follows from the definition of $c_V$, and the last inequality follows from the third statement in Lemma~\ref{lem:mixture}. By taking $\lambda$ s.t. $\lambda d \lesssim V_{\text{max}} \xi$, We have the desired Equation~\eqref{eq:inter_goal}. 
    
    Finally, combining all things together, with probability $1-2\delta$, for any $\pi^{*}\in \Pi$, we have 
    \begin{align*}
        V^{\pi^{*}}_{P^{\star}}-V^{\hat \pi}_{P^{\star}}&\leq    V^{\pi^{*}}_{P^{\star}}-\min_{P\in \cM_\varepsilon}V^{\pi^{*}}_{P}+ \min_{P\in \cM_\varepsilon}V^{\pi^{*}}_{P}- V^{\hat \pi}_{P^{\star}}\\ 
        &\leq    V^{\pi^{*}}_{P^{\star}}-\min_{P\in \cM_\varepsilon}V^{\pi^{*}}_{P}+ \min_{P\in \cM_\varepsilon}V^{\hat \pi}_{P}- V^{\hat \pi}_{P^{\star}} \tag{definition of $\hat \pi$}\\ 
          &\leq  V^{\pi^{*}}_{P^{\star}}-\min_{P\in \cM_\varepsilon}V^{\pi^{*}}_{P}  \tag{Second step, $P^{\star}\in \cM_\varepsilon$}\\
          &\lesssim  (1-\gamma)^{-2}\sqrt{ c^\ddagger d \xi} \cdot r_{\text{max}}.   \tag{From \eqref{eq:inter_goal}}
    \end{align*}

    Finally, recall from the relationship in the second and third step, we have $\varepsilon=\sqrt{\xi}/2\tilde{p},~\xi=c\{(d/N)\ln^2(Nd)+\ln(c/\delta)/N \}$, which leads to

    \begin{equation*}
        V^{\pi^{*}}_{P^{\star}}-V^{\hat \pi}_{P^{\star}}\lesssim \frac{c_3}{(1-\gamma)^{-2}}\sqrt{ c^\ddagger d^2 \zeta/N} \cdot r_{\text{max}},~\zeta=\log^2{(c_2Nd/\delta)}.
    \end{equation*}
\end{proof}

\subsection{Technical Lemmas}
\begin{lemma}[$\varepsilon$-Covering Number \citep{jin2020provably}] 
	For all $h\in [H]$ and all $\varepsilon > 0$, 
	we have 	\label{lem:covering_num}
	$$
	\log | \cN (  \varepsilon; R, B, \lambda)  | \leq d \cdot \log (1+ 4 R /  \varepsilon  ) + d^2  \cdot  \log\bigl(1+ 8 d^{1/2} B^2 / ( \varepsilon^2\lambda) \bigr). 	$$
\end{lemma}

\begin{proof}[Proof of Lemma \ref{lem:covering_num}]
	See Lemma D.6 in \cite{jin2020provably} for a detailed proof. 
\end{proof}

\begin{lemma}[Concentration of Self-Normalized Processes \citep{abbasi2011improved}]
    Let $\{\cF_t \}^\infty_{t=0}$ be a filtration and $\{\epsilon_t\}^\infty_{t=1}$ be an $\RR$-valued stochastic process such that $\epsilon_t$ is $\cF_{t} $-measurable for all $t\geq 1$.
    Moreover, suppose that conditioning on $\cF_{t-1}$, 
     $\epsilon_t $ is a  zero-mean and $\sigma$-sub-Gaussian random variable for all $t\geq 1$, that is,  
     \$
      \EE[\epsilon_t\given \cF_{t-1}]=0,\qquad \EE\bigl[ \exp(\lambda \epsilon_t) \biggiven \cF_{t-1}\bigr]\leq \exp(\lambda^2\sigma^2/2) , \qquad \forall \lambda \in \RR. 
      \$
     Meanwhile, let $\{\phi_t\}_{t=1}^\infty$ be an $\RR^d$-valued stochastic process such that  $\phi_t $  is $\cF_{t -1}$-measurable for all $ t\geq 1$. 
    Also, let  $M_0 \in \RR^{d\times d}$ be a  deterministic positive-definite matrix and 
    \$
    M_t = M_0 + \sum_{s=1}^t \phi_s\phi_s^\top
    \$ for all $t\geq 1$. For all $\delta>0$, it holds that
    \begin{equation*}
    \Big\| \sum_{s=1}^t \phi_s \epsilon_s \Big\|_{ M_t ^{-1}}^2 \leq 2\sigma^2\cdot  \log \Bigl( \frac{\det(M_t)^{1/2}\cdot \det(M_0)^{- 1/2}}{\delta} \Bigr)
    \end{equation*}
    for all $t\ge1$ with probability at least $1-\delta$.
    \label{lem:concen_self_normalized}
\end{lemma}
\begin{proof}
    See Theorem 1 of \cite{abbasi2011improved} for a detailed proof. 
\end{proof}

Given a function class $\cF$, let $\cN_{[]}(\delta,\cF,d)$ be the bracketing number of $\cF$ w.r.t the metric $d(a,b)$ given by 
\begin{align*}
    d(a,b )=\E_{(s,a)\sim \rho}\bracks{\int (a(s'\mid s,a)-b(s'\mid s,a))^2\ud(s')}^{1/2}. 
\end{align*}
Then, the entropy integral of $\cF$ is given by 
\begin{align}\label{eq:entropy}
    J_{B}(\delta,\cF,d)=\max\prns{ \int^{\delta}_{\delta^2/2} (\log\cN_{[]}(u,\cF,d))^{1/2}\mathrm{d}u,\delta}. 
\end{align}
We also define the localized class of $\cH$: 
\begin{align*}
    \cH(\delta)=\{h\in \cH :  \EE_{(s,a)\sim \rho}[h^2(P(\cdot \mid s,a)\| P^{\star}(\cdot \mid s,a) ) ]\leq \delta^2 \}, 
\end{align*}
where $h(P(\cdot \mid s,a)\| P^{\star}(\cdot \mid s,a) )$ denotes Hellinger distance defined by 
\begin{align*}
   \prns{ 0.5\int\{ \sqrt{P(s'\mid s,a)}-\sqrt{P^{\star}(s'\mid s,a)}\}^2\ud(s') }^{1/2}. 
\end{align*}

\begin{theorem}[MLE guarantee with general function approximation, \citet{uehara2021pessimistic}]\label{thm:mle}
    We take a function $G(\epsilon):[0,1]\to \RR$ s.t. $G(\epsilon)\geq J_{B}[\epsilon,\cH(\epsilon),d]$ and $G(\epsilon)/\epsilon^2$ is a non-increasing function w.r.t $\epsilon$. Then, letting $\xi_n$ be a solution to $\sqrt{n}\epsilon^2 \geq cG(\epsilon)$ w.r.t $\epsilon$. 
    With probability $1-\delta$, we have 
    \begin{align*}
        \EE_{(s,a)\sim \rho}[\|\hat P_{\MLE}(\cdot \mid s,a)-P(\cdot \mid s,a)\|^2_1]\leq  c_1\braces{\xi_n+\sqrt{\log(c_2/\delta)/n}}^2. 
    \end{align*}
\end{theorem}
\begin{proof}
    The proof follows directly by adapting to conditional distribution from Theorem 7.4 in~\citep{geer2000empirical}. Please refer to~\citep{geer2000empirical} for more details.
\end{proof}

\begin{lemma}[Property of linear MDPs]\label{lem:mixture}
    Let $P(M)= P(s'\given s,a;M)=\phi(s,a)^\top M\psi(s')$. Suppose $P(M)\in \cS\times \cA \to \Delta(\cS)$. For any function $V\in \cS \to [0,V_{\text{max}}]$, letting $\psi_{V}(s,a)=\int \text{vec}(\phi(s,a)\psi(s')^\top )V(s')\ud(s')$, we suppose $\|\phi(s,a)\|_2\leq 1$ and $\|\psi(s')\|_2\leq 1$. The following theorems hold: 
    \begin{enumerate}
        \item For any $(s,a,s')$, we have $|P(M)(s,a,s')-P(M')(s,a,s')|\leq \|M-M'\|_2$.
        \item For any $(s,a)$, we have  $\mathrm{TV}(P(M)(s,a,\cdot),P(M')(s,a,\cdot))\leq \|M-M'\|_2$. Besides, for any $V:\cS \to [0,1]$, we have 
        \begin{align*}
            |(M-M')\psi_{V}(s,a)|\leq V_{\text{max}}\mathrm{TV}(P(M)(s,a,\cdot),P(M')(s,a,\cdot)). 
        \end{align*}
        \item 
        \begin{align*}
       \sup_{V\in \{\cS \to [0,V_{\text{max}}]\}}\sup_{x}\frac{ x^{\top}\E_{(s,a)\sim d^{\pi^{*}}}[\psi_{V}(s,a)\psi^{\top}_{V}(s,a) ] x}{x^{\top}\E_{(s,a)\sim \rho}[\psi_{V}(s,a)\psi^{\top}_{V}(s,a) ] x}=    \sup_{x}\frac{x^{\top } \E_{d^{\pi^{*}}}[\phi(s,a)\phi(s,a)^{\top}] x}{x^{\top} \E_{\rho}[\phi(s,a)\phi(s,a)^{\top}]x }. 
        \end{align*}
    \end{enumerate}
    \end{lemma}
    
    \begin{proof}
    See Lemma 12 in~\citep{uehara2021pessimistic} for a detailed proof.
    
\end{proof}
    
\clearpage

\section{Complete experimental results on noised D4RL tasks}\label{complete results of TD3+BC}

\begin{figure}[h]
    \centering
    \subfigure[med(50) noise(0)]{
    \includegraphics[scale=0.23]{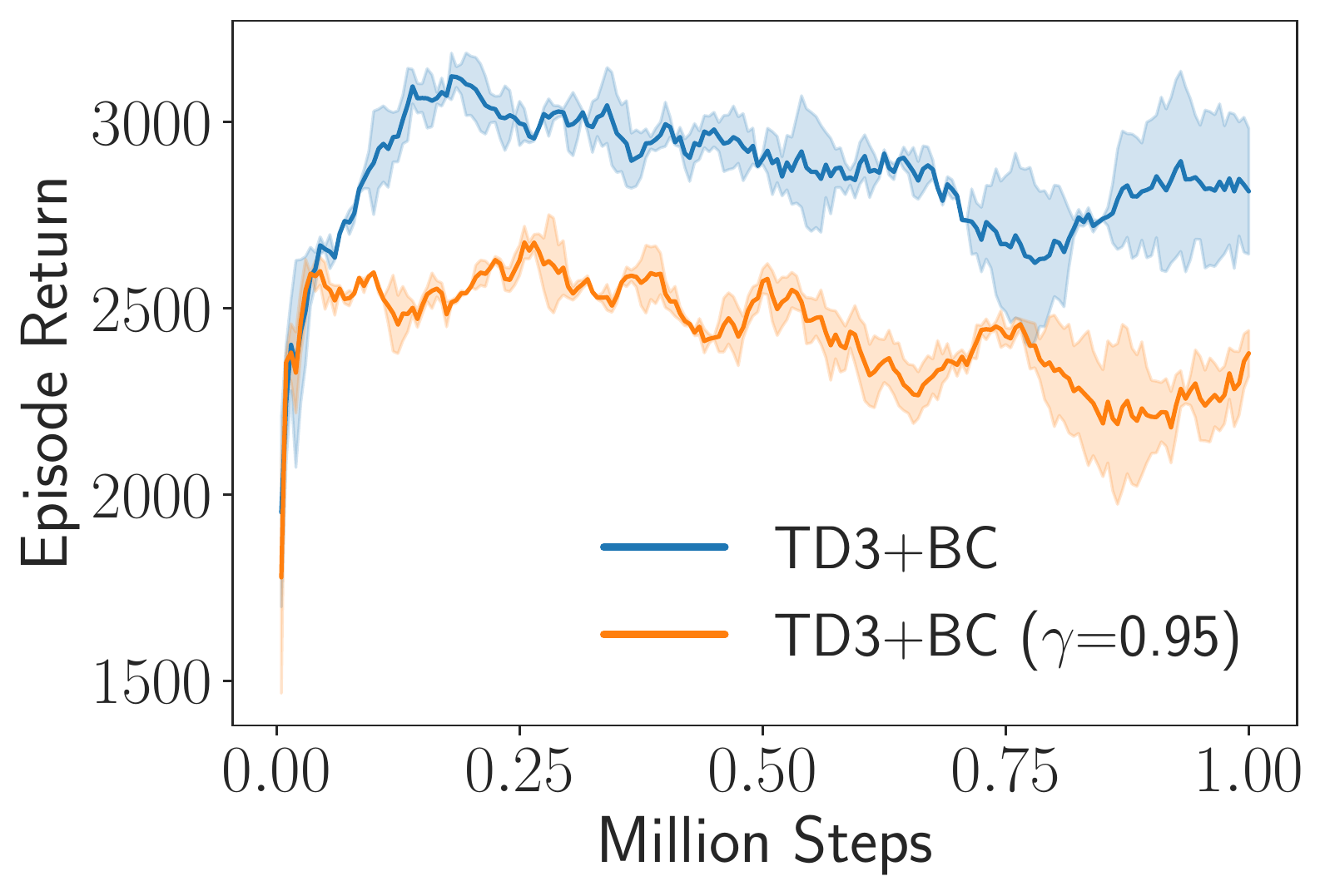}}
    \subfigure[med(50) noise(0)]{
    \includegraphics[scale=0.23]{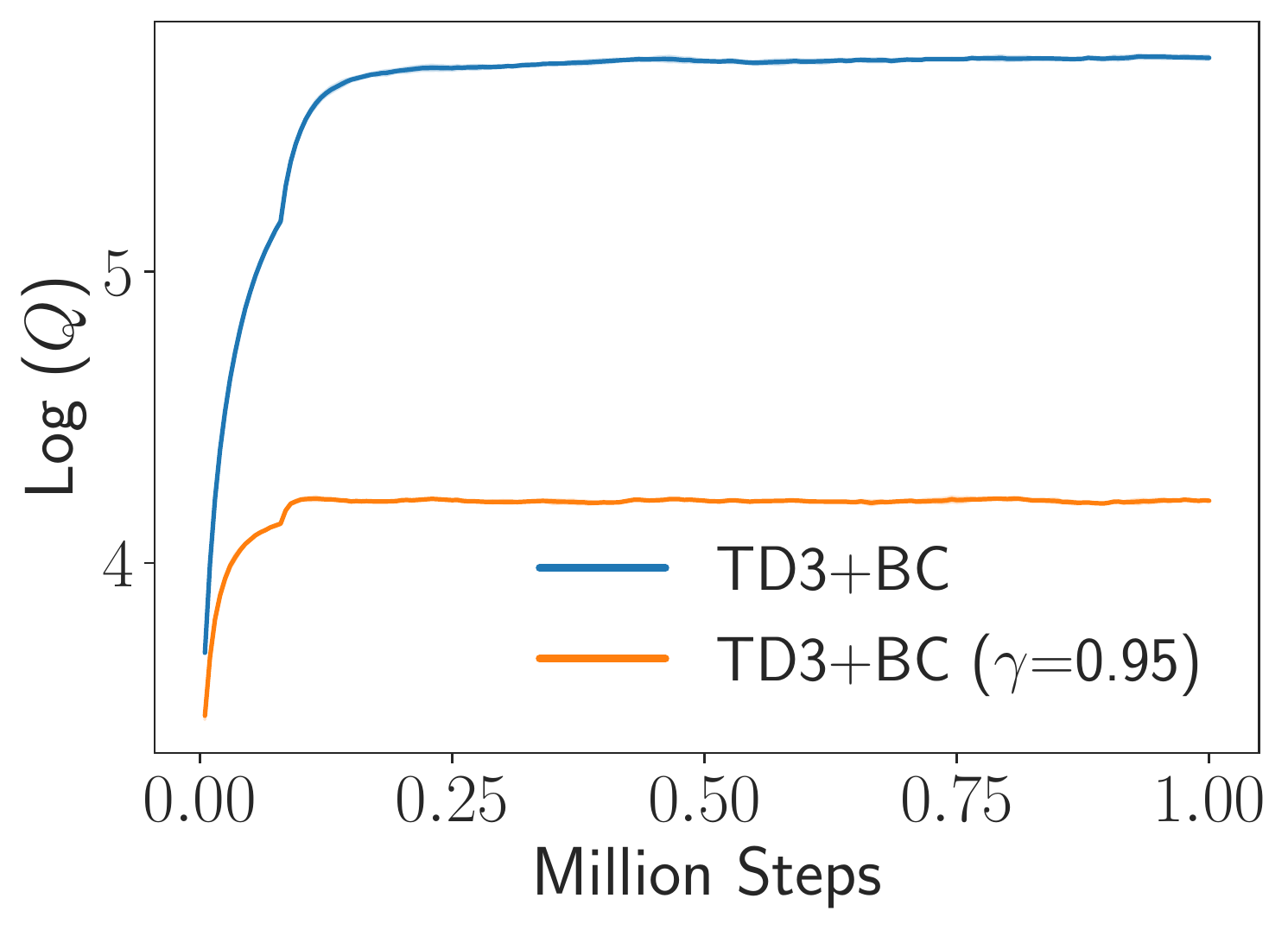}}
    \subfigure[med(50) noise(5)]{
    \includegraphics[scale=0.23]{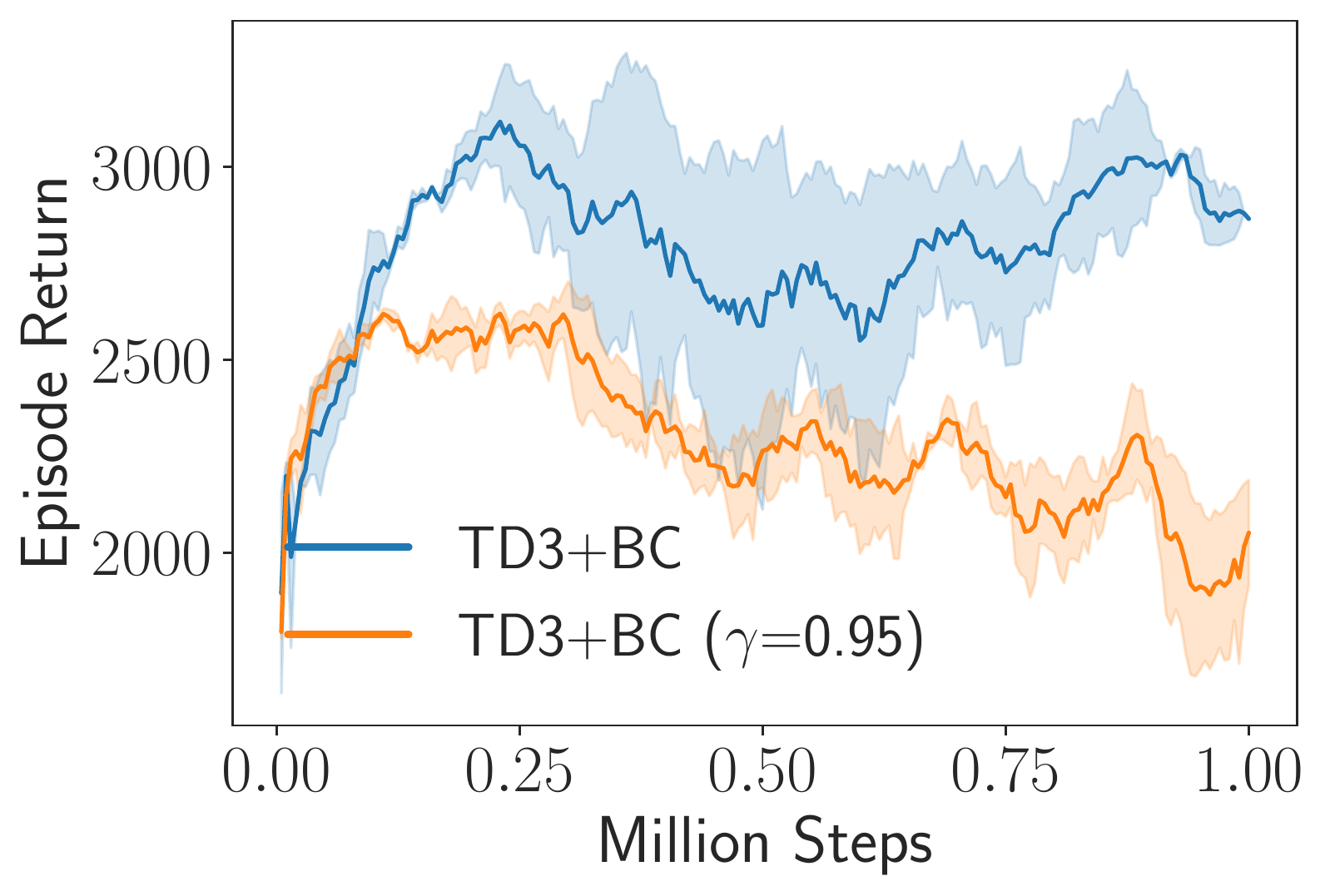}}
    \subfigure[med(50) noise(5)]{
    \includegraphics[scale=0.23]{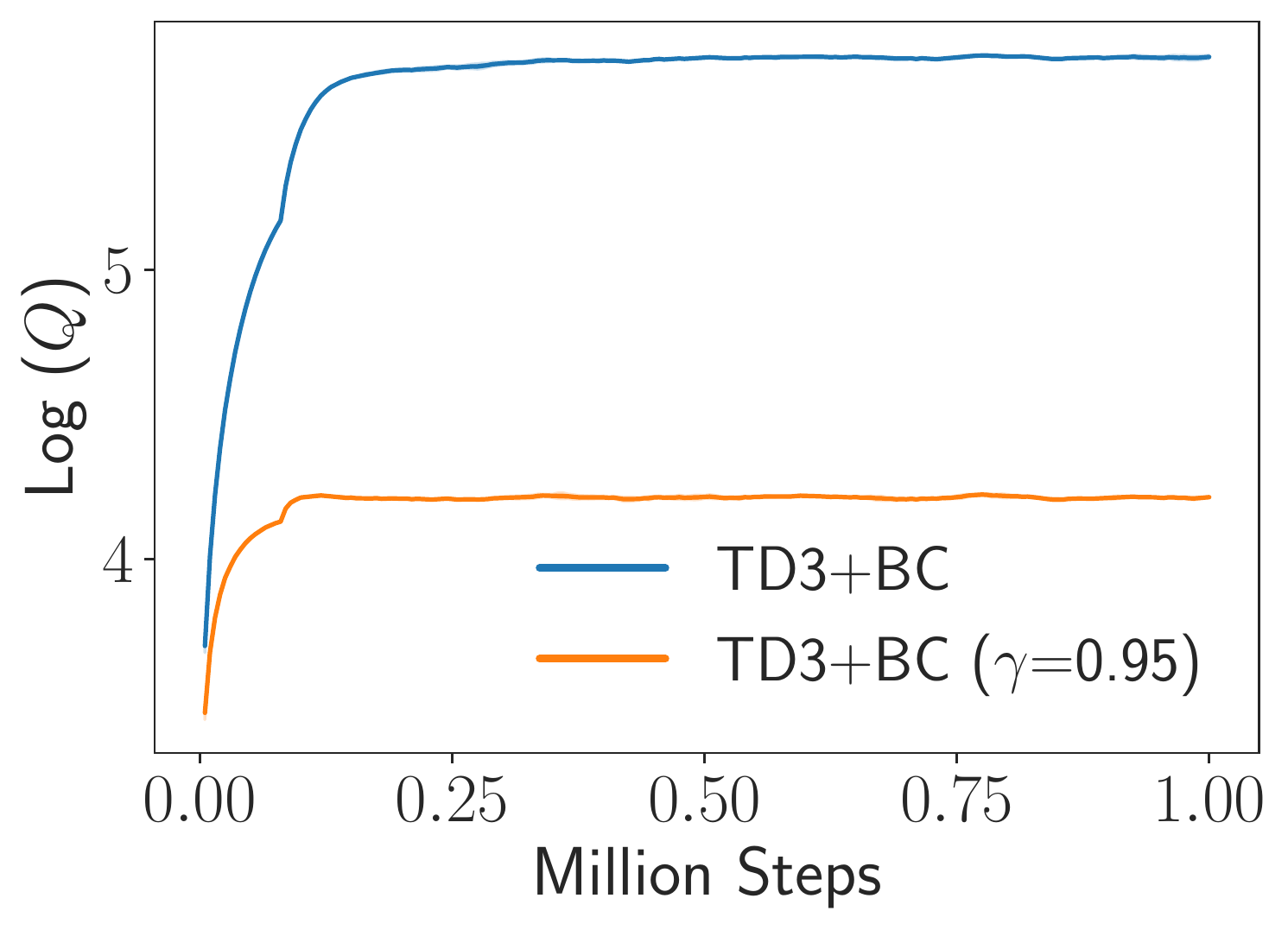}}
    
    \subfigure[med(50) noise(10)]{
    \includegraphics[scale=0.23]{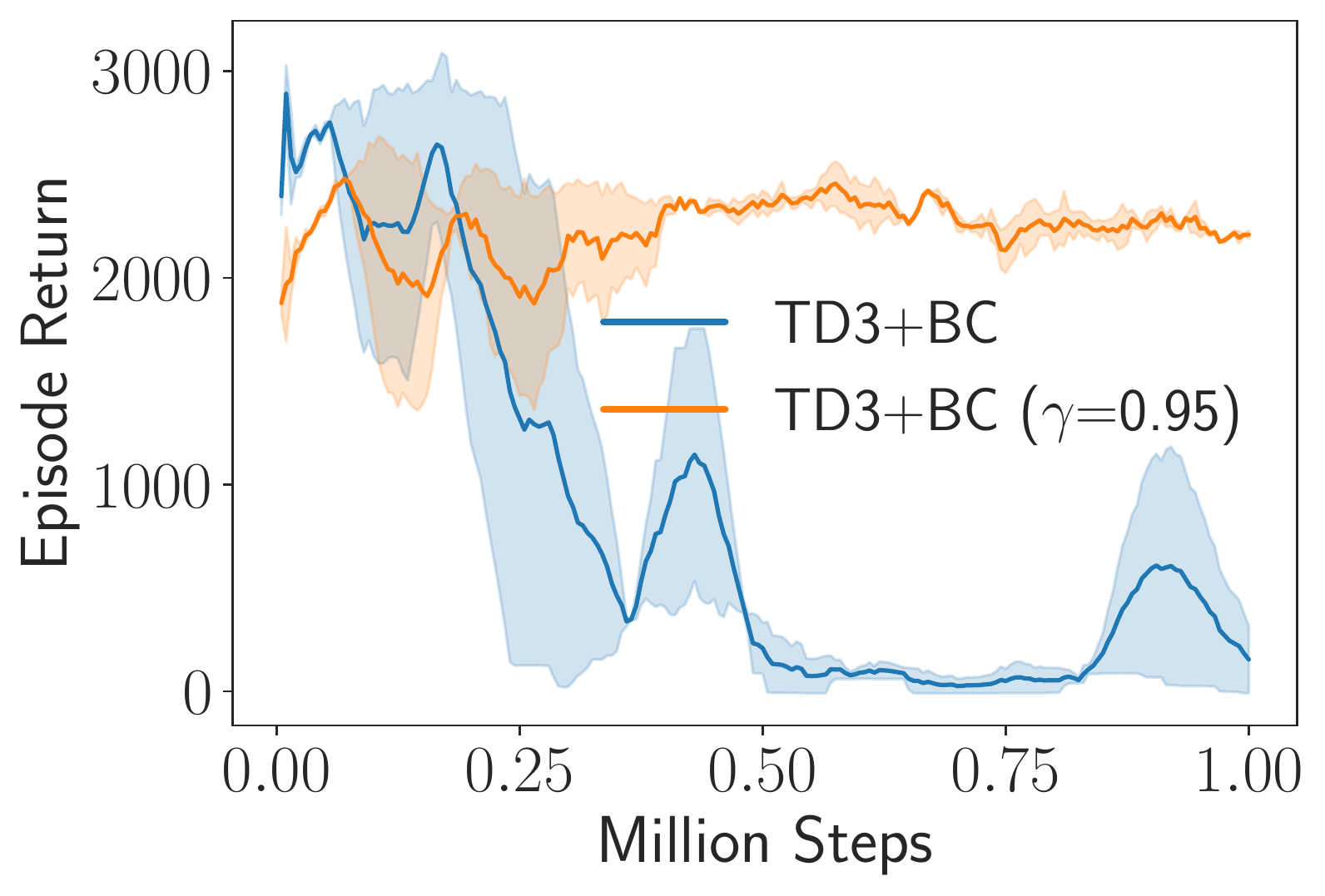}}
    \subfigure[med(50) noise(10)]{
    \includegraphics[scale=0.23]{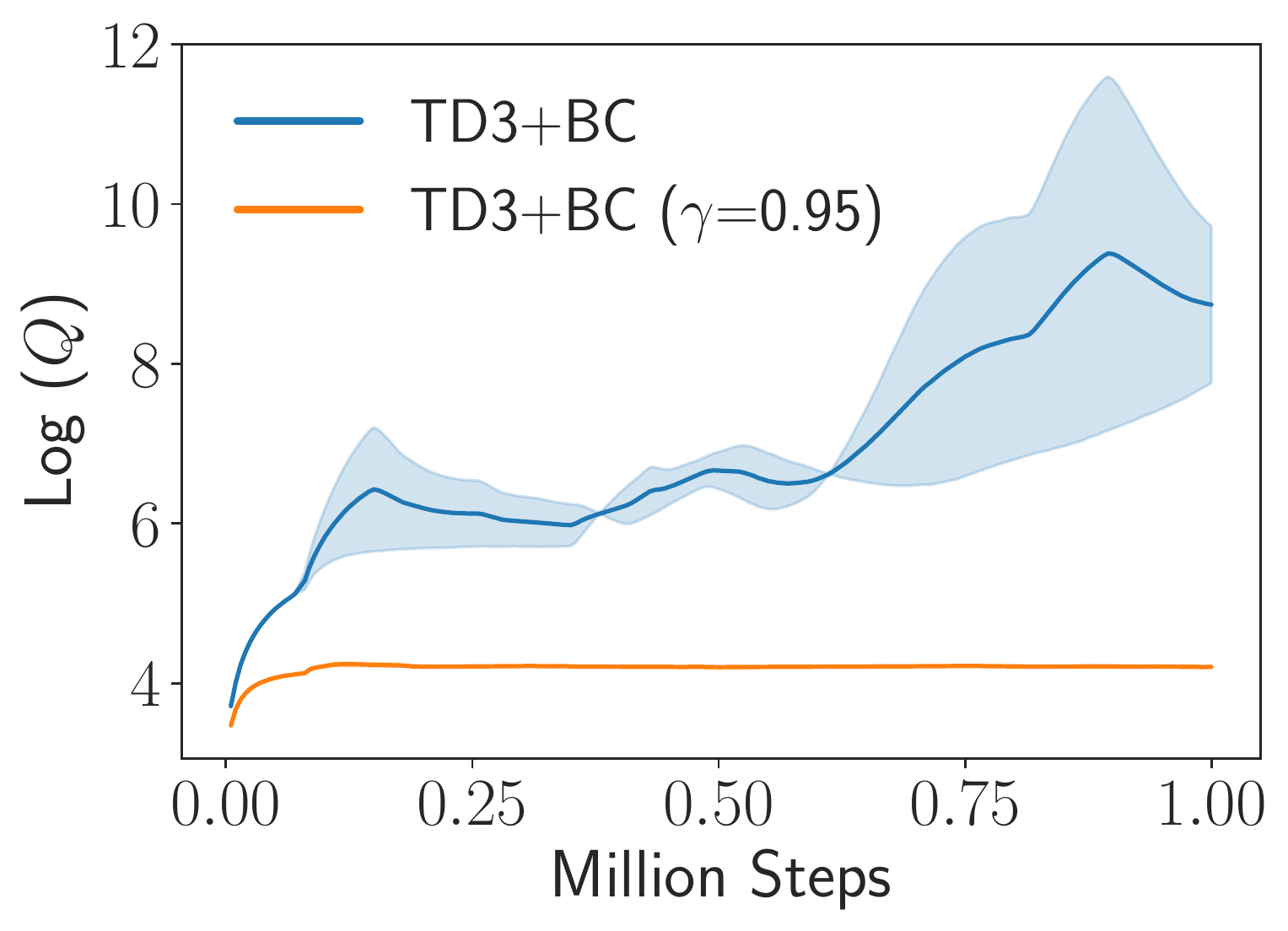}}
     \subfigure[med(50) noise(15)]{
    \includegraphics[scale=0.23]{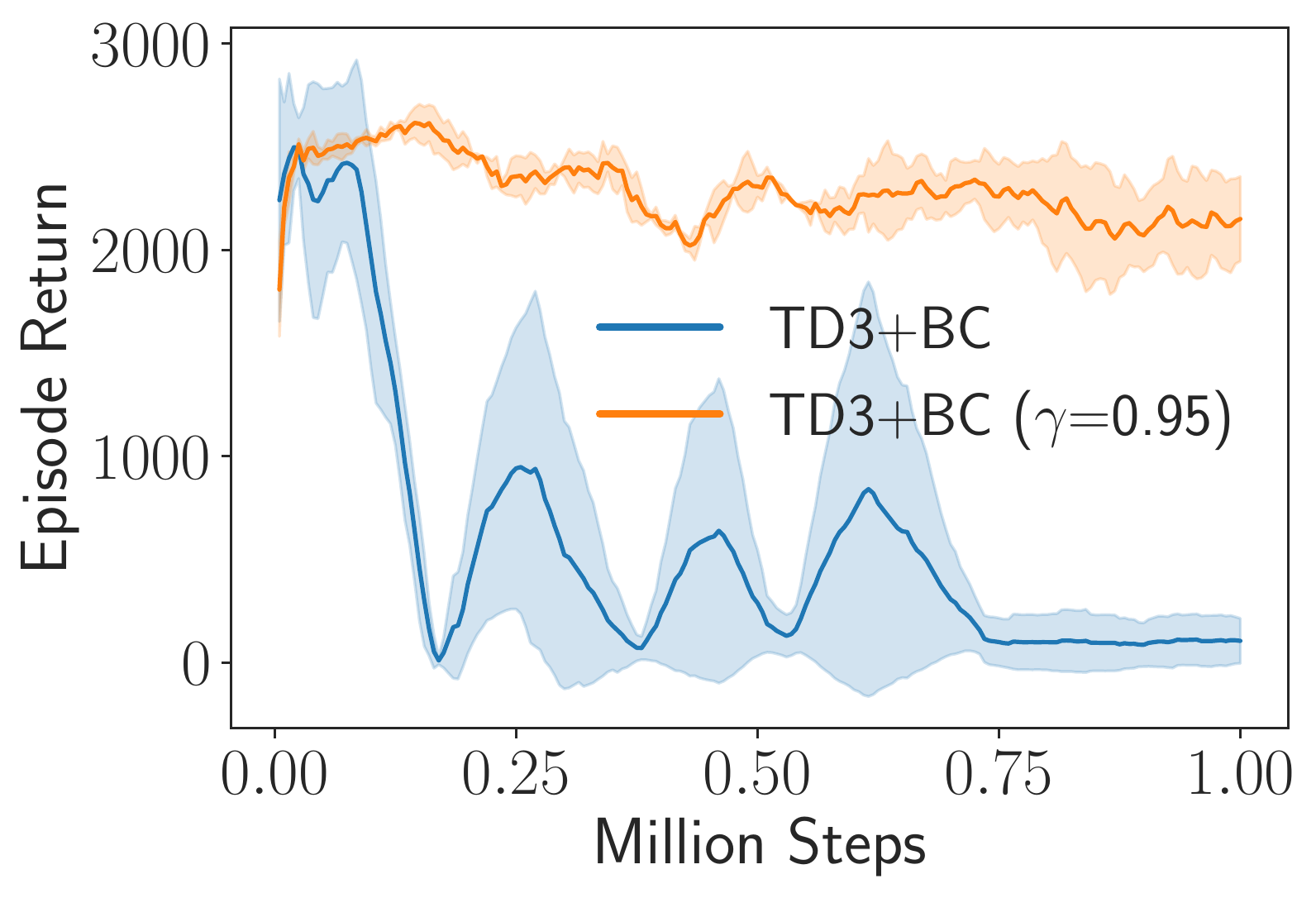}}
    \subfigure[med(50) noise(15)]{
    \includegraphics[scale=0.23]{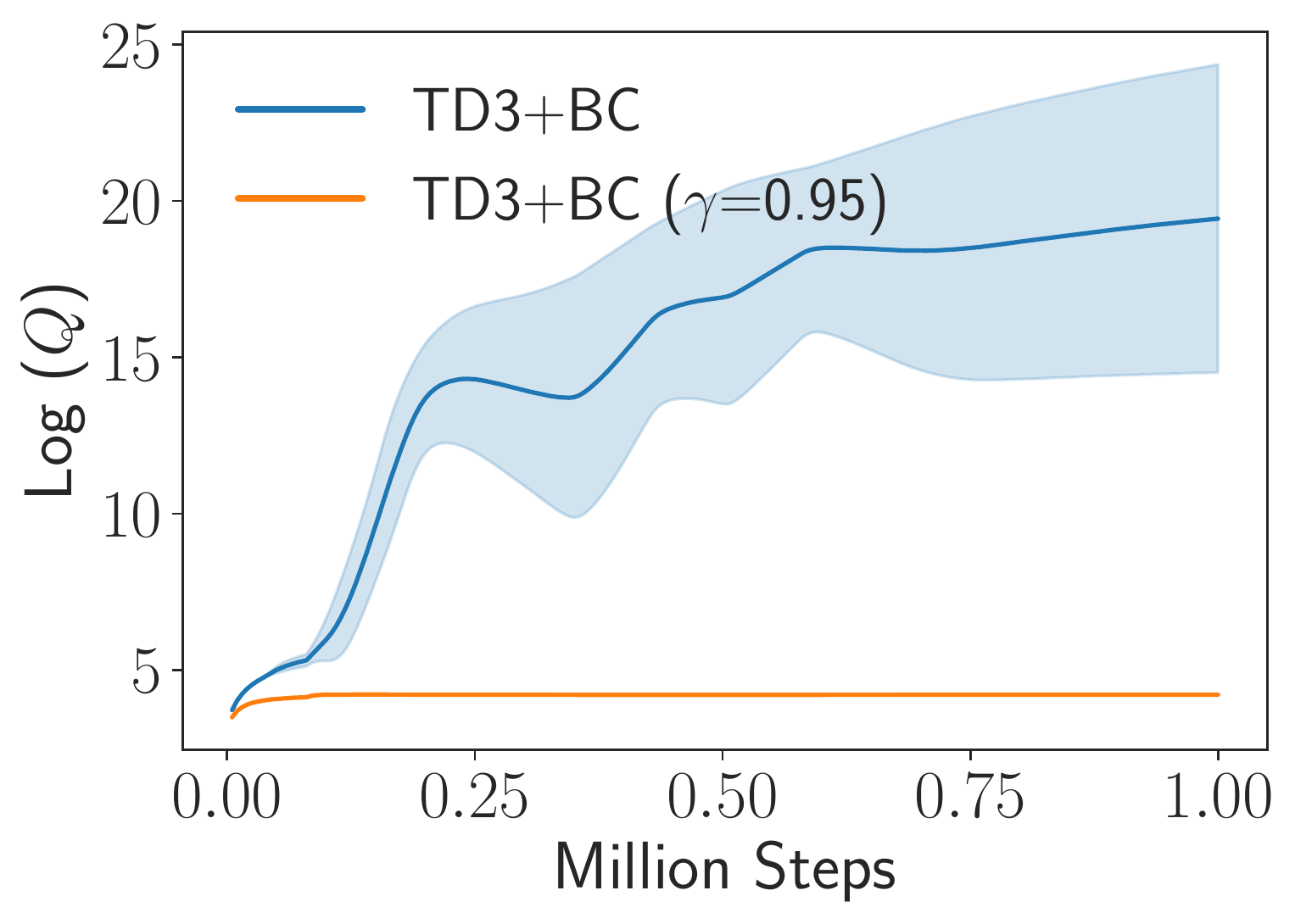}}
    
    \subfigure[med(50) noise(20)]{
    \includegraphics[scale=0.23]{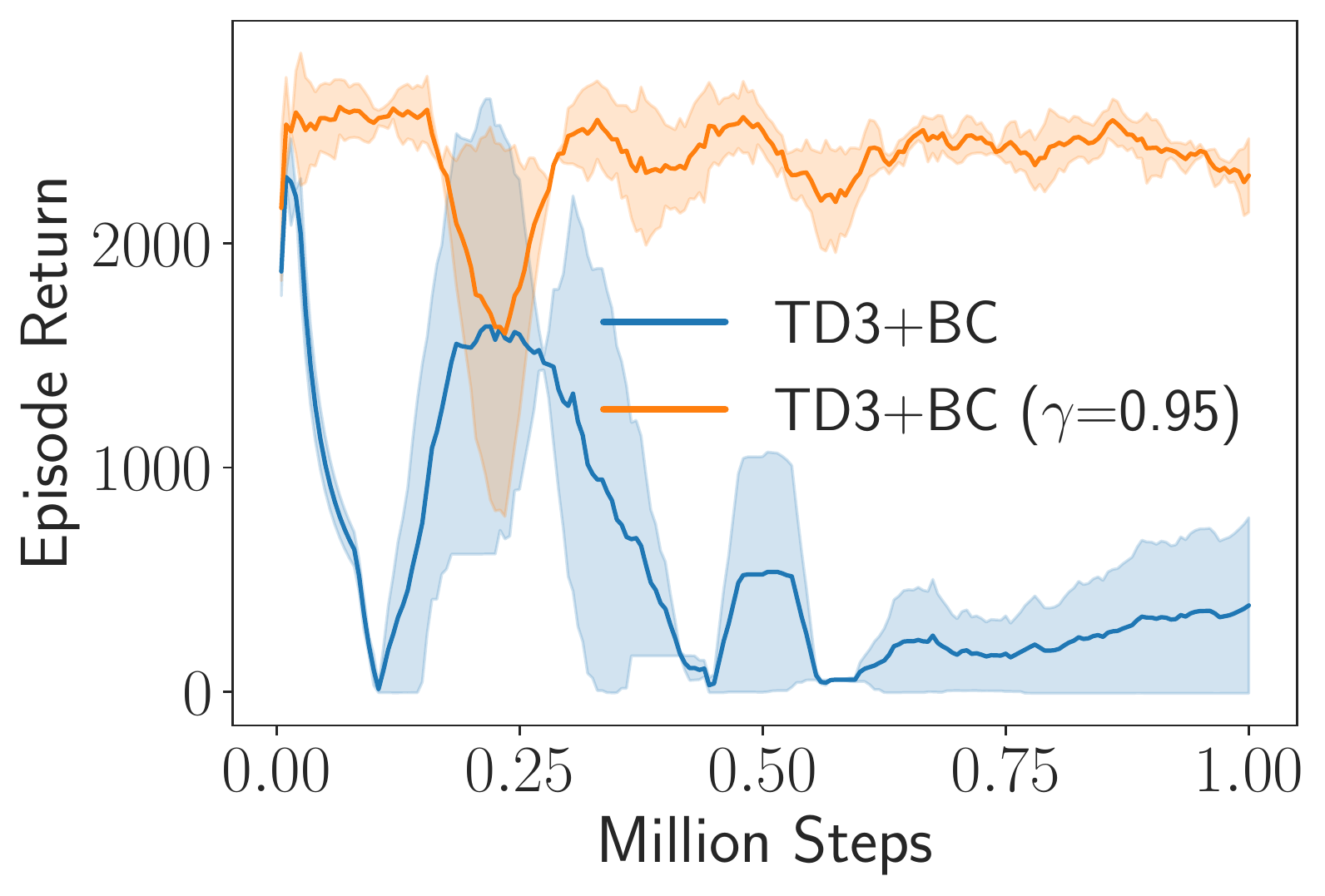}}
    \subfigure[med(50) noise(20)]{
    \includegraphics[scale=0.23]{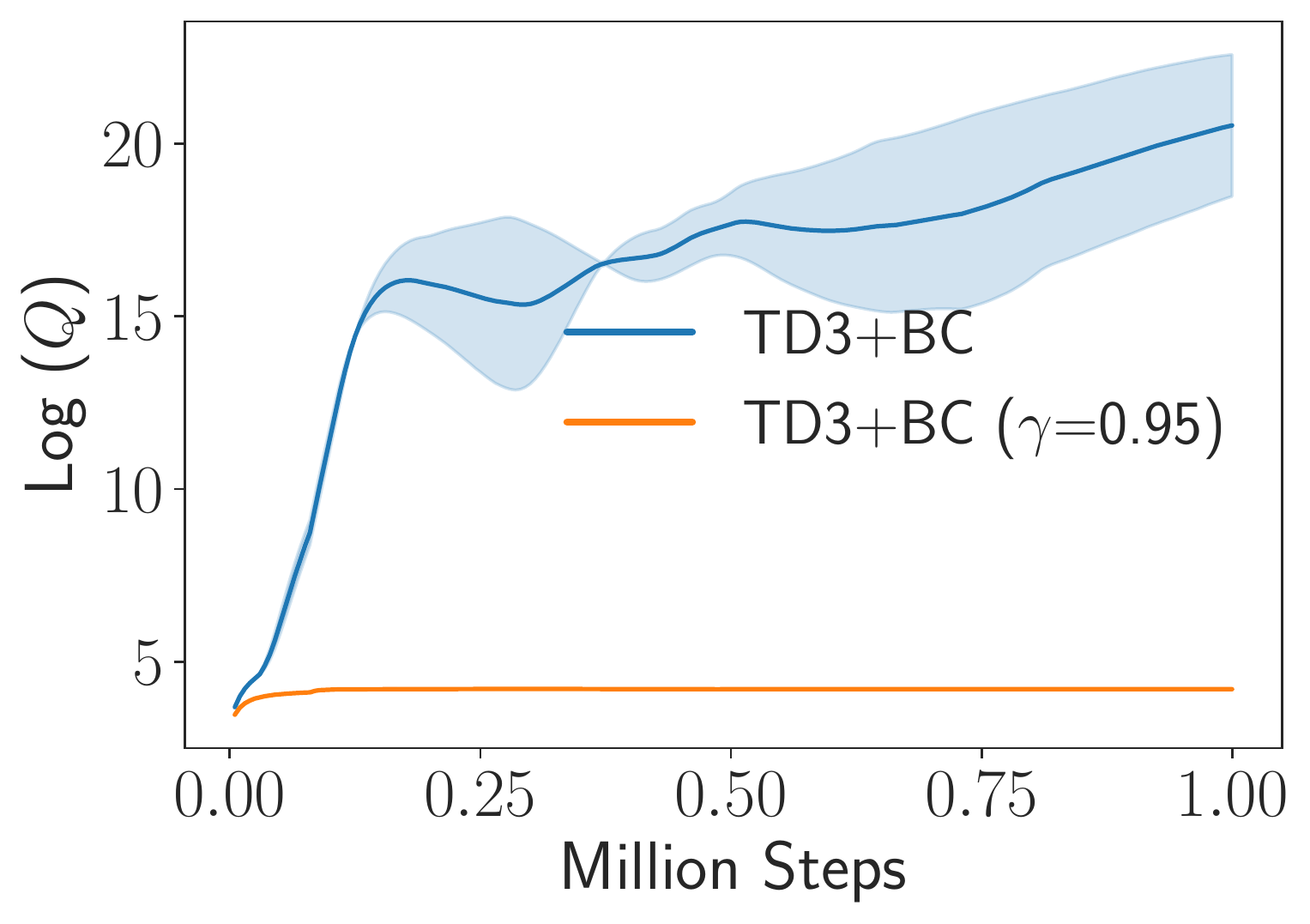}}
     \subfigure[med(50) noise(25)]{
    \includegraphics[scale=0.23]{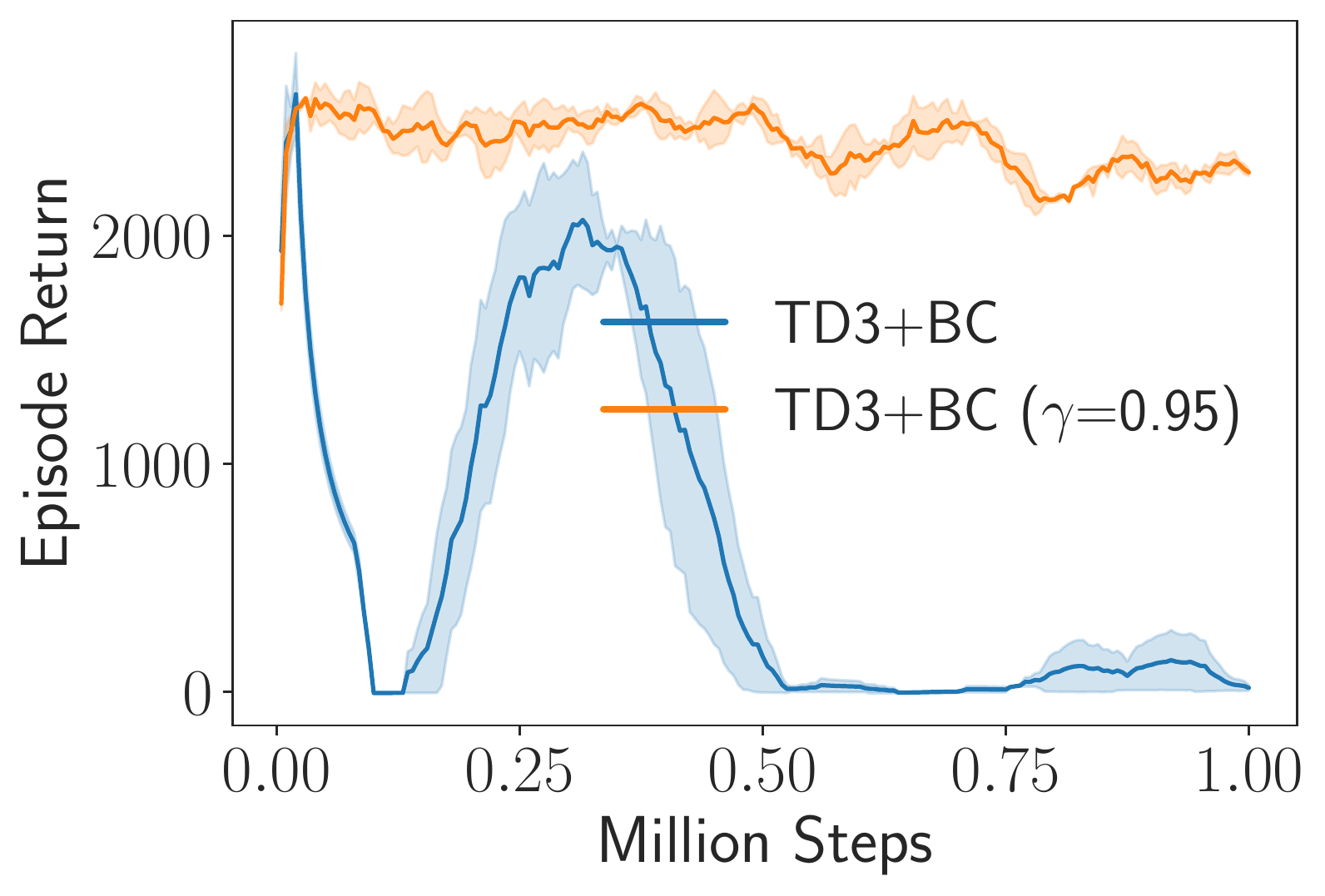}}
    \subfigure[med(50) noise(25)]{
    \includegraphics[scale=0.23]{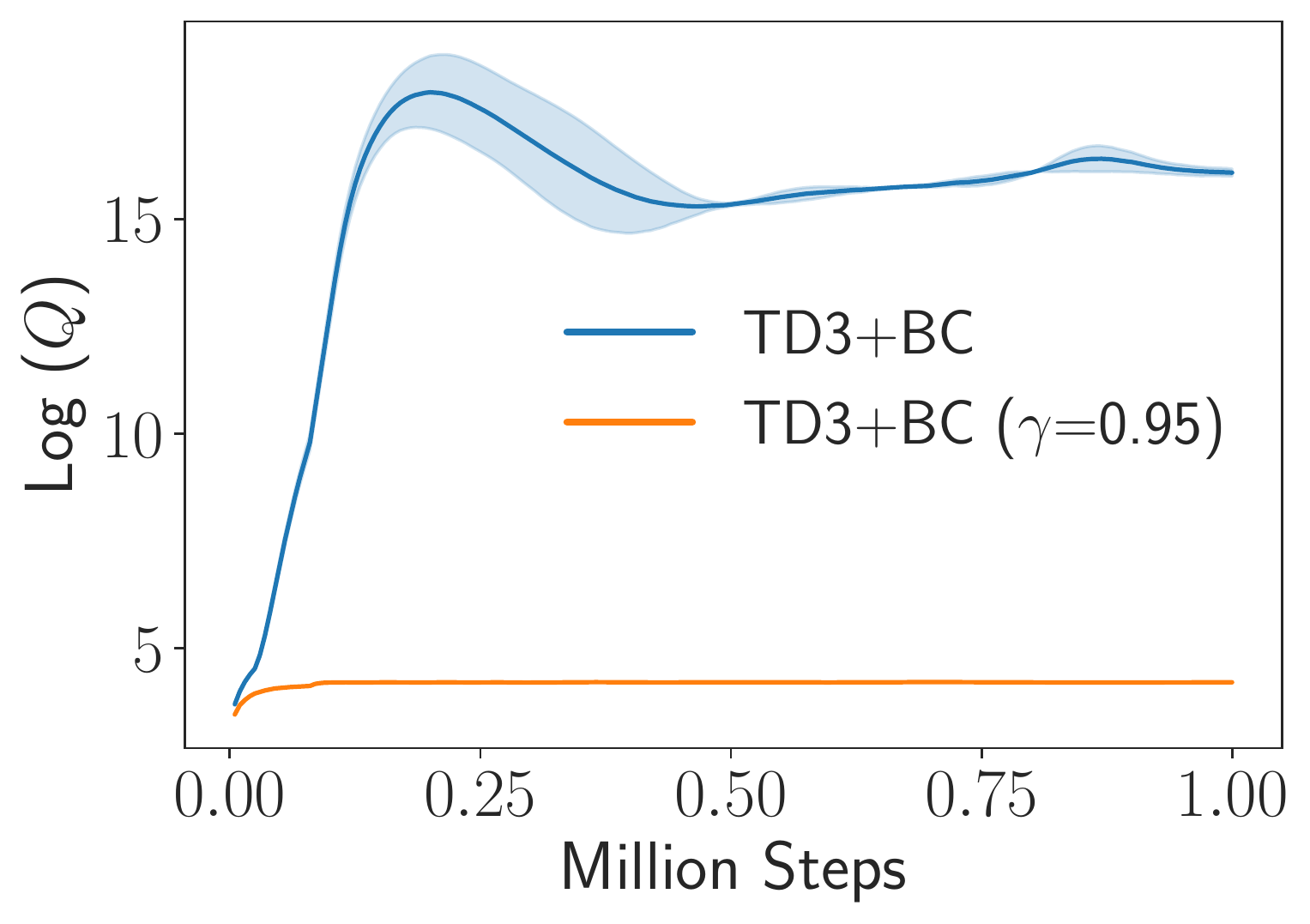}}
    \caption{Experimental results on walker2d task consisting of 50 medium trajectories and $x$ noised trajectories.
    The evaluation metric is the episode return and log $Q$-value.}
    \label{fig:my_label}
\end{figure}

\begin{figure}[h]
    \centering
    \subfigure[med(50) noise(0)]{
    \includegraphics[scale=0.23]{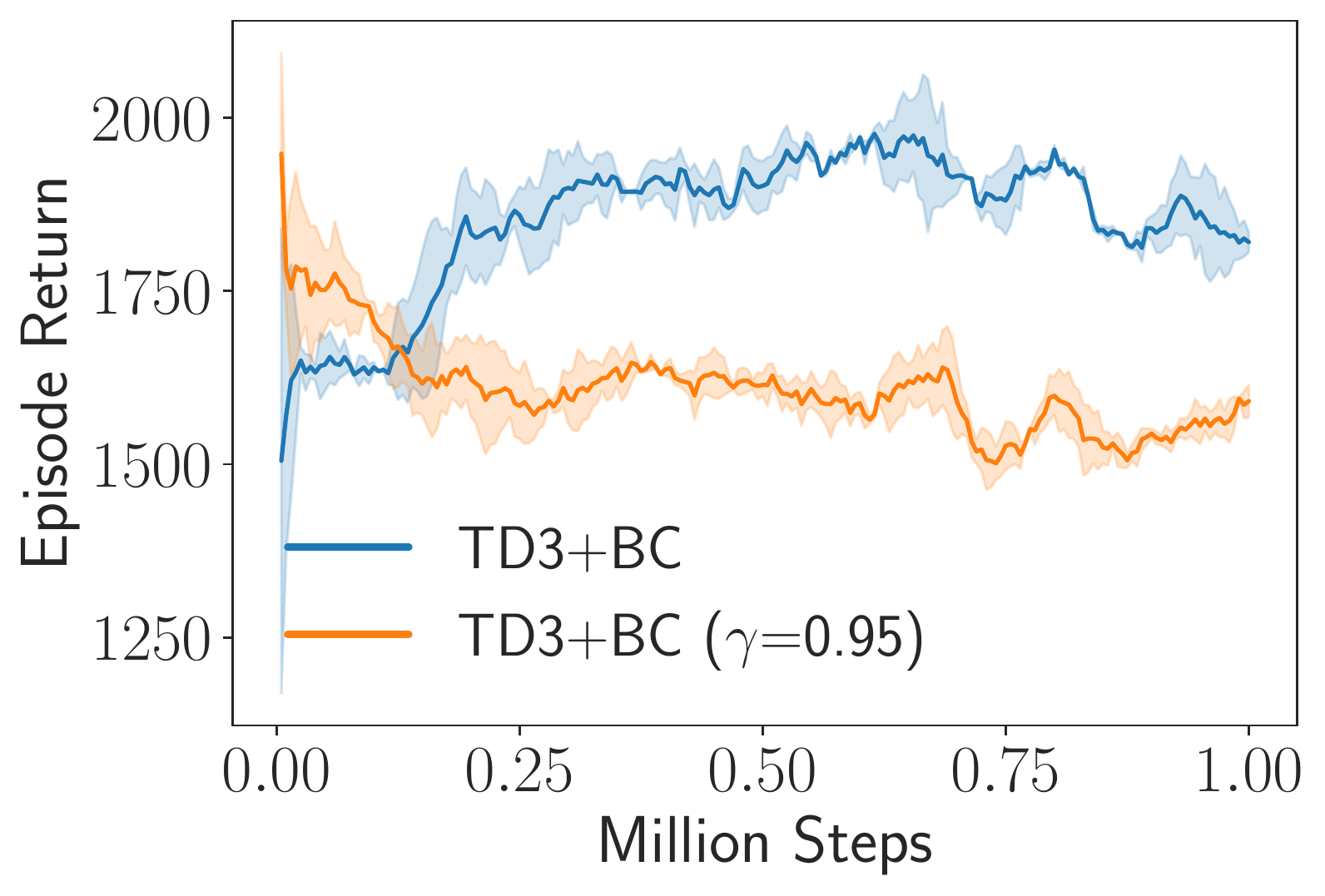}}
    \subfigure[med(50) noise(0)]{
    \includegraphics[scale=0.23]{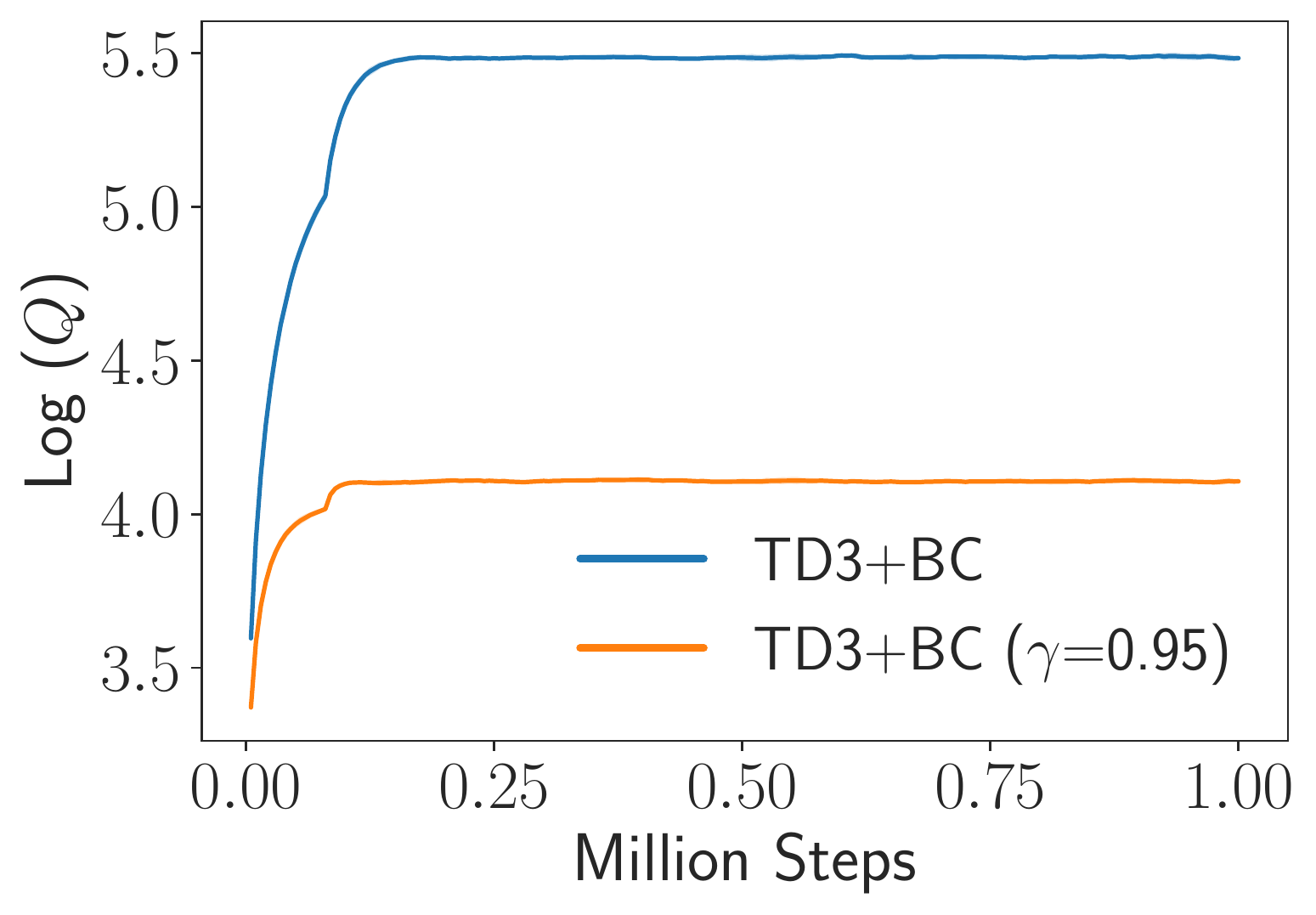}}
    \subfigure[med(50) noise(5)]{
    \includegraphics[scale=0.23]{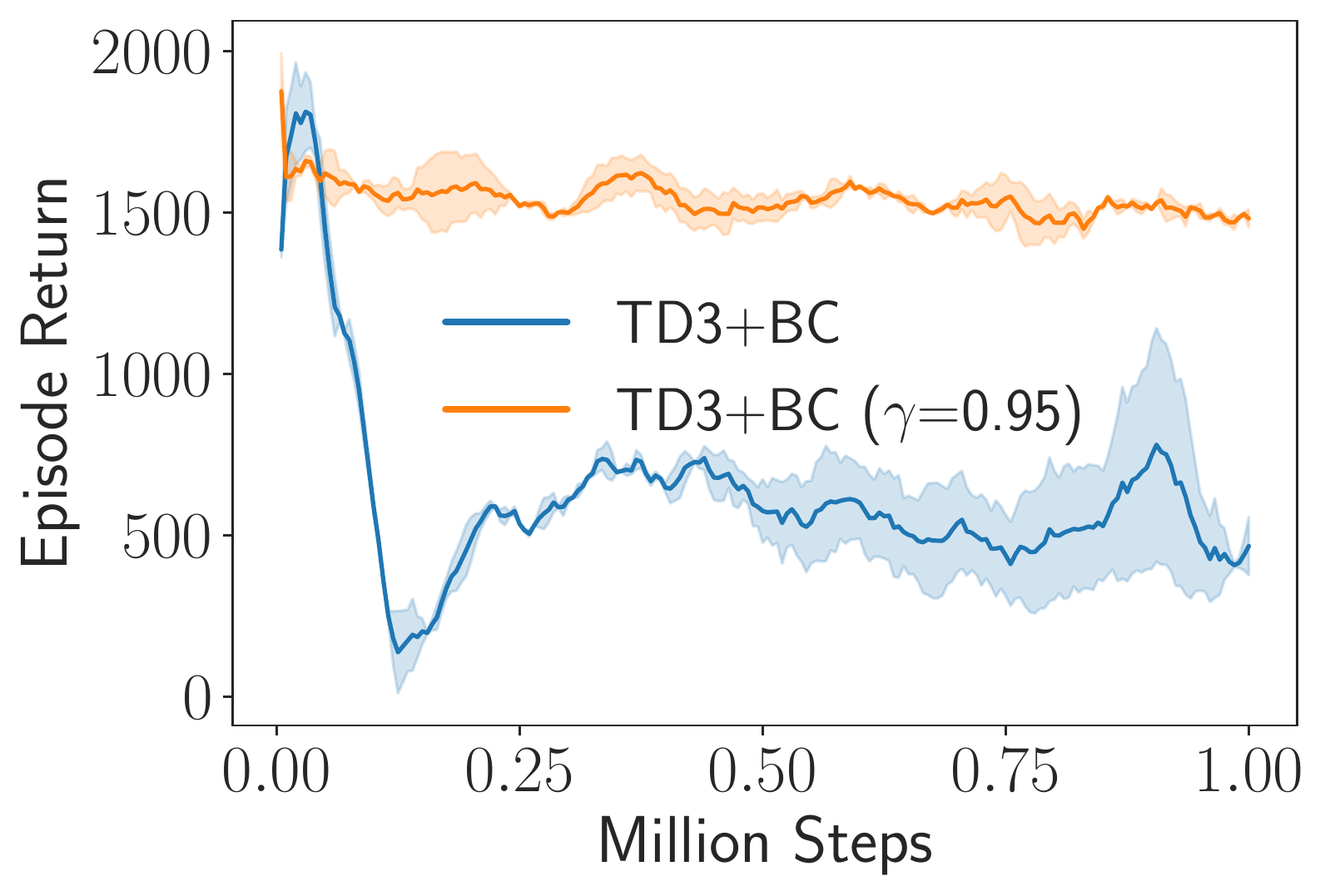}}
    \subfigure[med(50) noise(5)]{
    \includegraphics[scale=0.23]{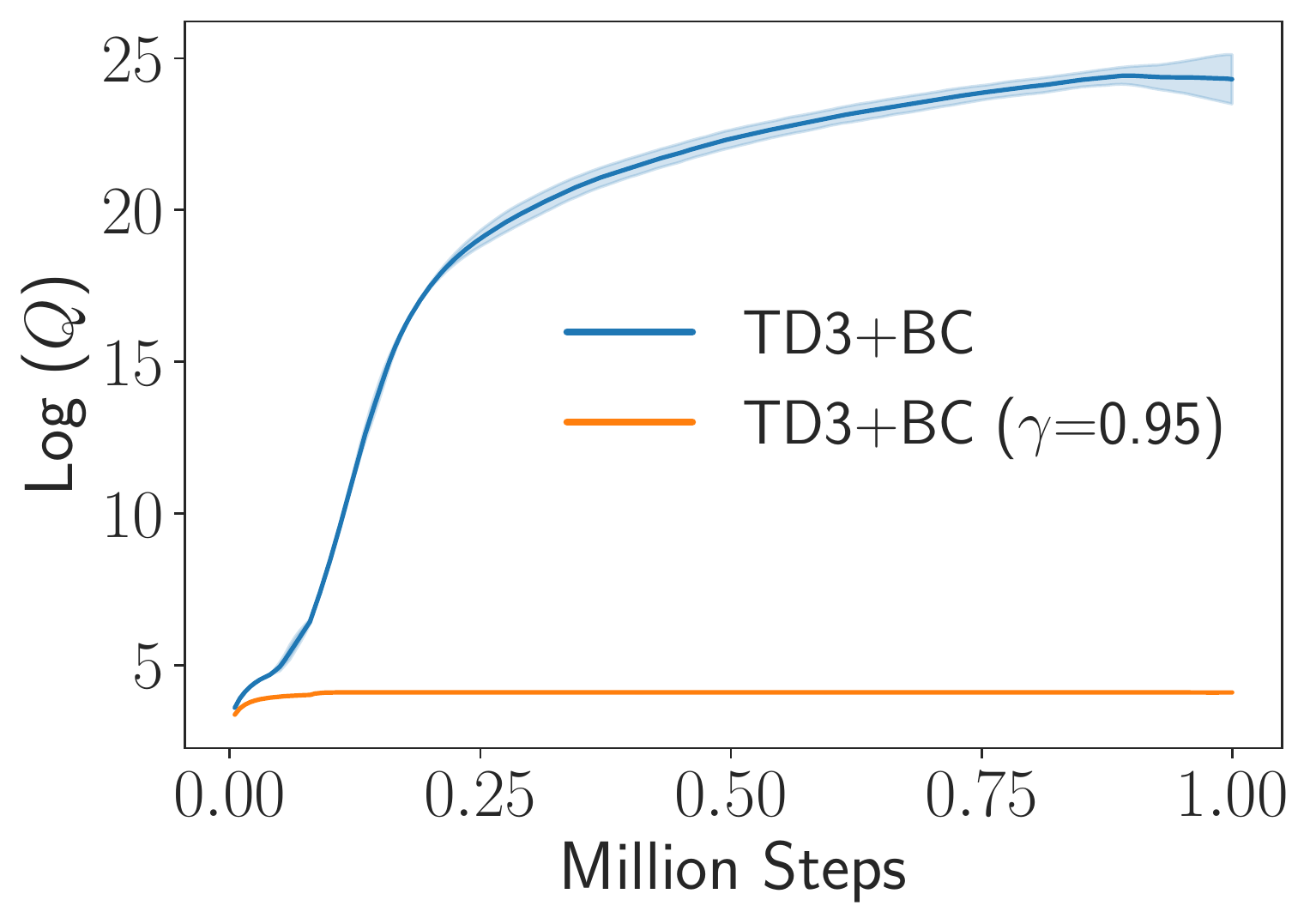}}
    
    \subfigure[med(50) noise(10)]{
    \includegraphics[scale=0.23]{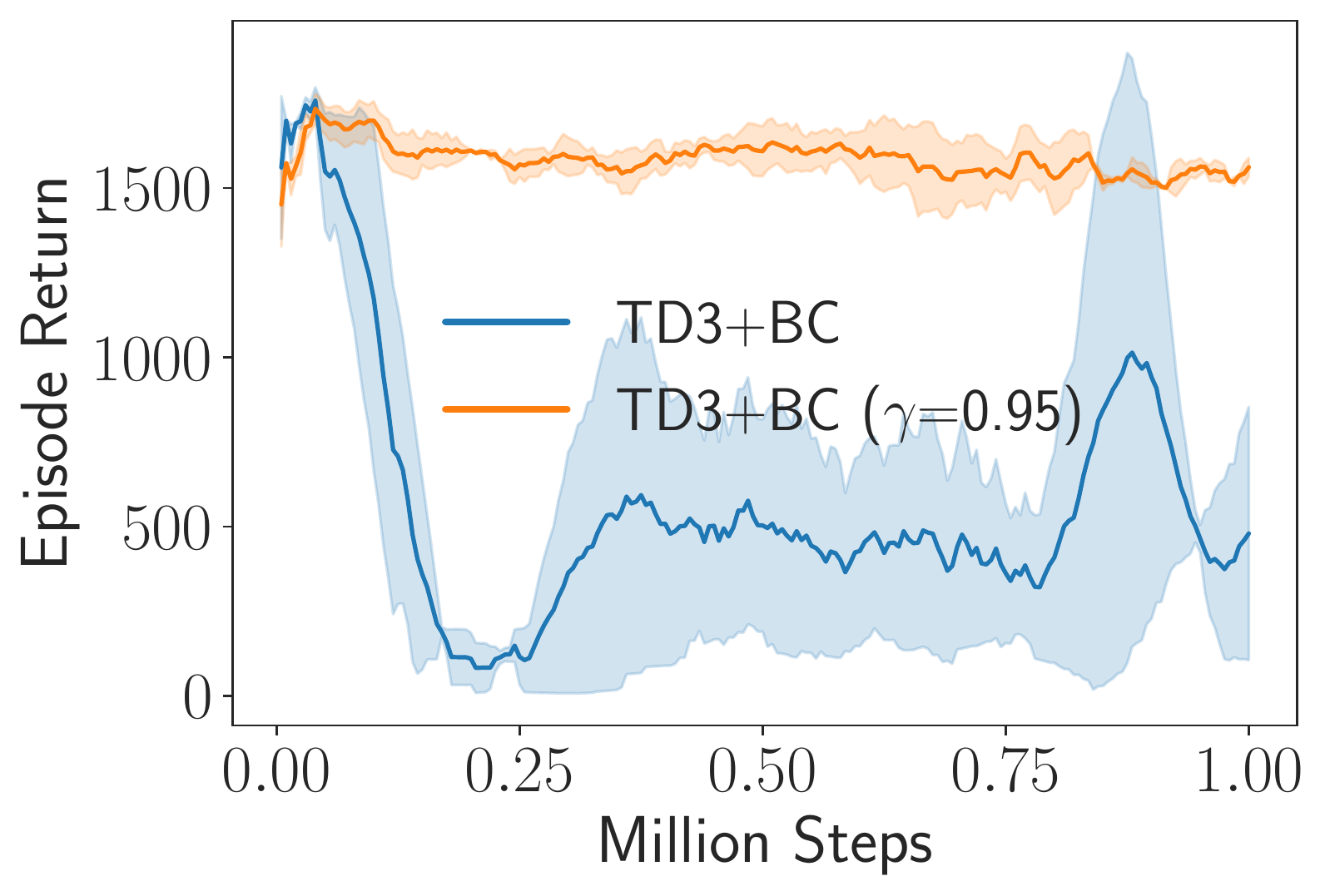}}
    \subfigure[med(50) noise(10)]{
    \includegraphics[scale=0.23]{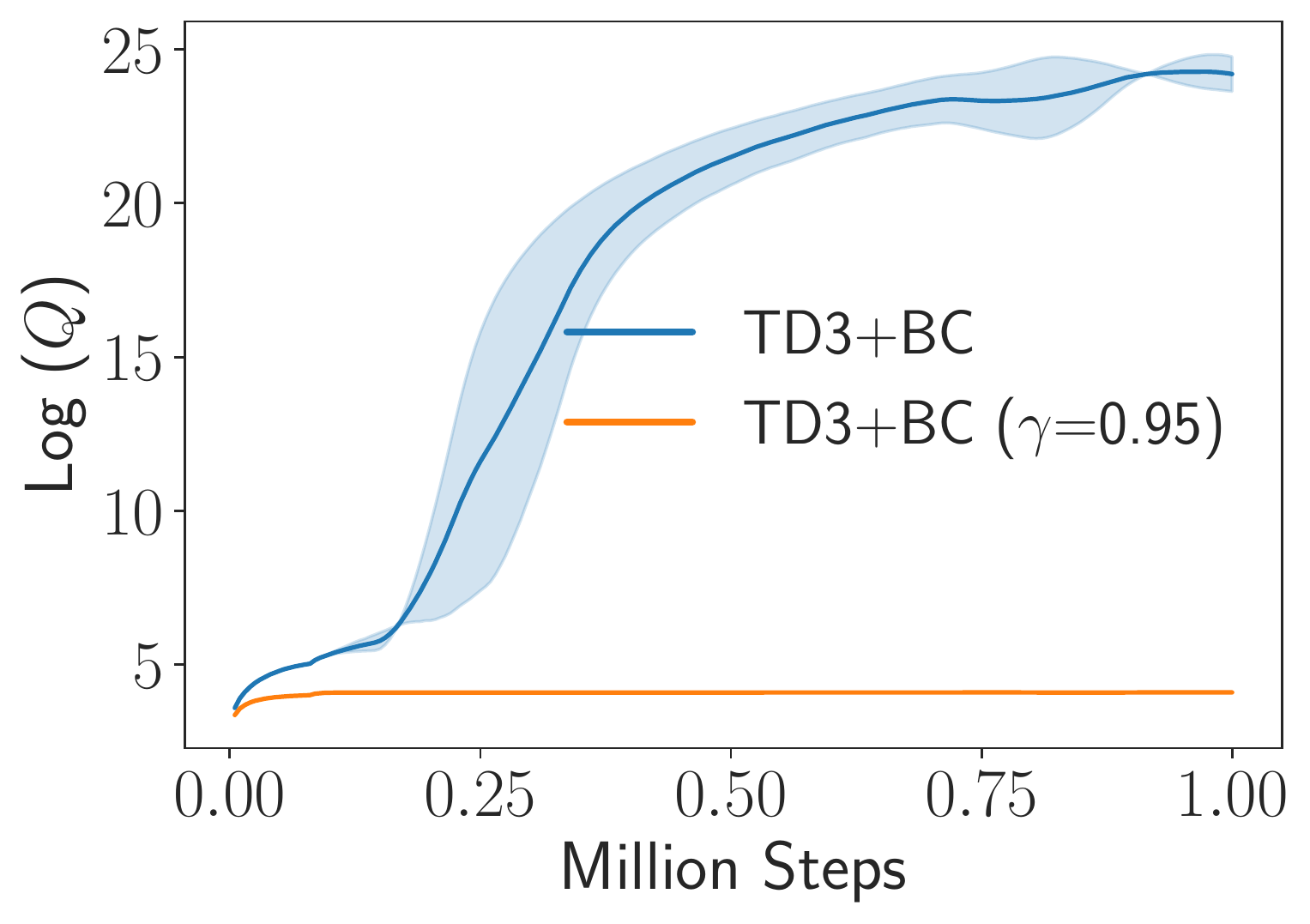}}
     \subfigure[med(50) noise(15)]{
    \includegraphics[scale=0.23]{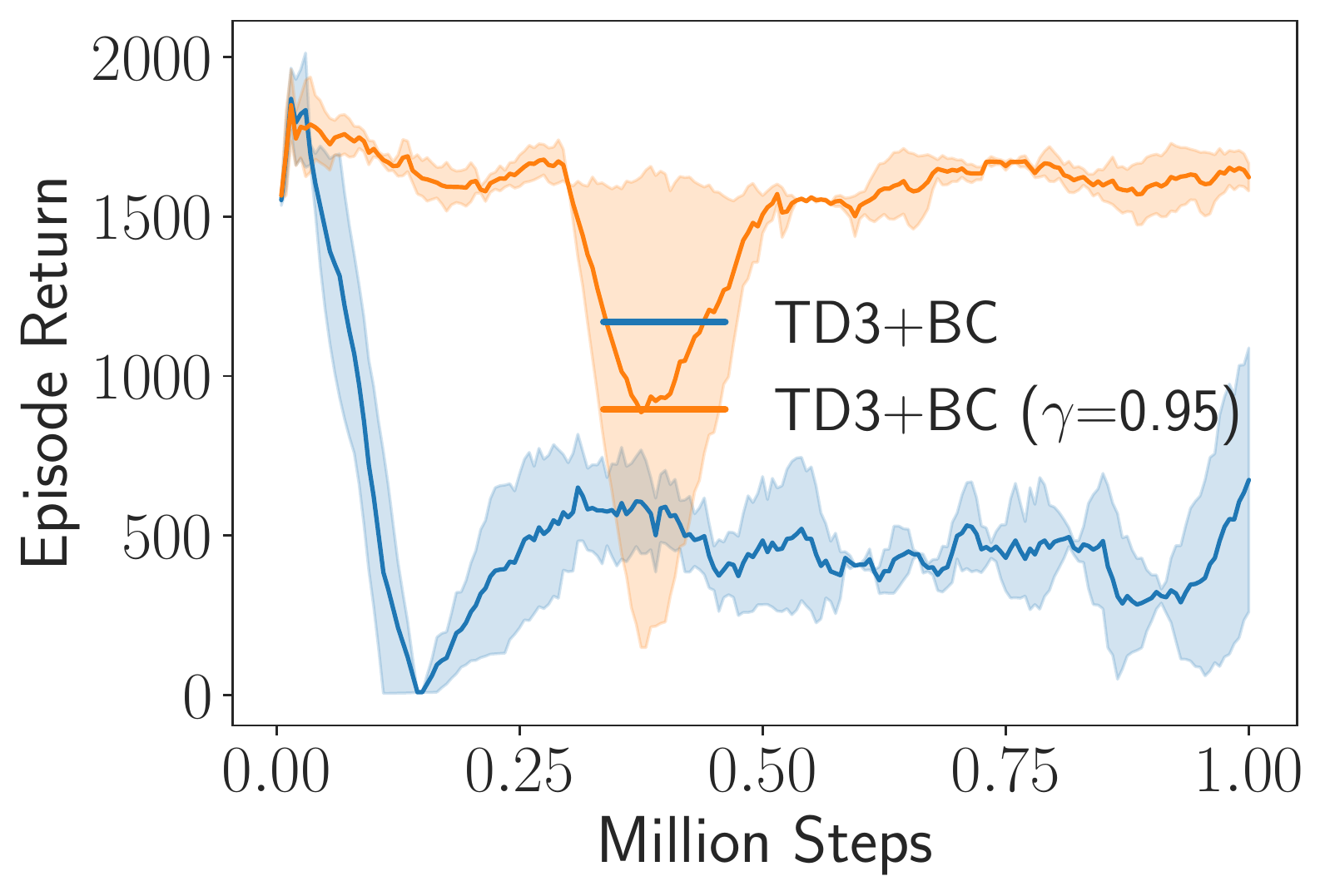}}
    \subfigure[med(50) noise(15)]{
    \includegraphics[scale=0.23]{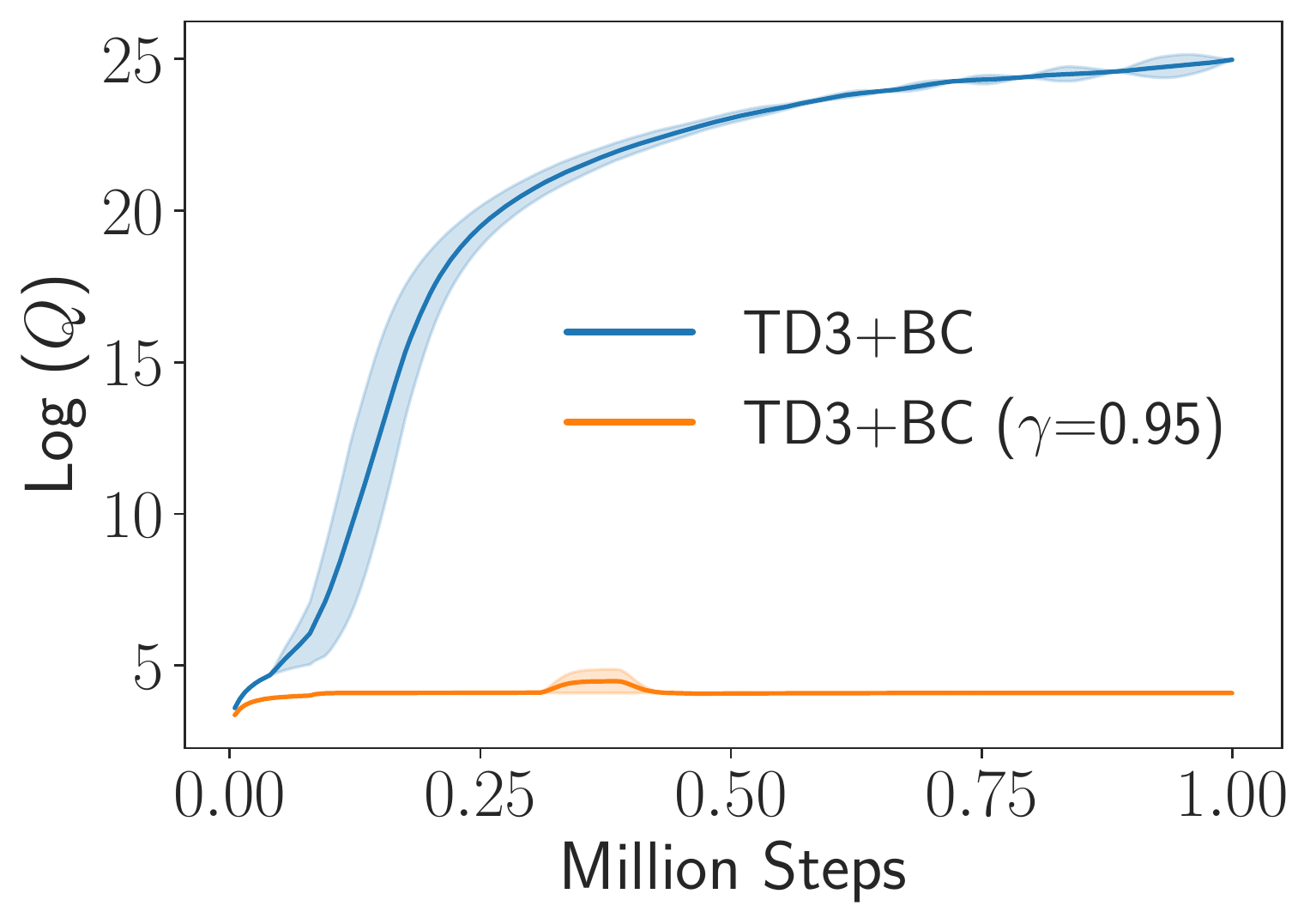}}
    
    \subfigure[med(50) noise(20)]{
    \includegraphics[scale=0.23]{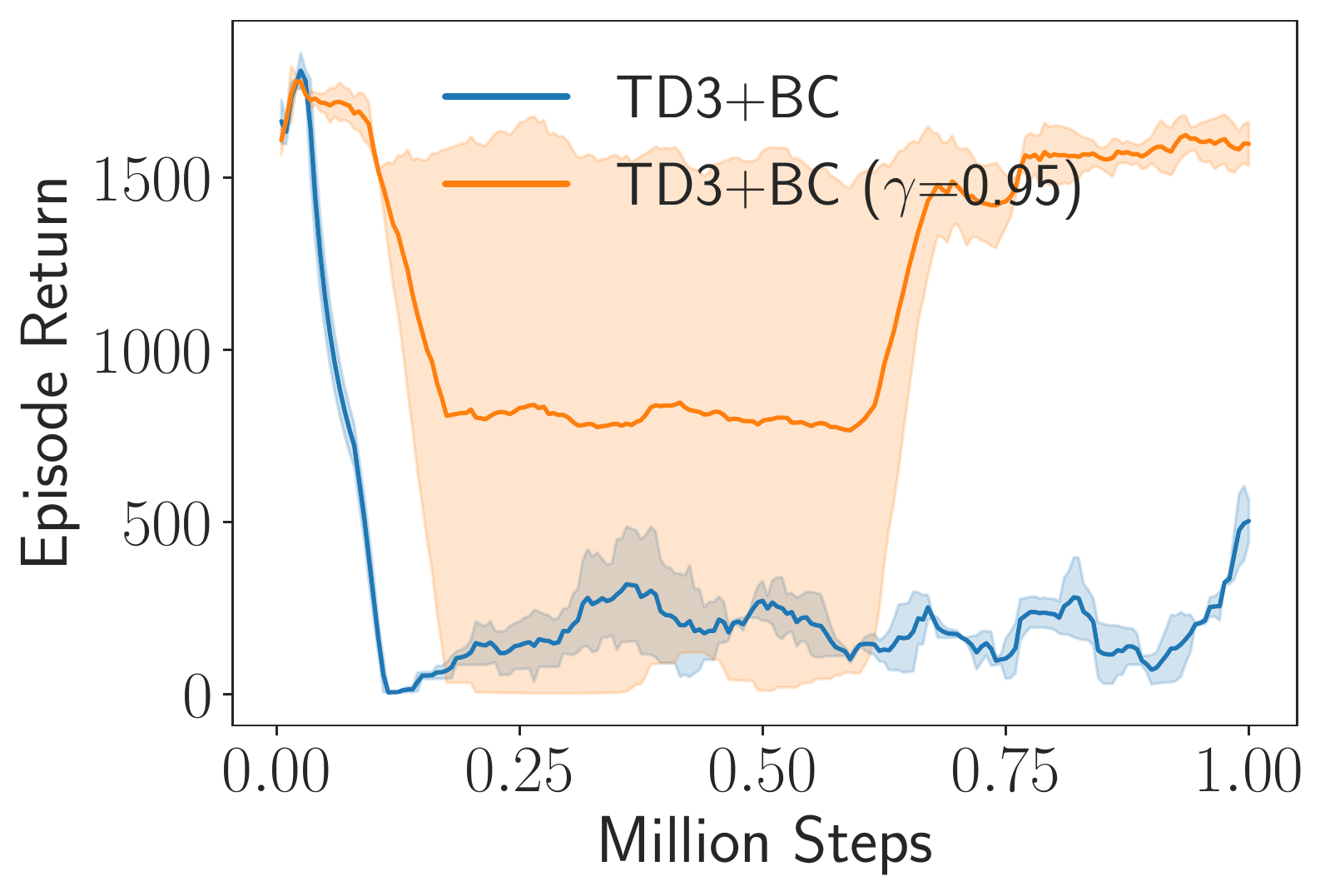}}
    \subfigure[med(50) noise(20)]{
    \includegraphics[scale=0.23]{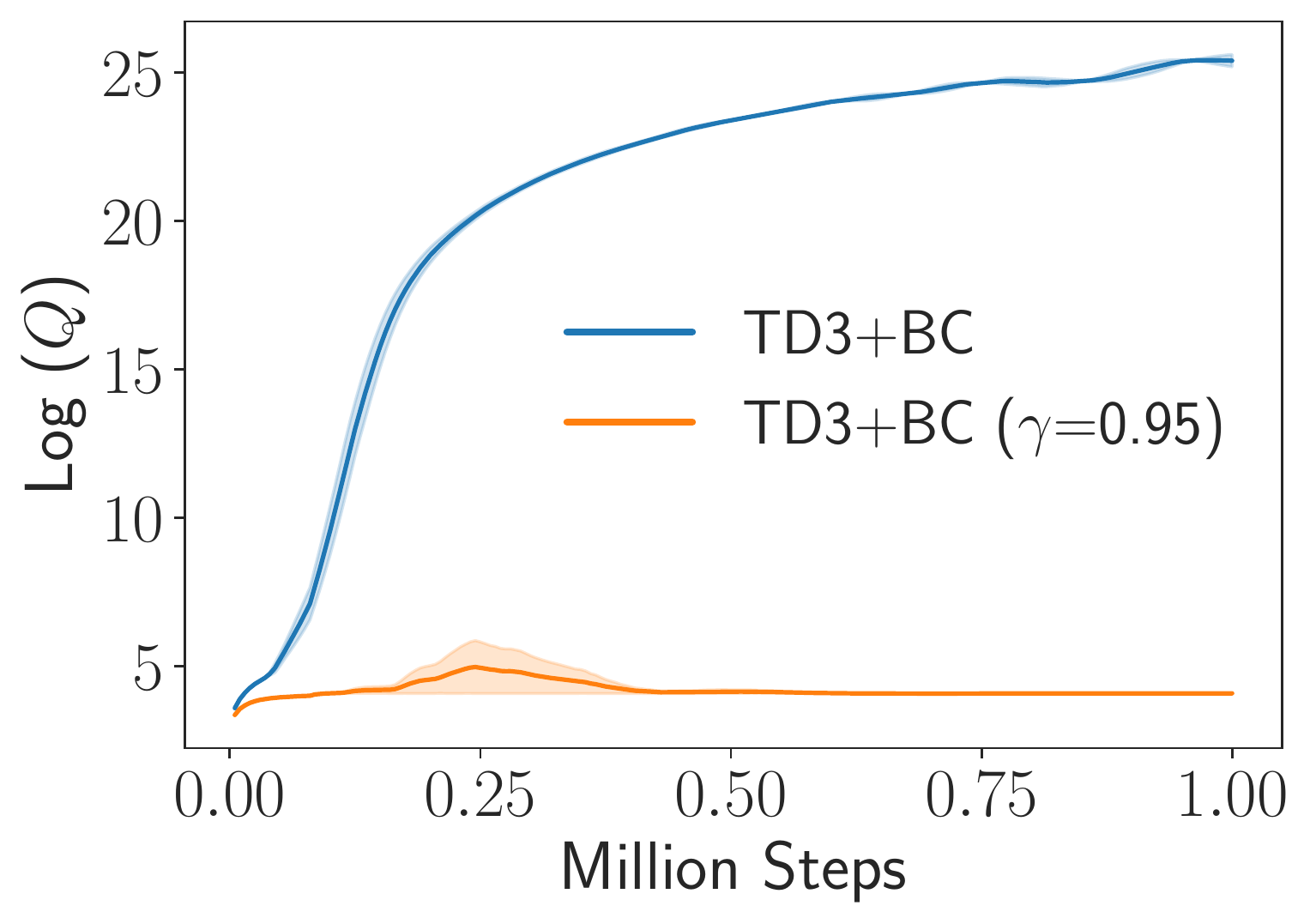}}
     \subfigure[med(50) noise(25)]{
    \includegraphics[scale=0.23]{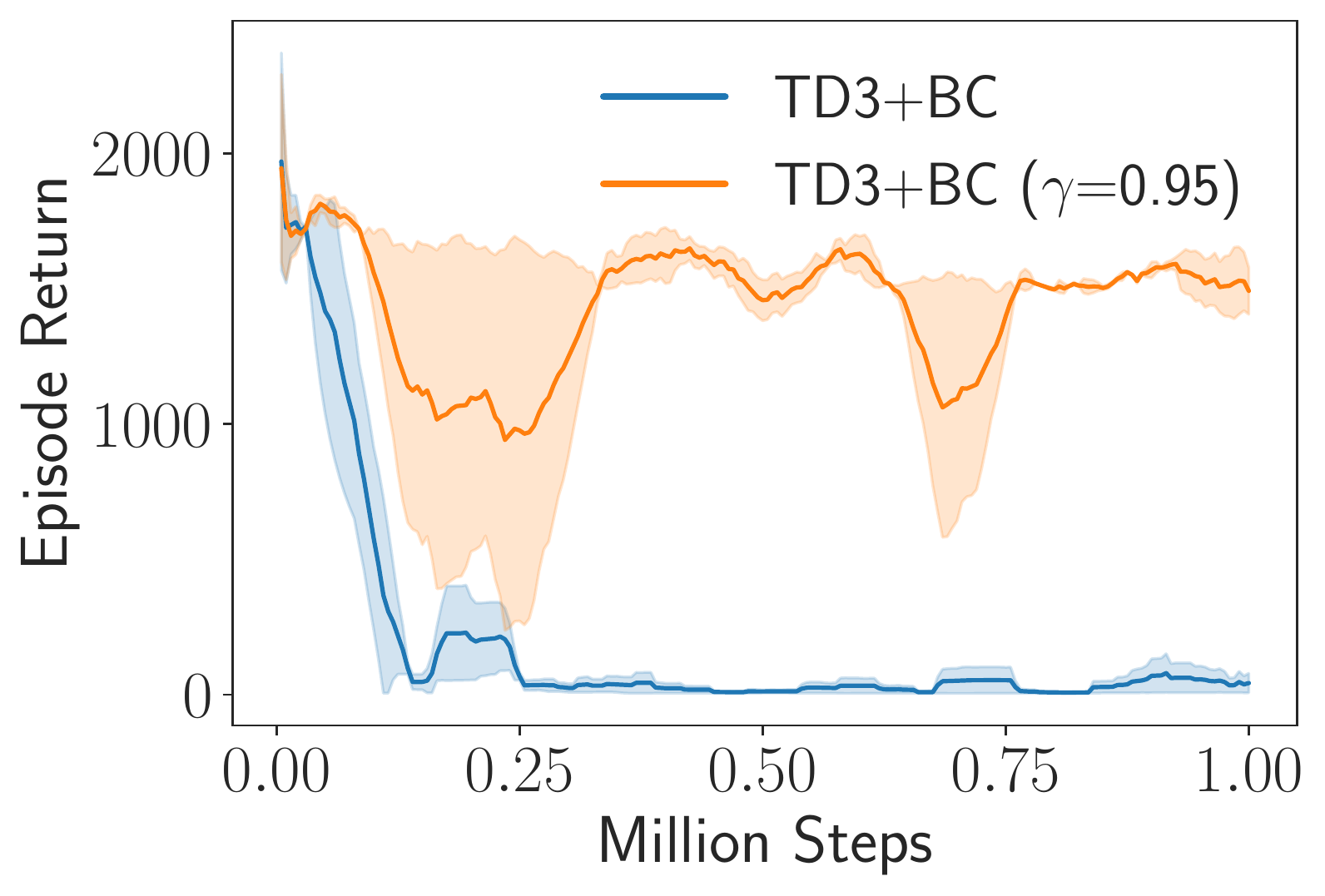}}
    \subfigure[med(50) noise(25)]{
    \includegraphics[scale=0.23]{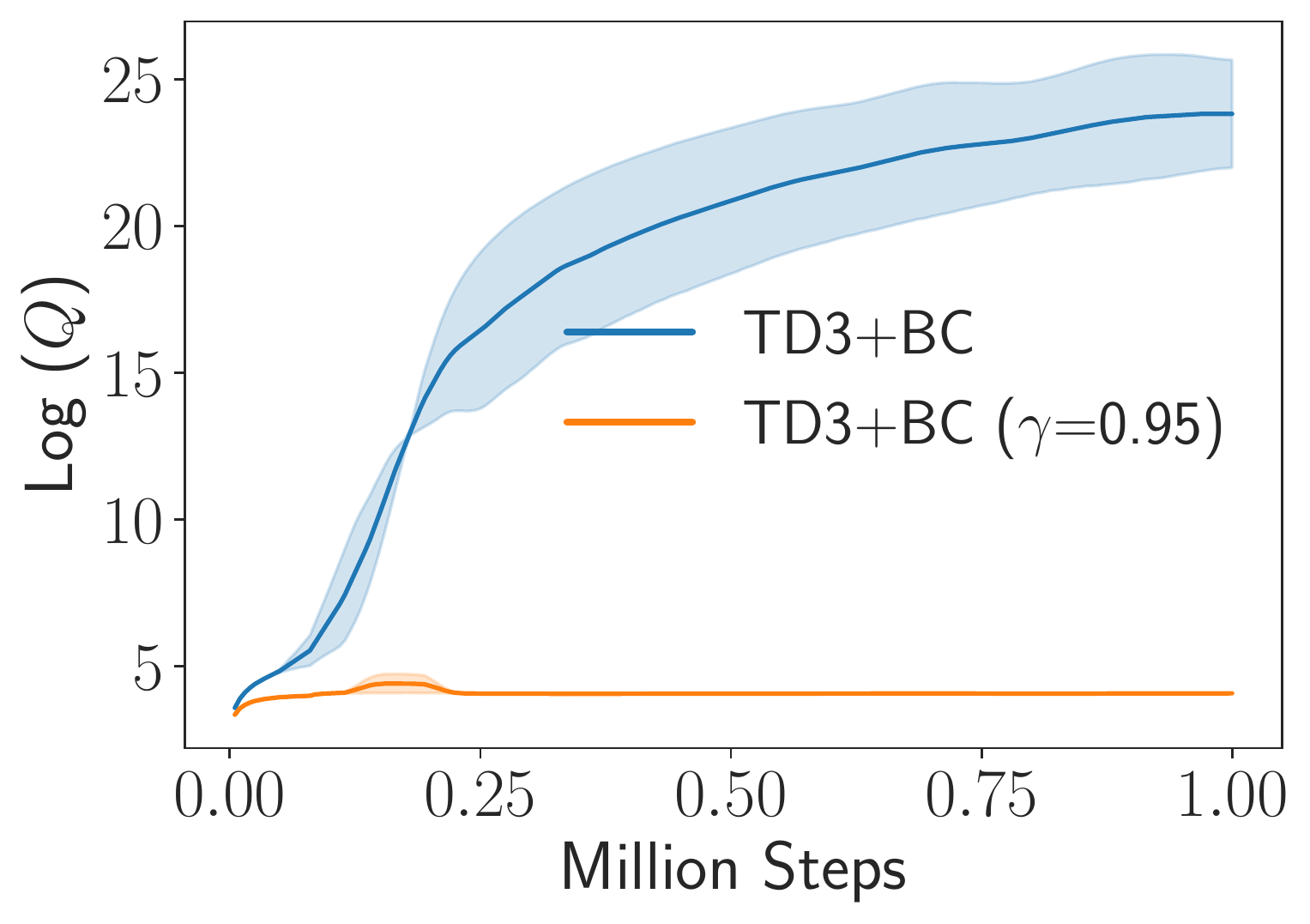}}
    \caption{Experimental results on hopper task consisting of 50 medium trajectories and $x$ noised trajectories.
    The evaluation metric is the episode return and log $Q$-value.}
    \label{fig:my_label_2}
\end{figure}

\begin{figure}[h]
    \centering
    \subfigure[med(50) noise(0)]{
    \includegraphics[scale=0.23]{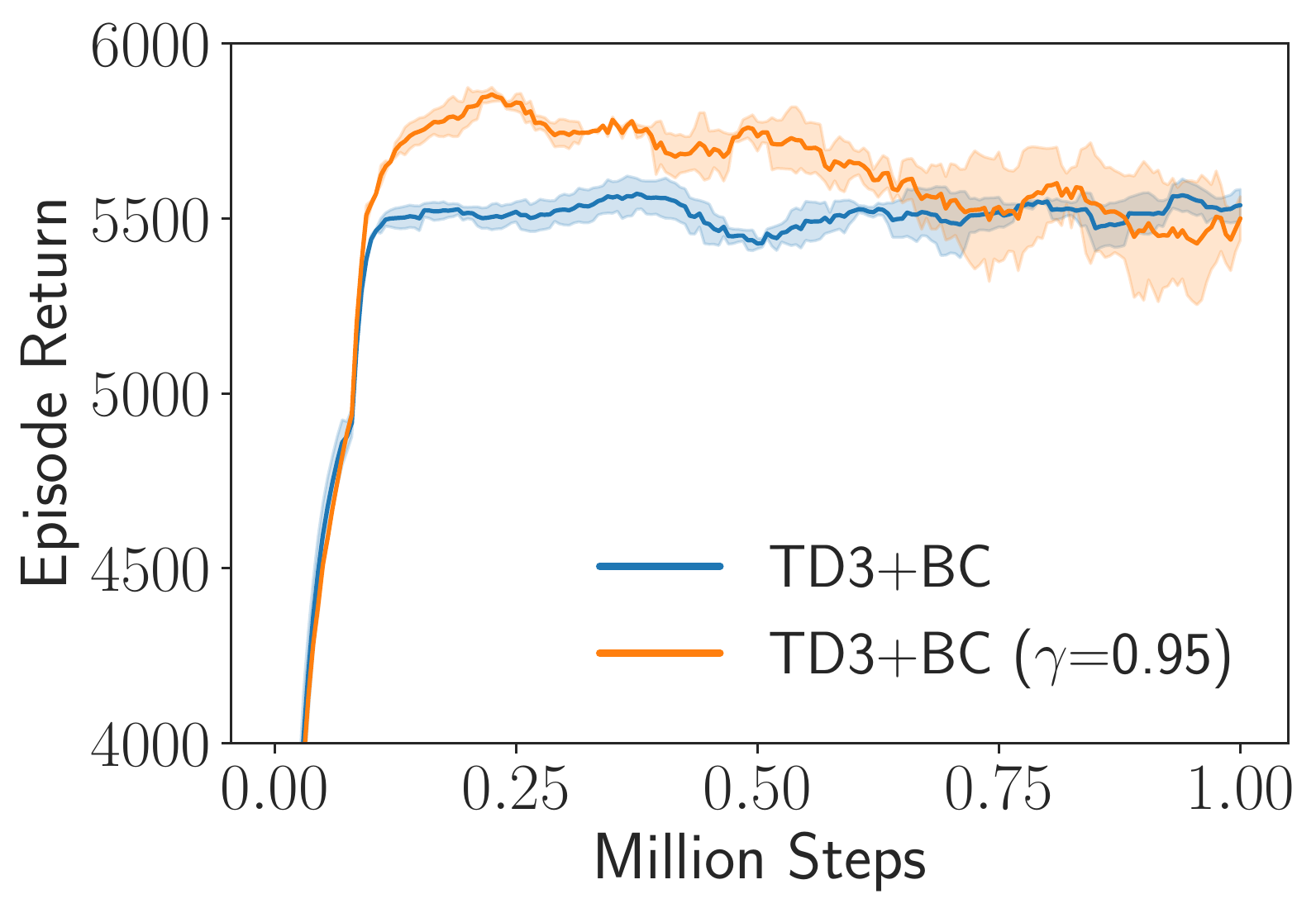}}
    \subfigure[med(50) noise(0)]{
    \includegraphics[scale=0.23]{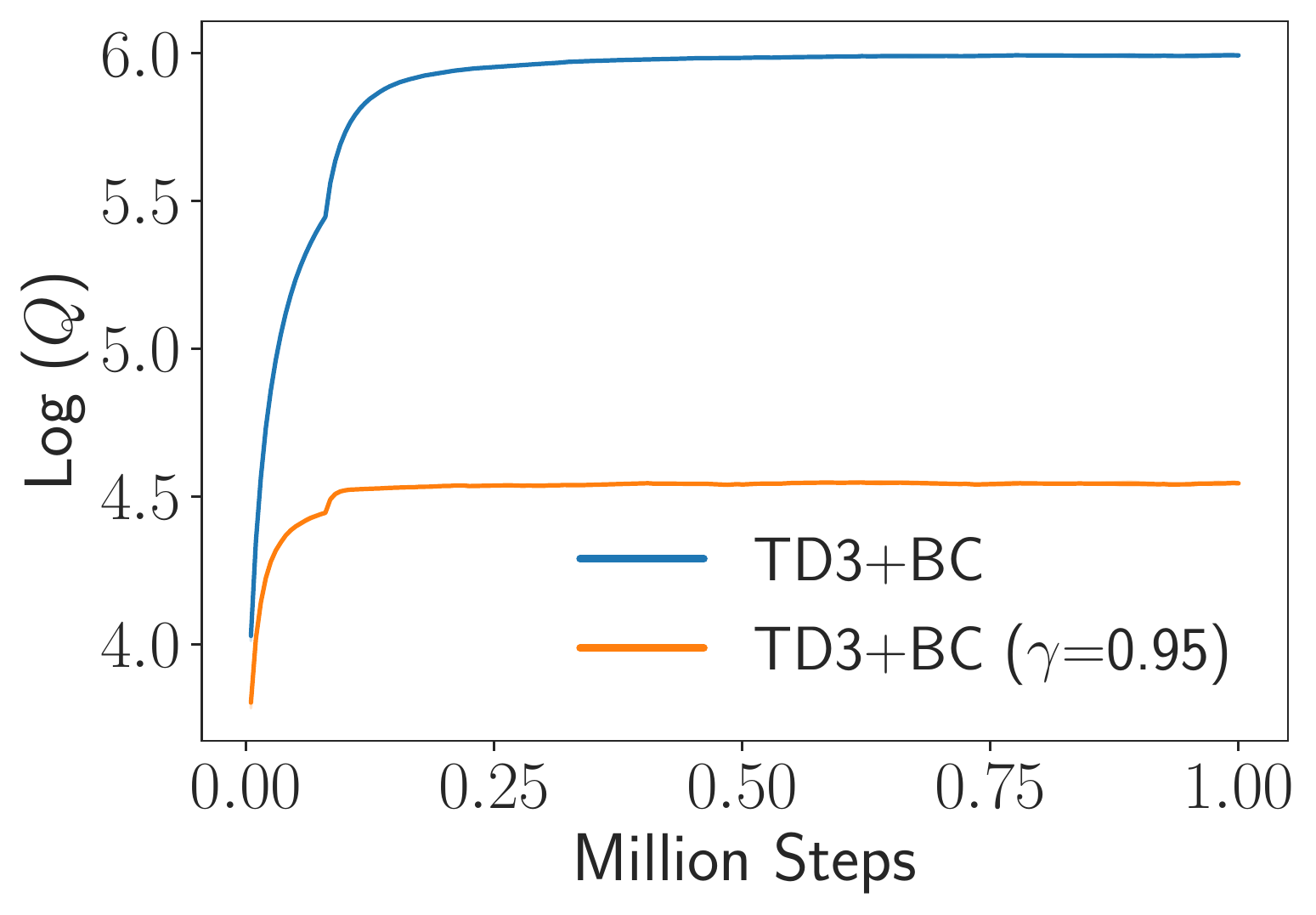}}
    \subfigure[med(50) noise(5)]{
    \includegraphics[scale=0.23]{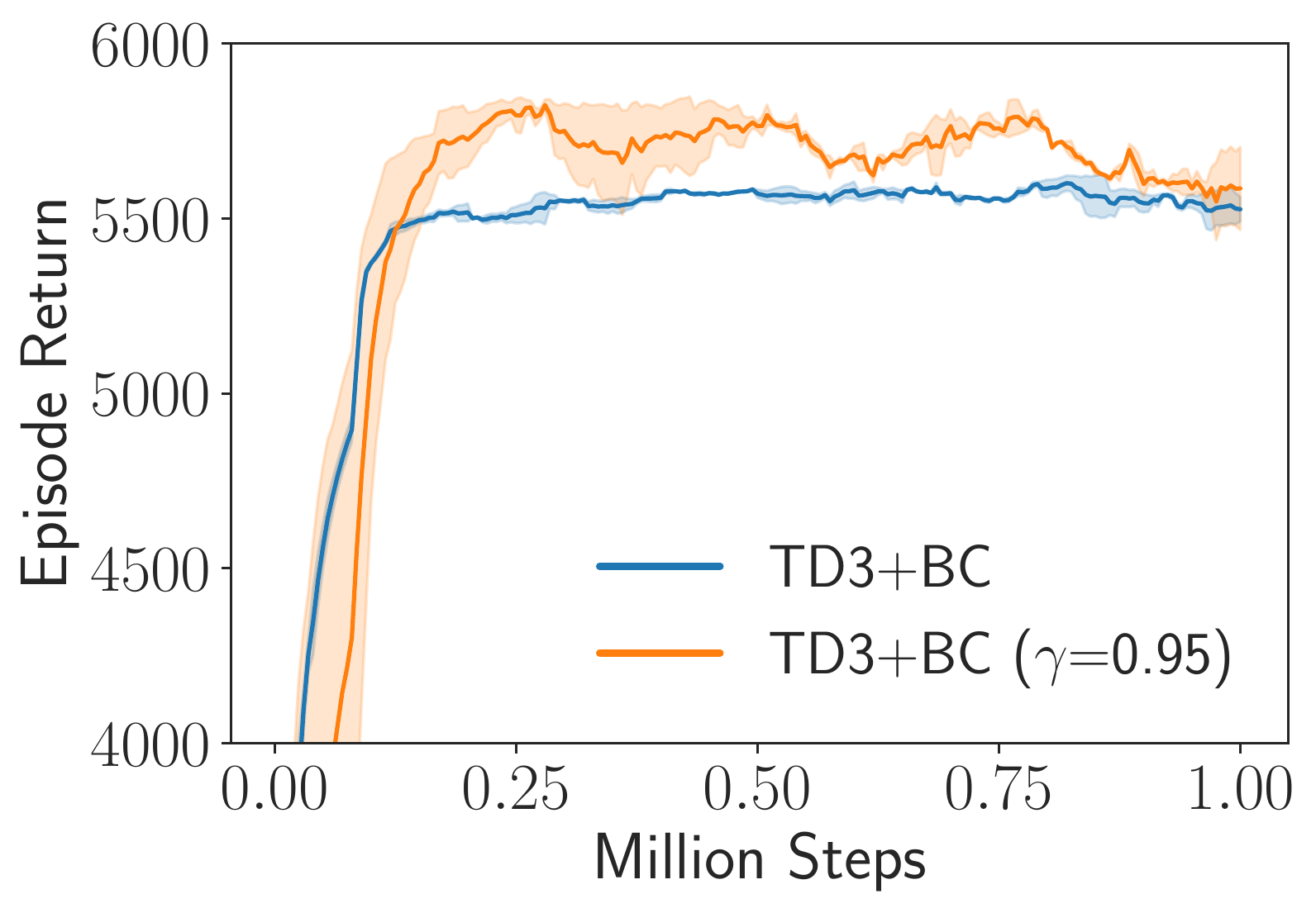}}
    \subfigure[med(50) noise(5)]{
    \includegraphics[scale=0.23]{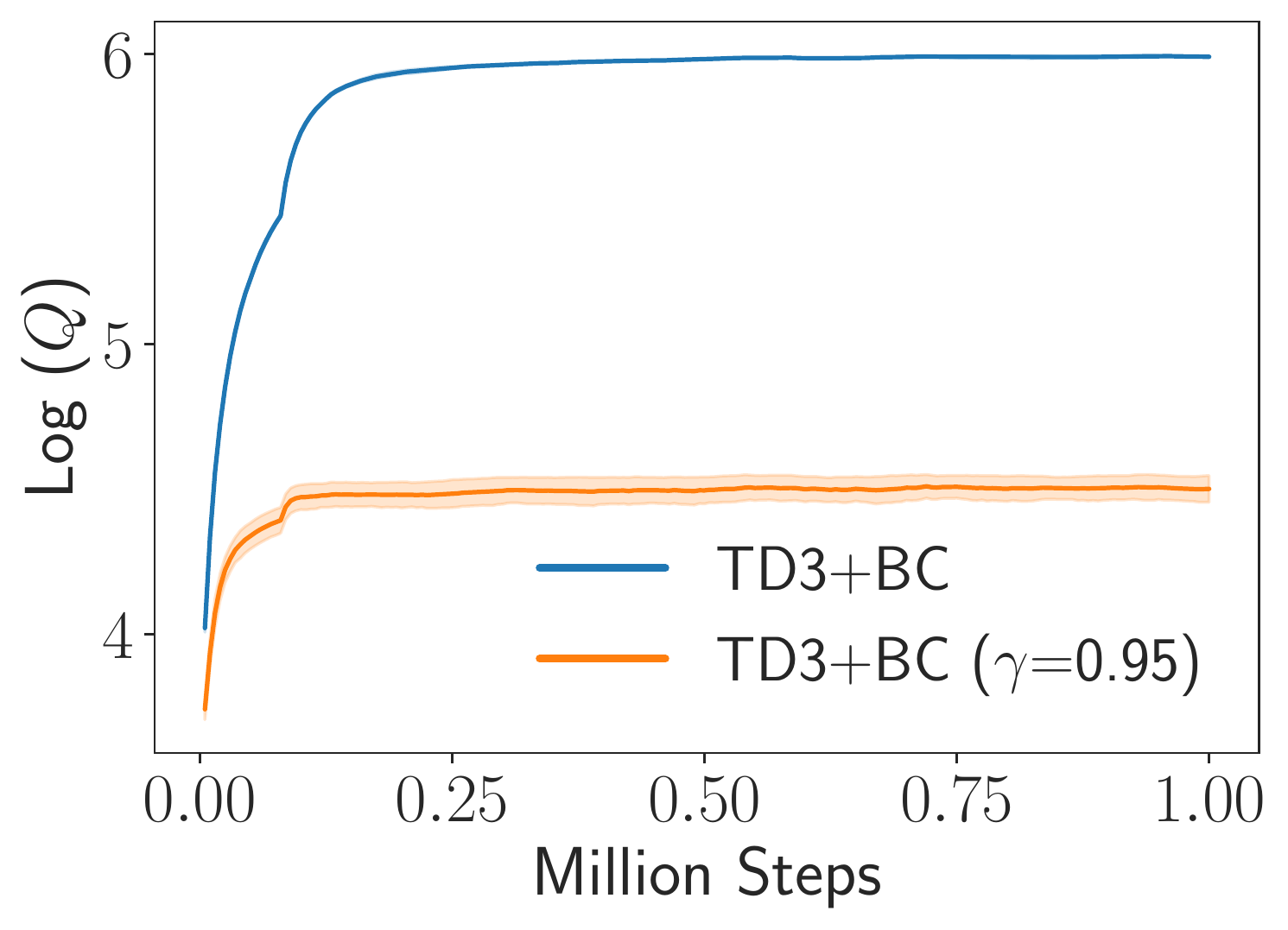}}
    
    \subfigure[med(50) noise(10)]{
    \includegraphics[scale=0.23]{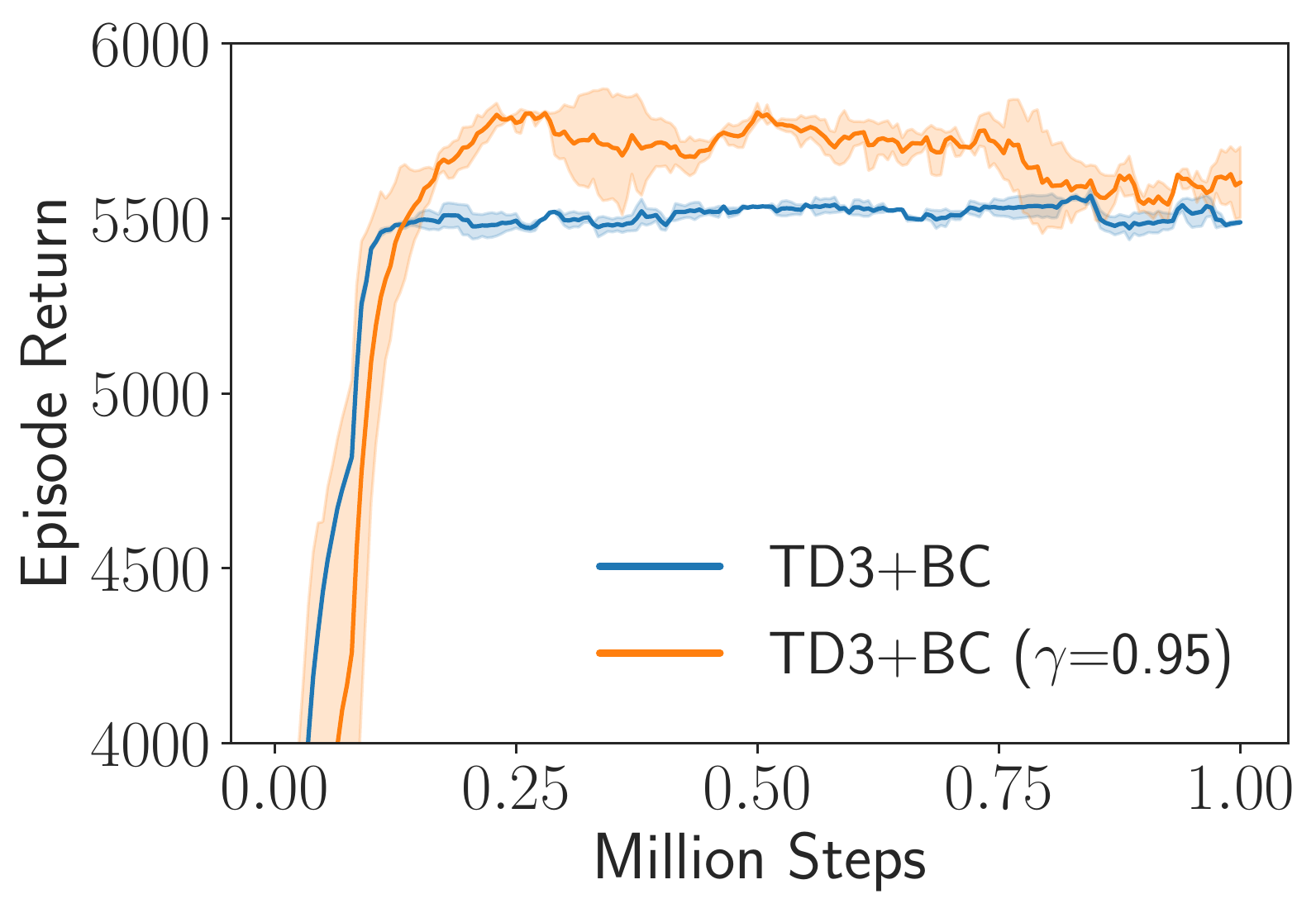}}
    \subfigure[med(50) noise(10)]{
    \includegraphics[scale=0.23]{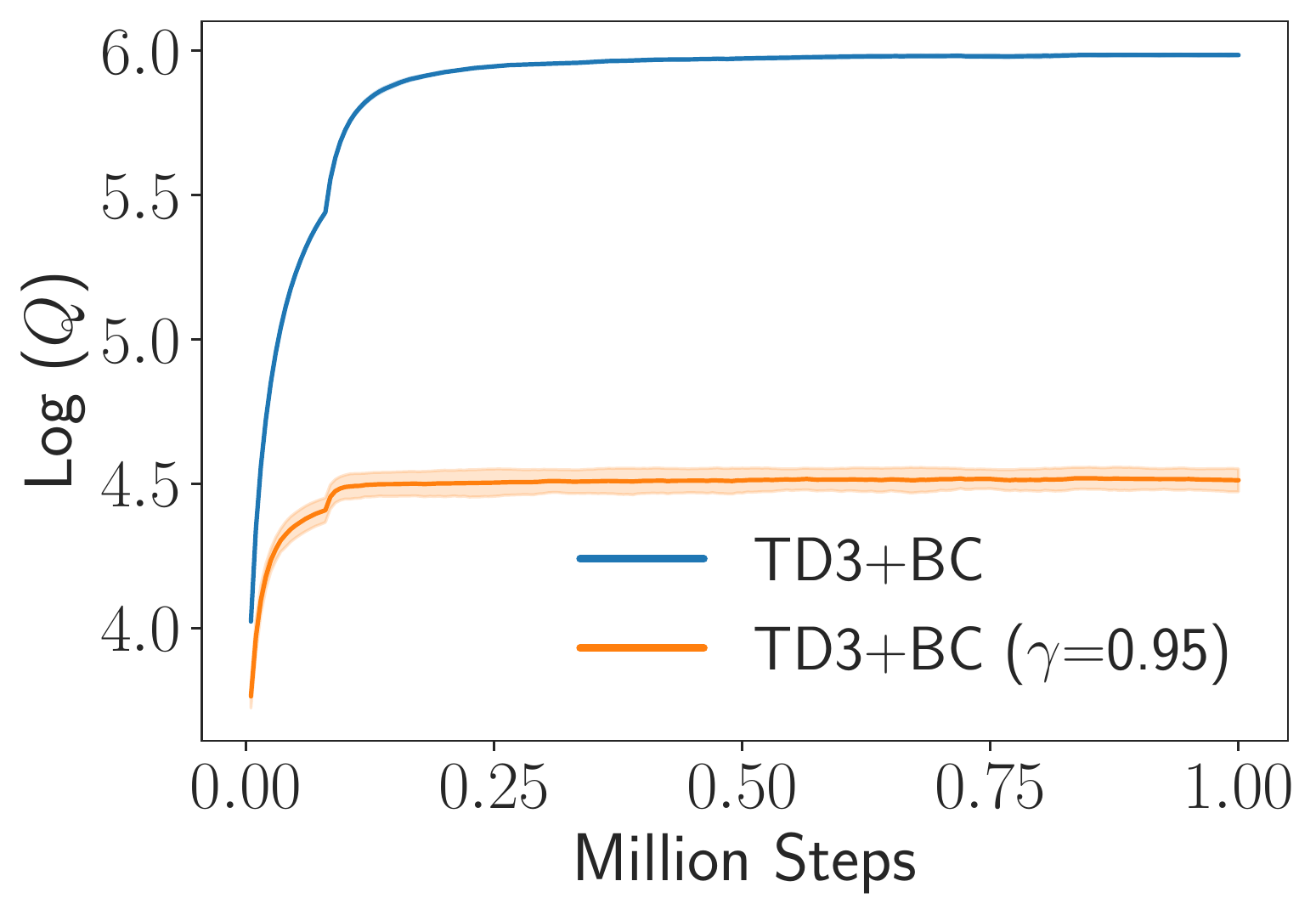}}
     \subfigure[med(50) noise(15)]{
    \includegraphics[scale=0.23]{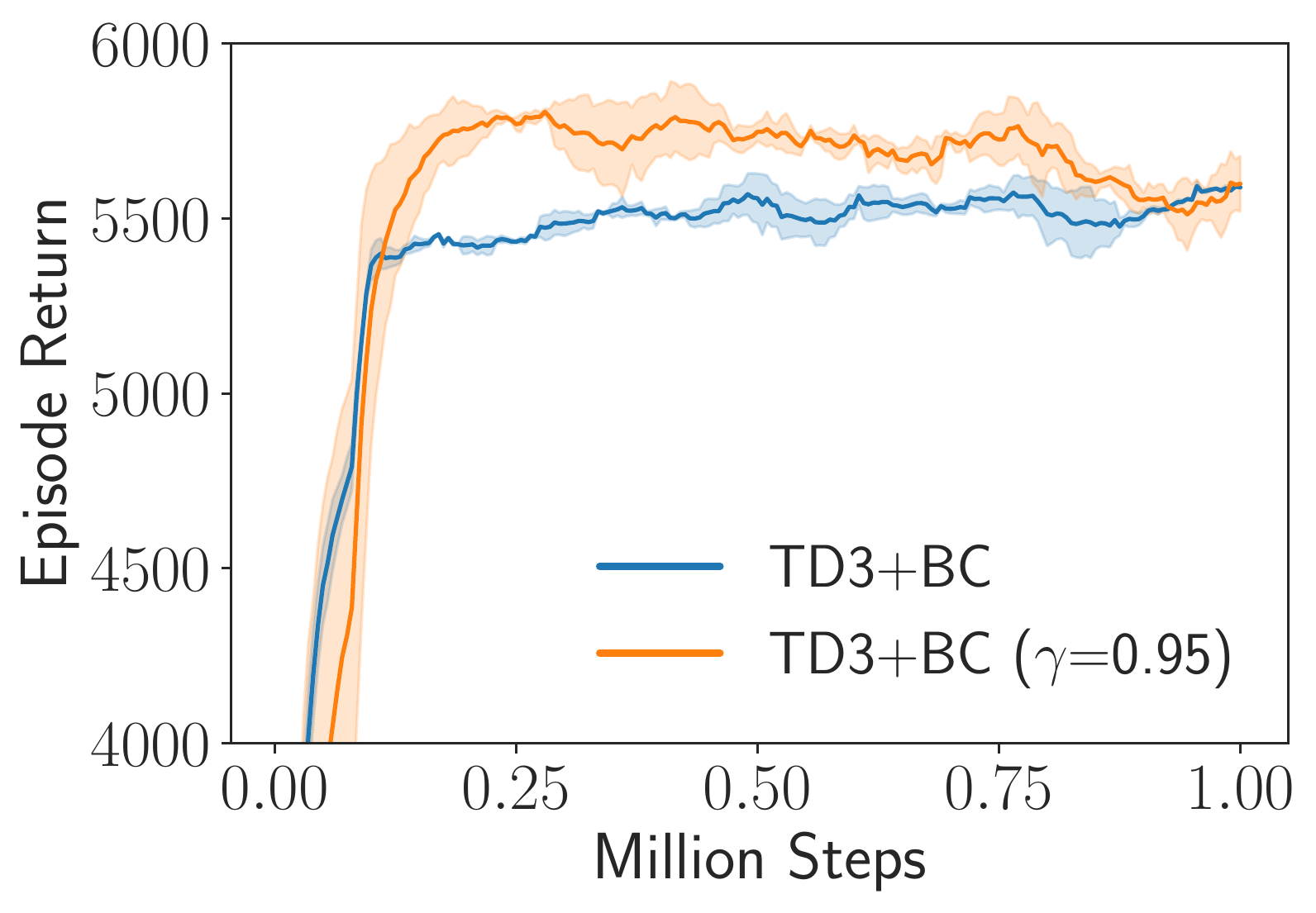}}
    \subfigure[med(50) noise(15)]{
    \includegraphics[scale=0.23]{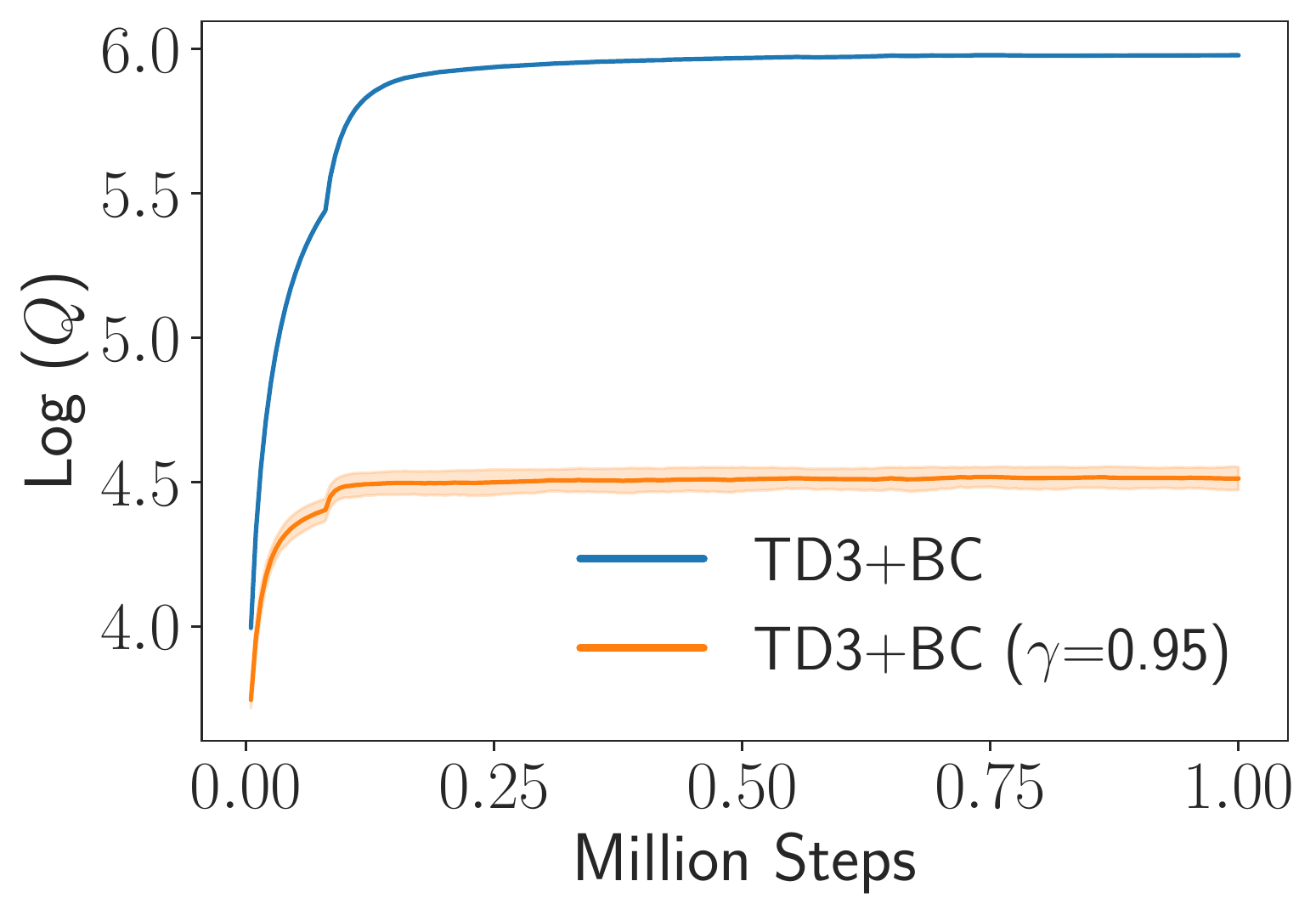}}
    
    \subfigure[med(50) noise(20)]{
    \includegraphics[scale=0.23]{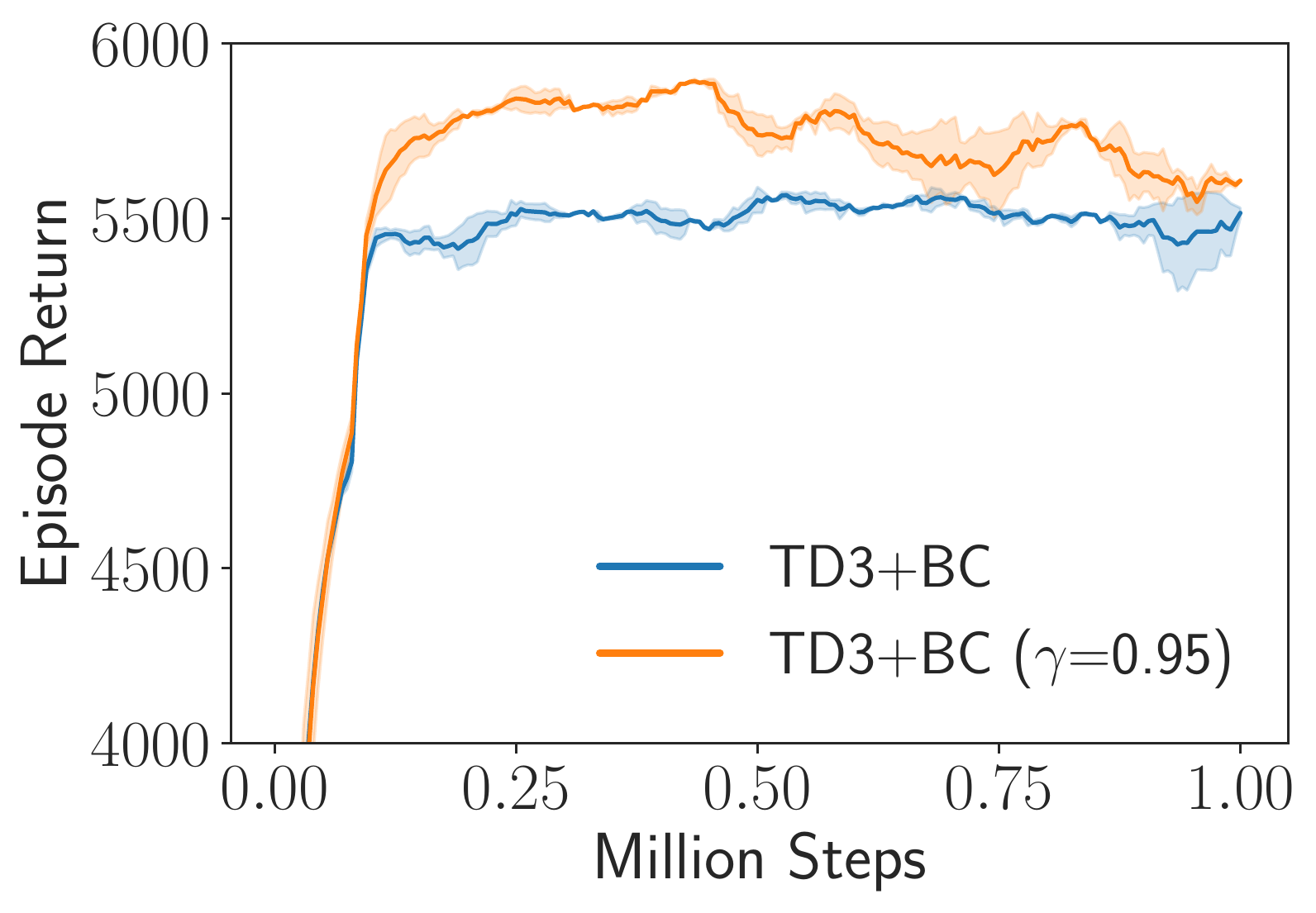}}
    \subfigure[med(50) noise(20)]{
    \includegraphics[scale=0.23]{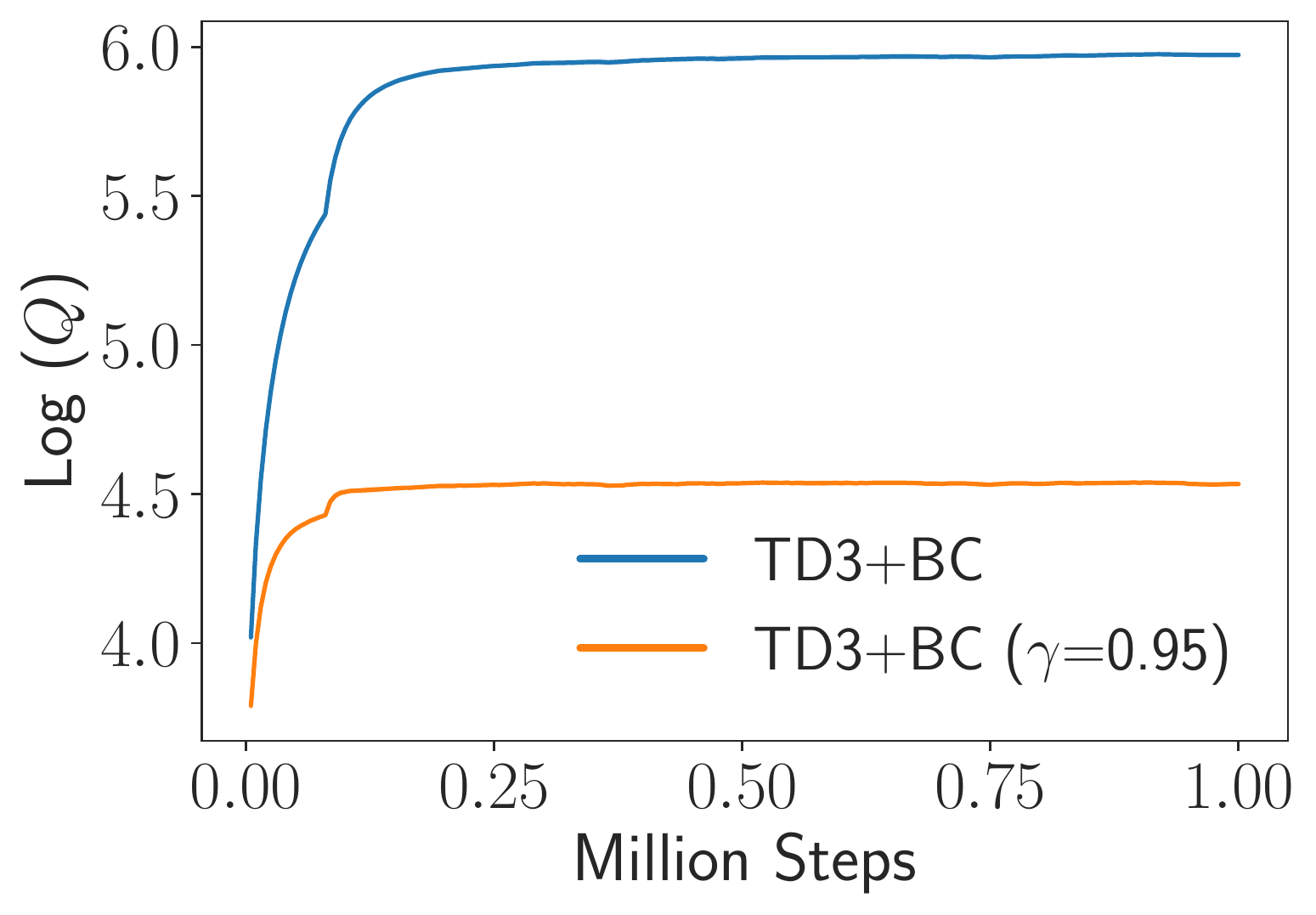}}
     \subfigure[med(50) noise(25)]{
    \includegraphics[scale=0.23]{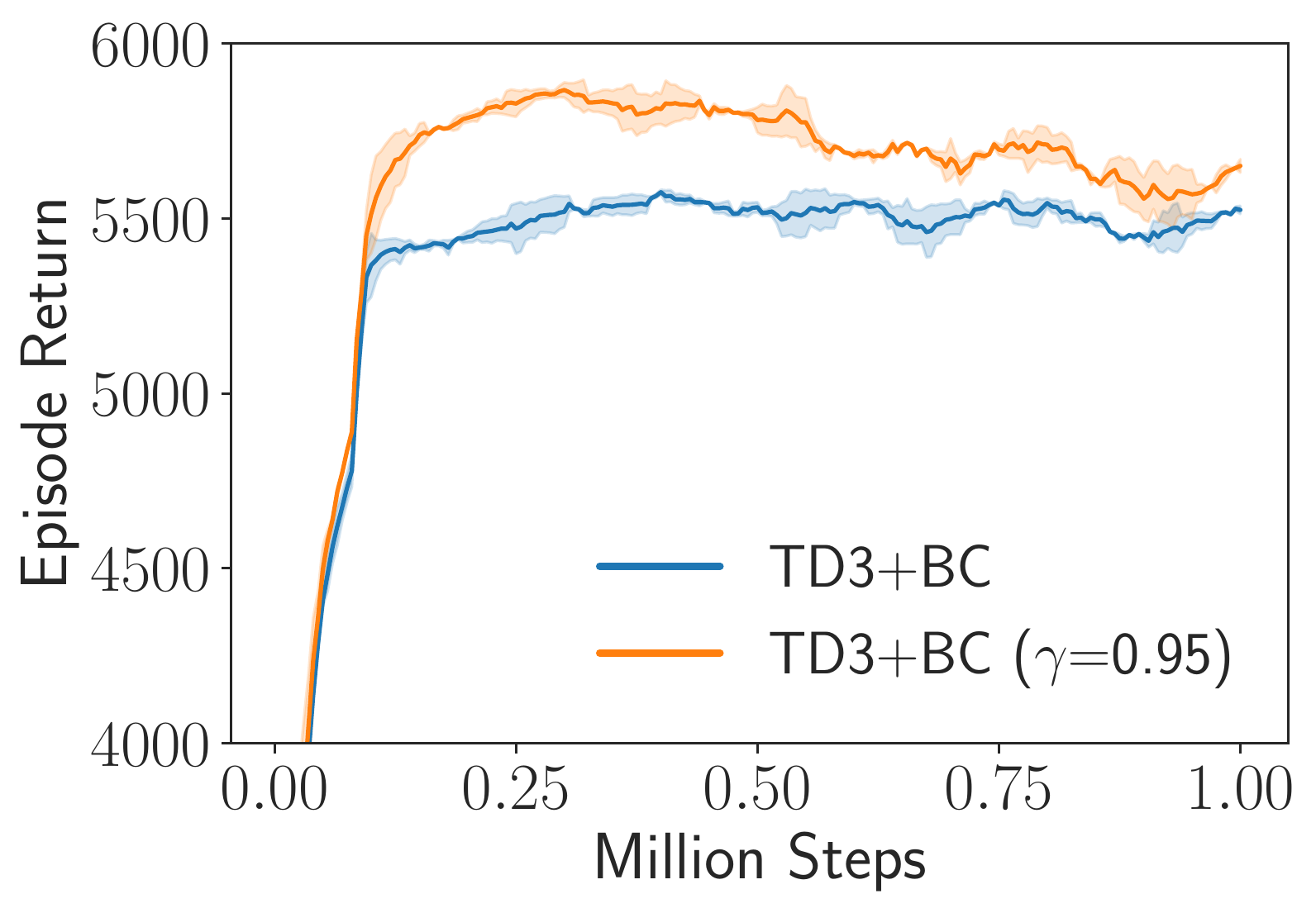}}
    \subfigure[med(50) noise(25)]{
    \includegraphics[scale=0.23]{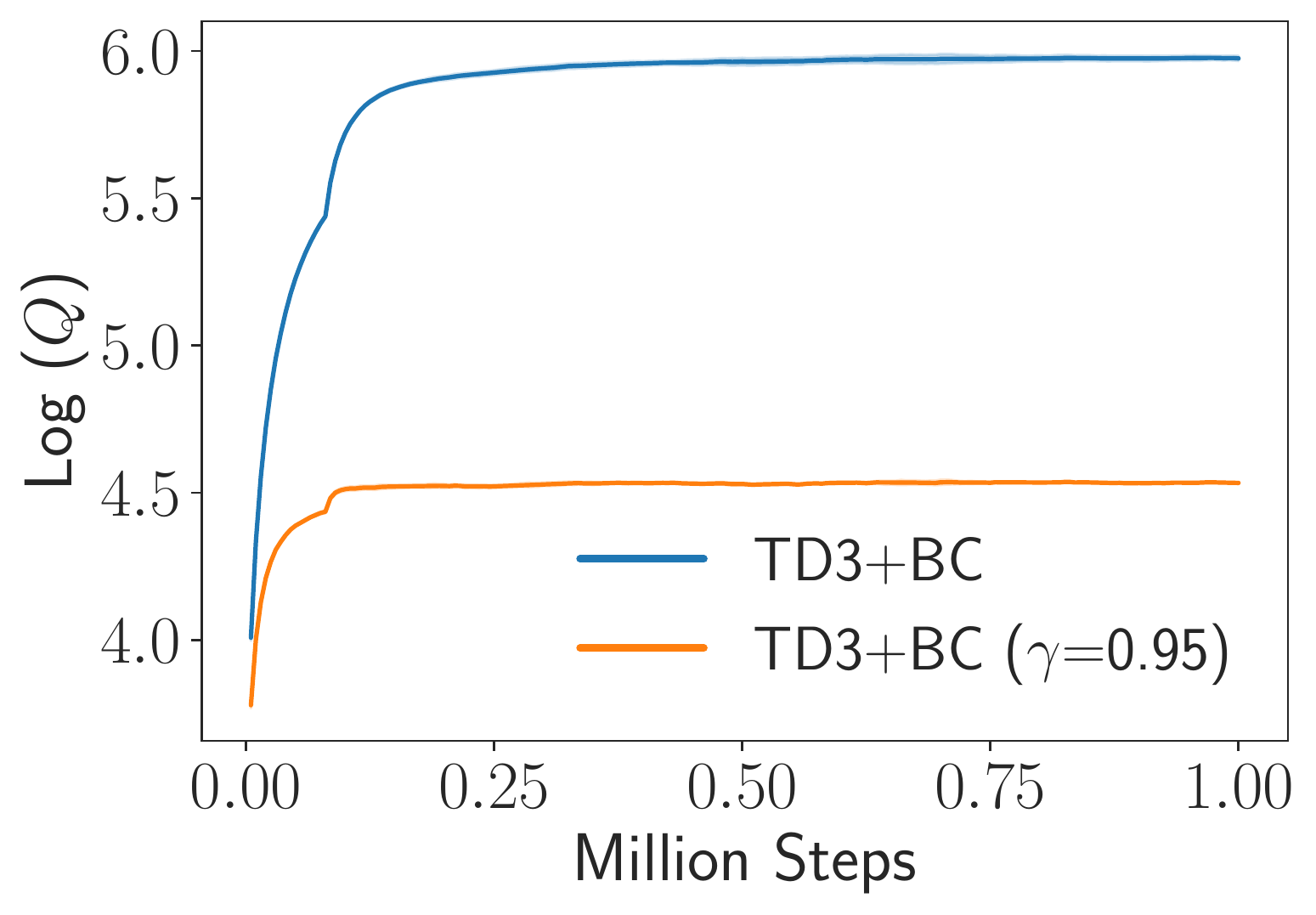}}
    \caption{Experimental results on halfcheetah task consisting of 50 medium trajectories and $x$ noised trajectories.
    The evaluation metric is the episode return and log $Q$-value.
    }
    \label{fig:my_label_3}
\end{figure}


\end{document}